\documentclass{article}

\usepackage[preprint]{neurips_2026}

\usepackage{amsmath,amssymb,amsfonts,amsthm,mathtools,bm}

\usepackage{graphicx}
\usepackage{booktabs}
\usepackage{longtable}
\usepackage{pdflscape}
\usepackage{multirow}
\usepackage{array}
\usepackage{adjustbox}
\usepackage{threeparttable}
\usepackage{rotating}
\usepackage{siunitx}
\usepackage{float}
\usepackage{xcolor}
\usepackage{enumitem}
\usepackage{comment}
\usepackage[most]{tcolorbox}
\usepackage{pgfplots}
\usepgfplotslibrary{groupplots}
\pgfplotsset{compat=1.18}
\usetikzlibrary{arrows.meta}

\usepackage{algorithm}
\usepackage{algorithmic}

\usepackage{url}
\usepackage{hyperref}

\providecommand{\R}{\mathbb{R}}

\theoremstyle{definition}

\theoremstyle{plain}
\newtheorem{proposition}{Proposition}
\newtheorem{theorem}{Theorem}

\newtheorem{corollary}{Corollary}
\theoremstyle{remark}
\newtheorem{remark}{Remark}

\usepackage[utf8]{inputenc} 
\usepackage[T1]{fontenc}    
\usepackage{hyperref}       
\usepackage{url}            
\usepackage{booktabs}       
\usepackage{amsfonts}       
\usepackage{nicefrac}       
\usepackage{microtype}      
\usepackage{xcolor}         

\title{High-Resolution Dynamic Functional Connectivity Generation with Graph-Variate Flow Matching}

\author{%
  Om Roy\thanks{Corresponding author.} \\
  University of Strathclyde \\
  \texttt{o.roy.2022@uni.strath.ac.uk}
  \And
  Yashar Moshfeghi \\
  University of Strathclyde \\
  \texttt{yashar.moshfeghi@strath.ac.uk}
  \And
  Keith Malcolm Smith \\
  University of Strathclyde \\
  \texttt{keith.smith@strath.ac.uk}
}

\begin{document}

\maketitle

\begin{abstract}
High-resolution dynamic functional connectivity (DFC) is attractive for studying rapidly evolving brain-network interactions, but it is difficult to estimate and generate reliably because short temporal windows produce noisy and often low-rank covariance estimates. Graph-Variate Dynamic (GVD) connectivity addresses this problem by modulating fast instantaneous interactions with a stable trial-level support, suppressing spurious fluctuations while emphasizing persistent and informative connections. We further show that this Hadamard construction lifts low-rank instantaneous connectivity from the positive-semidefinite to the positive-definite cone, enabling high-resolution connectivity trajectories to remain on the SPD manifold without additive ridge regularisation or post-hoc projection. Building on this structure, we introduce GVD-CFM, a class-conditional generative model designed specifically for high-resolution dynamic connectivity generation. Each trial is represented as a sequence of SPD GVD matrices on a product Riemannian manifold and mapped through a global log-Euclidean diffeomorphism and an invertible temporal DCT basis. A Transformer-based conditional flow models all spectral modes jointly, allowing the complete high-resolution trajectory to be generated non-autoregressively in Euclidean coordinates while preserving exact correspondence with valid SPD connectivity sequences. Because the full DCT basis is retained, the learned representation also defines a temporal basis expansion that can be decoded on denser temporal grids without retraining. Across multiple EEG motor-imagery datasets, GVD-CFM achieves the strongest overall performance across held-out distributional fidelity, preservation of temporal dynamics, and synthetic-to-real classification, while remaining computationally efficient relative to strong raw-signal and direct GVD-space generative baselines. These results establish GVD-CFM as a framework for generating realistic, temporally coherent, high-resolution brain-network trajectories while preserving manifold structure and supporting resolution-flexible decoding from a single trained model.
\end{abstract}

\section{Introduction}

Dynamic functional connectivity (DFC) has been a topic of major interest in neuroscience in recent years \citep{hutchison2013dfc,preti2017dynamic}. Brain-network connectivity is central to many neuroscientific studies \citep{fox2014resting,allen2014tracking,roy2023robust,roy2024fast}, and EEG, with its very high temporal resolution \citep{pfurtscheller1999erd,smith2017fast}, should in principle allow these networks to be followed at the time scale of cognitive events. In practice, connectivity estimated over short temporal windows is dominated by noise and spurious correlations \citep{leonardi2015spurious,hindriks2016sliding}. Graph-variate dynamic (GVD) connectivity addresses this problem by modulating instantaneous interactions with a stable, trial-level connectivity matrix through a Hadamard product \citep{smith2019gvsa,roy2025hodgefast}. This has been shown to give reliable time-varying dynamics at high temporal resolution, for example in the detection of connectivity changes during event-related potentials \citep{roy2023robust,roy2024fast}.

In the machine learning community, the study of connectivity has been largely limited to the static case \citep{sihag2022vnn,cavallo2024spatiotemporal,roy2026cdnn}, and most progress in architectures and generative models for EEG concerns the raw signal \citep{lotte2018review,lawhern2018eegnet,schirrmeister2017deep,barachant2012riemannian,kobler2022spdbn,hartmann2018eeggan}. Among recent generative models \citep{goodfellow2014gan,kingma2014vae,song2021sde,ho2020ddpm}, flow matching has emerged as an efficient way of training continuous normalizing flows \citep{chen2018neuralode,lipman2023flow,liu2023rectified,albergo2023interpolants}, and conditional flow matching makes its objective tractable \citep{tong2024improving,lipman2024guide}. Connectivity matrices are symmetric positive definite (SPD), and Riemannian geometry provides a natural setting for their study \citep{pennec2006tensor,arsigny2007logeuclidean,pennec2006intrinsic,thanwerdas2022thesis}. This has led to Riemannian flow matching \citep{chen2024geometries} and, to avoid its computational cost, to diffeomorphic pullback flow matching, which performs Euclidean flow matching through a global diffeomorphism and is equivalent to the Riemannian process \citep{collas2025diffeocfm}. These methods require every matrix to be SPD. A covariance matrix estimated from a window with fewer samples than channels is rank deficient, so the geometry that makes these methods tractable is not available for high-resolution DFC.

Here, we propose GVD-CFM, a spectral flow matching model that generates complete high-resolution EEG connectivity trajectories. We first show that the Hadamard product with a positive-definite support lifts rank-deficient window covariances to SPD matrices, so that a GVD trajectory is a point on a product Riemannian manifold. We map this manifold to Euclidean coordinates with a log-Euclidean chart followed by an orthonormal temporal DCT, and train a Transformer to predict the velocity of all DCT modes jointly. The DCT approximately decorrelates the temporal covariance of the trajectory, which simplifies the transport problem, and the generated coefficients can be decoded on finer temporal grids than the one used in training. We further show that current generative models for raw EEG may produce realistic signals but fail to reproduce the temporal dynamics of connectivity. Since GVD-CFM generates connectivity trajectories rather than raw EEG, its samples are directly useful to models that consume connectivity representations, a setting where labeled data are scarce \citep{lashgari2020augmentation,roy2019dlreview,jiang2024labram}.

\paragraph{Contributions.}
\begin{enumerate}[leftmargin=1.4em,itemsep=2pt,topsep=1pt]
\item \textbf{High-resolution DFC that overcomes the covariance rank issue.}
We show that GVD connectivity is SPD even when window covariances are rank deficient, provided the stable support is SPD and every channel has nonzero energy in the window. This gives a valid log-Euclidean representation when a window contains fewer samples than channels allowing high resolution temporal precision.
\item \textbf{Non-autoregressive spectral generation of complete trajectories.}
GVD-CFM performs conditional flow matching on complete log-Euclidean GVD trajectories in an orthonormal DCT basis, exactly equivalent to Riemannian flow matching on the product SPD manifold. The DCT approximately diagonalizes the optimal Gaussian transport field, and the generated coefficients can be evaluated on denser temporal grids without retraining.
\item \textbf{Controlled evaluation and physiological plausibility.}
We compare against four raw-EEG generators and three GVD-space controls that share the GVD targets and decoding pipeline, ablate the spectral representation and the stable support, and show that generated trajectories preserve the regional organization and time course of real motor-imagery connectivity.
\end{enumerate}
\section{Related Work}
\label{sec:related}

\paragraph{Dynamic functional connectivity.}
DFC is commonly estimated with sliding-window correlation \citep{hutchison2013dfc,allen2014tracking,preti2017dynamic}, but short windows increase noise and can produce spurious fluctuations \citep{leonardi2015spurious,hindriks2016sliding}. Regularized approaches improve stability through shrinkage or temporal smoothness \citep{ledoit2004wellconditioned,chen2010oas,monti2014estimating,hallac2017tvgl}, usually at the cost of temporal resolution. Graph-variate methods instead filter instantaneous interactions using a stable trial-level connectivity matrix, enabling high-resolution EEG connectivity estimates \citep{smith2017fast,smith2019gvsa,roy2023robust,roy2024fast,roy2025hodgefast}.

\paragraph{SPD geometry and generative modeling.}
EEG covariance matrices lie on the SPD manifold, motivating geometry-aware methods based on affine-invariant, log-Euclidean, and log-Cholesky metrics \citep{pennec2006tensor,arsigny2007logeuclidean,lin2019cholesky}. These ideas have also been used in neural models for covariance data \citep{kobler2022spdbn,sihag2022vnn,cavallo2024spatiotemporal,ju2025spdsurvey,roy2026cdnn}. Generative models on manifolds include Riemannian diffusion and flow matching \citep{debortoli2022riemannian,jo2023mixture,chen2024geometries}, while SPD-specific work has focused mainly on generating single matrices \citep{li2024spdddpm,desurrel2025wrapped,marti2020corrgan}. Diffeo-CFM \citep{collas2025diffeocfm} is closest to our approach, but considers static, full-rank connectivity matrices.

\paragraph{Graph, time-series and EEG generation.}
Graph generators mainly model static graphs \citep{vignac2023digress,qin2025defog,huang2025spectral,williams2025weighted}, while time-series models operate in Euclidean spaces \citep{esteban2017rcgan,rasul2021timegrad,tenbrinke2026stflow}. EEG generation has largely focused on synthesizing raw signals using GANs or denoising models \citep{hartmann2018eeggan,luo2018cwgan,wang2026jet}. In contrast, GVD-CFM directly generates high-temporal-resolution connectivity trajectories rather than raw EEG.

\section{Background}
\label{sec:background}

\paragraph{Notation.}
$\tau\in[0,1]$ denotes flow time and $t$ EEG sample time. A trial has $T$ samples on $d$ channels and is divided into $B$ temporal windows. $\mathbb{S}^{d}$ and $\mathbb{S}_{++}^{d}$ denote the symmetric and SPD $d\times d$ matrices, $\odot$ and $\oslash$ the Hadamard product and division, and $m=d(d+1)/2$. $\operatorname{svec}:\mathbb{S}^{d}\rightarrow\mathbb{R}^{m}$ is lower-triangular vectorization with off-diagonal entries scaled by $\sqrt{2}$, which preserves the Frobenius inner product.

\paragraph{Conditional flow matching.}
\label{sec:cfm}
Flow matching \citep{lipman2023flow} learns a velocity field $v_\theta:[0,1]\times\R^{D}\to\R^{D}$ that transports a prior $p_0$ to a data distribution $p_1$. For a coupling $\pi(z_0,z_1)$, the linear path $z_\tau=(1-\tau)z_0+\tau z_1$ has constant velocity $z_1-z_0$, and the conditional objective
\begin{equation}
\mathcal{L}_{\mathrm{CFM}}(\theta)=\mathbb{E}_{\tau,\,(z_0,z_1)\sim\pi}\bigl\|v_\theta(\tau,z_\tau)-(z_1-z_0)\bigr\|^2
\label{eq:cfm}
\end{equation}
has the same gradient as the intractable marginal objective. An independent coupling gives the standard method; a minibatch optimal-transport coupling straightens the marginal paths \citep{villani2009optimal,tong2024improving}.

\paragraph{Riemannian flow matching by diffeomorphic pullback.}
\label{sec:pullback}
Riemannian flow matching \citep{chen2024geometries} regresses onto geodesic velocities, at the cost of computing geodesics and Riemannian norms. \citet{collas2025diffeocfm} observed that this cost disappears whenever $\mathcal{M}$ admits a global diffeomorphism $\varphi:\mathcal{M}\to E$ onto a Euclidean space; under the pullback metric $\varphi^\ast g_E$, $\varphi$ is an isometry and geodesics are pulled-back straight lines.
\begin{proposition}[Pullback reduction; \citealp{collas2025diffeocfm}]
\label{prop:pullback}
On $(\mathcal{M},\varphi^\ast g_E)$ the Riemannian CFM objective equals the Euclidean objective of Equation~\ref{eq:cfm} on $z=\varphi(x)$. Integrating in $E$ and decoding by $\varphi^{-1}$ gives exactly the samples obtained by integrating on $\mathcal{M}$, and every sample lies on $\mathcal{M}$.
\end{proposition}

\section{Methods}
\subsection{Graph-Variate Dynamic Connectivity}
\label{sec:gvd}

Graph-variate signal analysis (GVSA) represents a multivariate time series through evolving interactions on a stable support \citep{smith2019gvsa}. Let $u_t\in\mathbb{R}^d$ be the channel-standardized EEG sample at time $t$, $U=[u_1,\ldots,u_T]\in\R^{d\times T}$, and $W=\frac{1}{T}UU^\top$ the whole-trial correlation matrix. We use the signed outer product $J_t=u_tu_t^\top$ as the instantaneous interaction and define $\Delta_t=W\odot J_t$. Unlike the original GVSA definition, we retain the sign of both the support and the instantaneous interaction, and we retain the main diagonal, since both are needed for the positive-definite geometry of our model (Appendix~\ref{app:gvsa}). When the stable support is computed from the signal itself, this is graph-variate dynamic (GVD) connectivity \citep{smith2019gvsa,roy2024fast,roy2025hodgefast}.

We divide the $T$ samples into $B$ disjoint, full-coverage windows $\{\mathcal{I}_b\}_{b=1}^{B}$. Individual samples are noisy and give prohibitively long sequences, so a trial is represented by the window averages
\begin{equation}
    \overline{\Delta}_b
    =\frac{1}{|\mathcal{I}_b|}\sum_{t\in\mathcal{I}_b}\Delta_t
    =W\odot J_b,
    \qquad
    J_b=\frac{1}{|\mathcal{I}_b|}\sum_{t\in\mathcal{I}_b}u_tu_t^\top .
    \label{eq:window_gvd}
\end{equation}
Setting $B=T$ recovers sample-resolution connectivity, while smaller $B$ trades temporal resolution for lower variance. The window covariance $J_b$ has rank at most $\min(d,|\mathcal{I}_b|)$, so it is singular whenever a window contains fewer samples than channels, and $\log J_b$ does not exist. The Hadamard product with the support removes this barrier.

\begin{proposition}[GVD lifts rank-deficient covariance]
\label{prop:gvd_spd}
Let $W\in\mathbb{S}_{++}^{d}$. If every channel has nonzero energy in window $b$, $q_i=\frac{1}{|\mathcal{I}_b|}\sum_{t\in\mathcal{I}_b}u_{i,t}^2>0$ for $i=1,\ldots,d$, then $\overline{\Delta}_b=W\odot J_b$ is SPD, whatever the rank of $J_b$, and
\begin{equation}
\lambda_{\min}(\overline{\Delta}_b)\geq\lambda_{\min}(W)\min_{i}q_i>0 .
\end{equation}
In particular, a single sample with rank-one $J_t$ gives $\Delta_t\in\mathbb{S}_{++}^{d}$ whenever every component of $u_t$ is nonzero. A GVD trajectory $\overline{\boldsymbol{\Delta}}=(\overline{\Delta}_1,\ldots,\overline{\Delta}_B)$ therefore lies on the product Riemannian manifold $\mathcal{M}_B=(\mathbb{S}_{++}^{d})^{B}$.
\end{proposition}

The proof uses $W\odot u_tu_t^\top=D_tWD_t$ with $D_t=\operatorname{diag}(u_t)$ and is given in Appendix~\ref{app:proof_gvd_spd}. The result requires $W\succ0$, which a sample correlation matrix need not satisfy. Here $\operatorname{rank}(W)=\operatorname{rank}(U)$, which is full whenever the $T\geq512$ samples span $\mathbb{R}^d$ ($d\leq30$). We verify strict positive definiteness of every support and window matrix in double precision, and no trial failed this check.

\subsection{GVD-CFM}

\begin{figure}[t]
\centering
\includegraphics[width=\textwidth]{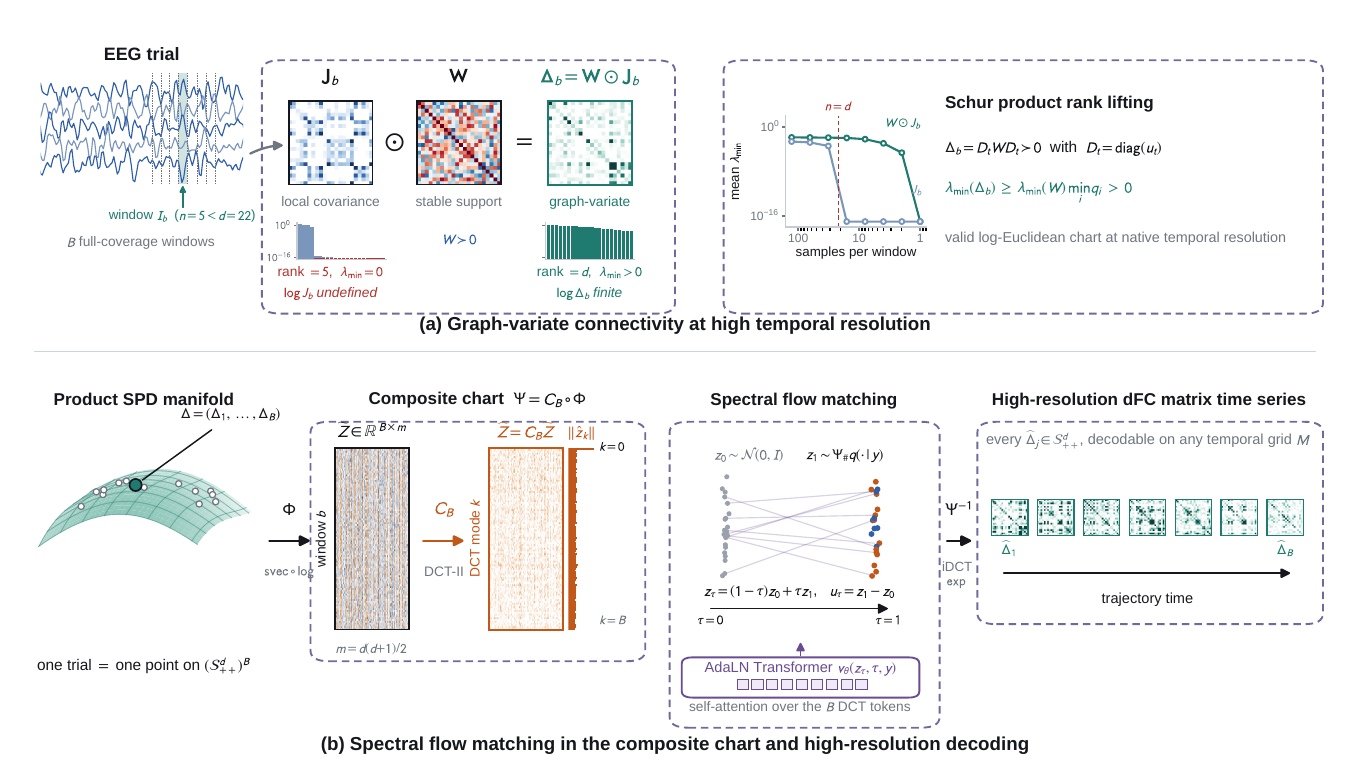}
\caption{\textbf{Overview of GVD-CFM.}
\textbf{(a)} When a window has fewer samples than channels ($n<d$), the window covariance $J_b$ is rank deficient and $\log J_b$ is undefined. With a positive-definite support $W$, $\Delta_b=W\odot J_b$ is positive definite with $\lambda_{\min}(\Delta_b)\geq\lambda_{\min}(W)\min_i q_i>0$, so the log-Euclidean chart remains valid at native temporal resolution.
\textbf{(b)} A trial is a point on $(\mathbb{S}^d_{++})^{B}$. The chart $\Psi$ applies $z_b=\operatorname{svec}\log\Delta_b$, featurewise standardization and the orthonormal temporal DCT-II; it is a global diffeomorphism and an isometry for its pullback metric (Remark~\ref{rem:standardization}). Flow matching along $z_\tau=(1-\tau)z_0+\tau z_1$ is performed in these coordinates by a Transformer over DCT modes, and decoding through $\Psi^{-1}$ returns a trajectory in which every $\widehat{\Delta}_b$ lies on $\mathbb{S}^d_{++}$.}
\label{fig:gvdcfm}
\end{figure}

We propose \emph{Graph-Variate Dynamic Conditional Flow Matching} (GVD-CFM). Essentially, GVD-CFM maps each trajectory to Euclidean coordinates with a global diffeomorphism, organizes these coordinates by temporal frequency, and learns a flow over the complete trajectory at once (Figure~\ref{fig:gvdcfm}).

\subsubsection{Composite Diffeomorphism}
\label{sec:chart}

We apply the log-Euclidean diffeomorphism \citep{pennec2006tensor,arsigny2007logeuclidean} to each window, $z_b=\operatorname{svec}(\log\overline{\Delta}_b)\in\mathbb{R}^{m}$, and stack the results into $Z=\Phi(\overline{\boldsymbol{\Delta}})=[z_1,\ldots,z_B]^\top\in\mathbb{R}^{B\times m}$. We standardize each feature with training-set statistics, $\widetilde{Z}=(Z-\mathbf{1}\mu^\top)\oslash\mathbf{1}\sigma^\top$, and apply the orthonormal DCT-II matrix $C_B$ \citep{ahmed1974dct,strang1999dct}, $C_B C_B^\top=I_B$, along the temporal axis. The complete chart is
\begin{equation}
    \Psi:\left(\mathbb{S}_{++}^{d}\right)^{B}\longrightarrow\mathbb{R}^{B\times m},
    \qquad
    \Psi(\overline{\boldsymbol{\Delta}})=\bar{Z}=C_B\widetilde{Z}.
\end{equation}
The matrix logarithm, $\operatorname{svec}$, the standardization and the orthonormal DCT are all globally invertible, so $\Psi$ is a diffeomorphism. Its inverse applies $C_B^\top$, reverses the standardization and exponentiates each window (Appendix~\ref{app:chart_inverse}). GVD-CFM can therefore be trained in Euclidean coordinates while remaining in one-to-one correspondence with trajectories on the product SPD manifold. For the unmodulated window covariance this chart does not exist when $|\mathcal{I}_b|<d$.

\begin{remark}[Standardization and the metric]
\label{rem:standardization}
Standardization is not an isometry of the log-Euclidean metric. The chart $C_B\circ\Phi$ is an isometry for the product log-Euclidean metric $g_{\mathrm{LE}}$, while $\Psi$ is an isometry for its pullback metric $g_\sigma=\Psi^\ast g_E$, the constant reweighting $\langle\xi,\eta\rangle_\sigma=\sum_{b,j}\xi_{bj}\eta_{bj}/\sigma_j^2$ of $g_{\mathrm{LE}}$ in log coordinates. The two metrics have the same geodesics, and the two flow-matching objectives differ only by a fixed weighting of the velocity residual and share the same population minimizer. Our equivalence statements refer to $g_\sigma$ and reduce to $g_{\mathrm{LE}}$ when $\sigma\equiv1$ (Corollary~\ref{cor:standardized_chart}).
\end{remark}

\subsubsection{Spectral Velocity Network}
\label{sec:transformer}

The DCT organizes temporal variation by frequency \citep{strang1999dct,shuman2013gsp}: persistent structure concentrates in low-order modes and faster changes in higher-order modes. All $B$ modes are retained, so no temporal resolution is discarded. Each row $\bar{z}_k$ of $\bar{Z}$ is a Transformer token for DCT mode $k$ \citep{vaswani2017attention}, with hidden state given by a learned projection of $\bar{z}_k$ plus an embedding of the mode index. The model is class-conditional through $c_\tau=e_{\mathrm{flow}}(\tau)+e_{\mathrm{class}}(y)$, a Fourier encoding of flow time \citep{tancik2020fourier} plus a learned class embedding, which modulates adaptive LayerNorm (AdaLN) Transformer blocks \citep{ba2016layernorm,peebles2023dit}. Self-attention acts across all DCT modes and gives hidden states $h_D\in\mathbb{R}^{B\times h}$, and an output projection gives the joint velocity $v_\theta:\mathbb{R}^{B\times m}\times[0,1]\times\mathcal{Y}\to\mathbb{R}^{B\times m}$ of the complete trajectory in a single forward pass.

\paragraph{Temporal branch.}
The amplitude of one window is spread across all $B$ modes, while the decoder exponentiates each window separately. Errors that are unbiased in log coordinates therefore inflate the power of decoded windows (Remark~\ref{rem:appendix_amplitude}), and a purely spectral network can produce a systematic amplitude offset (Appendix~\ref{app:time_context_visual}). To correct this we add a small temporal branch $T_\theta$ of two AdaLN blocks, which reads the same flow state in the window basis and returns its features to DCT alignment through a learned tokenwise gate,
\begin{equation}
    h_T=C_B\,T_\theta\bigl(C_B^\top z_\tau,c_\tau\bigr),
    \qquad
    v_\theta(z_\tau,\tau,y)=P_{\mathrm{out}}\bigl(h_D+g\odot h_T\bigr),
    \qquad
    g=\sigma\bigl(W_g[h_D;h_T]+b_g\bigr).
    \label{eq:gated_fusion}
\end{equation}
The flow state and the velocity remain in DCT coordinates, and since $C_B^\top$ is a fixed linear map, the objective, its minimizer and the Riemannian equivalence are unchanged. The branch's primary role is to reduce the amplitude bias (Section~\ref{sec:timecontext_ablation}, Appendix~\ref{app:velocity_network}).

\subsubsection{Training and Sampling}
\label{sec:training}

Let $\Psi_{\#}q(\cdot\mid y)$ be the Euclidean pushforward of the class-conditional distribution of GVD trajectories. We draw $z_0\sim\mathcal{N}(0,I)$, $z_1\sim\Psi_{\#}q(\cdot\mid y)$ and $\tau\sim\mathcal{U}[0,1]$, and use $z_\tau=(1-\tau)z_0+\tau z_1$ with target $u_\tau=z_1-z_0$. Within each minibatch, source and target samples of the same class are paired by an entropic optimal-transport plan $\pi_y$ computed with Sinkhorn iterations \citep{cuturi2013sinkhorn,tong2024improving} on the squared Euclidean cost in the DCT chart. The orthonormal DCT preserves distances, so this cost is the squared geodesic distance under $g_\sigma$, and the coupling is computed in the geometry of the product manifold. GVD-CFM minimizes
\begin{equation}
    \mathcal{L}_{\mathrm{GVD\text{-}CFM}}(\theta)
    =
    \mathbb{E}_{y,\,(z_0,z_1)\sim\pi_y,\,\tau}
    \Bigl[
        \tfrac{1}{Bm}\bigl\|v_\theta(z_\tau,\tau,y)-u_\tau\bigr\|_F^2
    \Bigr].
\end{equation}
To sample, we draw $z(0)\sim\mathcal{N}(0,I)$, integrate $\mathrm{d}z/\mathrm{d}\tau=v_{\theta^\star}(z,\tau,y)$ with fourth-order Runge--Kutta \citep{hairer1993solving}, and decode the result through $\Psi^{-1}$. Algorithms~\ref{alg:gvd_cfm_training} and~\ref{alg:gvd_cfm_sampling} in Appendix~\ref{app:experimental_details} state both procedures in full.

\subsection{Why Spectral Coordinates Simplify Trajectory Flow Matching}
\label{sec:spectral_theory_main}

The DCT reparametrizes the dependence structure of the complete log-Euclidean trajectory. Let $K_T\in\mathbb{R}^{B\times B}$ be the temporal covariance of the log-GVD trajectory. The Karhunen--Lo\`eve transform (KLT) diagonalizes $K_T$ exactly, and the orthonormal DCT is a fixed, data-independent approximation to the KLT for strongly correlated, temporally smooth processes \citep{ahmed1974dct,strang1999dct}. This has a direct consequence for conditional flow matching.

\begin{theorem}[Spectral decoupling of Gaussian trajectory flow]
\label{thm:spectral_cfm_main}
Let $x_1\sim\mathcal{N}(0,\Sigma)$ with $\Sigma=V\Lambda V^\top$, $x_0\sim\mathcal{N}(0,I)$, and $x_\tau=(1-\tau)x_0+\tau x_1$. In the covariance eigenbasis $y=V^\top x$, the population-optimal squared-error conditional flow-matching field is diagonal:
\begin{align}
    v_{\tau,k}^{\star}(y)
    &=
    \frac{\tau\lambda_k-(1-\tau)}{(1-\tau)^2+\tau^2\lambda_k}\,y_k
    &&\text{(independent coupling)},
    \label{eq:main_modewise_velocity}\\
    v_{\tau,k}^{\star}(y)
    &=
    \frac{\sqrt{\lambda_k}-1}{(1-\tau)+\tau\sqrt{\lambda_k}}\,y_k
    &&\text{(optimal-transport coupling)}.
    \nonumber
\end{align}
Hence, in the exact KLT basis, no cross-mode interaction is needed to represent the optimal second-order Gaussian transport under either coupling.
\end{theorem}

The proofs are given in Appendix~\ref{app:spectral_theory} (Corollaries~\ref{cor:appendix_modewise} and~\ref{cor:appendix_ot_modewise}). The residual off-diagonal temporal covariance measures how closely the DCT approximates the KLT; on BNCI2014\_001 the DCT reduces it from $0.778$ to $0.068$ (Figure~\ref{fig:dct_cfm_decoupling}). The DCT therefore removes most of the second-order temporal coupling before the Transformer is applied, while remaining invertible and isometric.

\section{Empirical Benchmarks}
\label{sec:benchmarks}

The main benchmark uses five two-class motor-imagery datasets from MOABB \citep{jayaram2018moabb,aristimunha2023moabb}: BNCI2014\_001, BNCI2014\_002, BNCI2015\_001, Shin2017A and Zhou2016. Every trial is band-pass filtered to 4--38~Hz, resampled to 128~Hz, standardized per channel and converted to a GVD trajectory of $B=100$ windows. The final session, or the final run when only runs are available, is held out, and all results are averaged over three generator seeds (Appendix~\ref{app:experimental_details}).

We compare against seven generators. Three are GVD-space controls that share the GVD targets and decoding pipeline of GVD-CFM and therefore isolate the contribution of the generative model: GVD-cVAE, GVD-DDPM and Window-DIFFEO-CFM, the last of which applies diffeomorphic flow matching \citep{collas2025diffeocfm} to each window independently. Four generate raw multichannel EEG, which is converted to a GVD trajectory using its own support and window covariances: JET \citep{wang2026jet}, a U-Net DDPM/DDIM \citep{ho2020ddpm,song2021ddim}, EEGGAN-2025 \citep{williams2025eeggan} and a conditional VAE \citep{sohn2015cvae}. No real support or window covariance is ever reused for a raw-EEG baseline.

We report a full-dimensional GVD Fr\'echet distance relative to the real-train to real-test reference \citep{dowson1982frechet,heusel2017fid}; EvaGeM $\alpha$-precision, $\beta$-recall and F1 \citep{alaa2022faithful}; the classification accuracy score (CAS) of a classifier trained only on generated trajectories and tested on held-out real ones \citep{ravuri2019cas}; held-out temporal diagnostics; and novelty measures \citep{kynkaanniemi2019precision,alaa2022faithful}. All are defined in Appendix~\ref{app:experimental_details}.

\section{Results}
\label{sec:results}

\begin{table*}[t]
\centering
\caption{Dataset-balanced generative performance across five EEG datasets, averaged over three generator seeds per dataset. Rel. GVD-FID and EvaGeM are evaluated jointly on trajectory positions and temporal increments. The real-data row reports the corresponding real-train/held-out-real-test reference and is excluded from generator rankings. Entries after $\pm$ are the seed standard deviation of the dataset-balanced mean, $\bigl(\sum_{i=1}^{5}\sigma_i^2\bigr)^{1/2}/5$, computed from the per-dataset seed standard deviations $\sigma_i$. Lower is better for Rel. GVD-FID; higher is better otherwise. Best generative values are \textbf{bold}; second-best generative values are \underline{\textit{underlined italics}}.}
\label{tab:main_generative_5ds}
\small
\setlength{\tabcolsep}{4.0pt}
\resizebox{\textwidth}{!}{%
\begin{tabular}{lcccccc}
\toprule
Method
& Rel. GVD-FID $\downarrow$
& Eva $\alpha$ $\uparrow$
& Eva $\beta$ $\uparrow$
& Eva F1 $\uparrow$
& CAS AUC $\uparrow$
& CAS F1 $\uparrow$ \\
\midrule
GVD-CFM
& \underline{\textit{1.022}}{\scriptsize$\pm$0.019}
& \textbf{0.677}{\scriptsize$\pm$0.034}
& \textbf{0.610}{\scriptsize$\pm$0.029}
& \textbf{0.613}{\scriptsize$\pm$0.028}
& \textbf{0.790}{\scriptsize$\pm$0.008}
& \textbf{0.725}{\scriptsize$\pm$0.006} \\

GVD-cVAE
& \textbf{0.815}{\scriptsize$\pm$0.003}
& 0.017{\scriptsize$\pm$0.005}
& 0.004{\scriptsize$\pm$0.002}
& 0.005{\scriptsize$\pm$0.002}
& \underline{\textit{0.727}}{\scriptsize$\pm$0.015}
& \underline{\textit{0.675}}{\scriptsize$\pm$0.010} \\

GVD-DDPM
& 1.753{\scriptsize$\pm$0.005}
& 0.089{\scriptsize$\pm$0.011}
& 0.039{\scriptsize$\pm$0.005}
& 0.053{\scriptsize$\pm$0.007}
& 0.599{\scriptsize$\pm$0.014}
& 0.560{\scriptsize$\pm$0.012} \\

Window-DIFFEO-CFM
& 1.697{\scriptsize$\pm$0.005}
& 0.005{\scriptsize$\pm$0.002}
& 0.082{\scriptsize$\pm$0.002}
& 0.009{\scriptsize$\pm$0.003}
& 0.683{\scriptsize$\pm$0.003}
& 0.634{\scriptsize$\pm$0.002} \\

cVAE
& 1.416{\scriptsize$\pm$0.016}
& \underline{\textit{0.407}}{\scriptsize$\pm$0.036}
& 0.006{\scriptsize$\pm$0.002}
& 0.012{\scriptsize$\pm$0.003}
& 0.631{\scriptsize$\pm$0.006}
& 0.572{\scriptsize$\pm$0.017} \\

JET
& 3.908{\scriptsize$\pm$0.274}
& 0.005{\scriptsize$\pm$0.002}
& 0.002{\scriptsize$\pm$0.002}
& 0.002{\scriptsize$\pm$0.002}
& 0.495{\scriptsize$\pm$0.031}
& 0.349{\scriptsize$\pm$0.014} \\

Vanilla-Diffusion
& 1.867{\scriptsize$\pm$0.096}
& 0.217{\scriptsize$\pm$0.042}
& \underline{\textit{0.366}}{\scriptsize$\pm$0.098}
& \underline{\textit{0.247}}{\scriptsize$\pm$0.056}
& 0.640{\scriptsize$\pm$0.012}
& 0.549{\scriptsize$\pm$0.029} \\

EEGGAN-2025
& 3.696{\scriptsize$\pm$0.016}
& 0.006{\scriptsize$\pm$0.002}
& 0.002{\scriptsize$\pm$0.001}
& 0.002{\scriptsize$\pm$0.001}
& 0.524{\scriptsize$\pm$0.031}
& 0.401{\scriptsize$\pm$0.017} \\

\midrule
Real data reference
& --
& 0.806
& 0.665
& 0.695
& 0.831
& 0.756 \\

\bottomrule
\end{tabular}}
\end{table*}

\begin{table*}[t]
\centering
\caption{Complementary data-quality diagnostics averaged across five datasets and three seeds. Dynamic and diversity ratios have ideal value 1. Temporal quantities are evaluated against held-out real trajectories after removing each trial's temporal mean, so the static support does not contribute to them. $\pm$ as in Table~\ref{tab:main_generative_5ds}. Best values are \textbf{bold}; second-best values are \underline{\textit{underlined italics}}. For ratio metrics, ranking is by proximity to 1.}
\label{tab:main_quality_5ds}
\scriptsize
\setlength{\tabcolsep}{4.0pt}
\resizebox{\textwidth}{!}{%
\begin{tabular}{lcccccc}
\toprule
Method
& Temp. corr. $\uparrow$
& Lag-ACF $\uparrow$
& Energy $\to 1$
& Dyn. frac. $\to 1$
& Diversity $\to 1$
& Coverage $\uparrow$ \\
\midrule
GVD-CFM
& \textbf{0.910}{\scriptsize$\pm$0.004}
& \underline{\textit{0.997}}{\scriptsize$\pm$0.001}
& \textbf{0.953}{\scriptsize$\pm$0.010}
& \textbf{0.987}{\scriptsize$\pm$0.013}
& \textbf{0.987}{\scriptsize$\pm$0.005}
& \textbf{0.226}{\scriptsize$\pm$0.005} \\

GVD-cVAE
& 0.774{\scriptsize$\pm$0.014}
& 0.995{\scriptsize$\pm$0.001}
& 0.260{\scriptsize$\pm$0.007}
& 0.289{\scriptsize$\pm$0.009}
& 0.396{\scriptsize$\pm$0.008}
& \underline{\textit{0.128}}{\scriptsize$\pm$0.004} \\

GVD-DDPM
& 0.059{\scriptsize$\pm$0.005}
& 0.327{\scriptsize$\pm$0.068}
& 1.688{\scriptsize$\pm$0.007}
& 1.543{\scriptsize$\pm$0.008}
& 1.367{\scriptsize$\pm$0.004}
& 0.038{\scriptsize$\pm$0.002} \\

Window-DIFFEO-CFM
& 0.012{\scriptsize$\pm$0.005}
& 0.055{\scriptsize$\pm$0.091}
& 1.794{\scriptsize$\pm$0.008}
& 1.755{\scriptsize$\pm$0.006}
& 1.189{\scriptsize$\pm$0.002}
& 0.031{\scriptsize$\pm$0.001} \\

cVAE
& 0.527{\scriptsize$\pm$0.016}
& 0.957{\scriptsize$\pm$0.006}
& 0.898{\scriptsize$\pm$0.006}
& \underline{\textit{0.981}}{\scriptsize$\pm$0.014}
& 0.817{\scriptsize$\pm$0.008}
& 0.034{\scriptsize$\pm$0.001} \\

JET
& 0.385{\scriptsize$\pm$0.019}
& 0.933{\scriptsize$\pm$0.009}
& 1.940{\scriptsize$\pm$0.118}
& 6.835{\scriptsize$\pm$0.338}
& 1.501{\scriptsize$\pm$0.079}
& 0.009{\scriptsize$\pm$0.004} \\

Vanilla-Diffusion
& \underline{\textit{0.783}}{\scriptsize$\pm$0.014}
& \textbf{0.998}{\scriptsize$\pm$0.001}
& \underline{\textit{0.941}}{\scriptsize$\pm$0.037}
& 1.759{\scriptsize$\pm$0.418}
& 1.070{\scriptsize$\pm$0.049}
& 0.085{\scriptsize$\pm$0.010} \\

EEGGAN-2025
& 0.022{\scriptsize$\pm$0.010}
& 0.306{\scriptsize$\pm$0.137}
& 0.609{\scriptsize$\pm$0.012}
& 4.864{\scriptsize$\pm$0.199}
& \underline{\textit{1.033}}{\scriptsize$\pm$0.019}
& 0.003{\scriptsize$\pm$0.000} \\

\bottomrule
\end{tabular}}
\end{table*}

\subsection{Generating High-Resolution Connectivity Trajectories}

Table~\ref{tab:main_generative_5ds} reports the main results; per-dataset values are given in Appendix~\ref{app:extended_tables}. In downstream classification GVD-CFM performs best by a large margin, with a CAS AUC of $0.790$ and CAS F1 of $0.725$ against $0.727$ and $0.675$ for the second-best model, GVD-cVAE, and it comes within $0.041$ AUC of a classifier trained on real data. GVD-cVAE attains the lowest relative GVD-FID, but its EvaGeM precision and recall are close to zero and its diversity ratio is $0.396$ (Table~\ref{tab:main_quality_5ds}), so its samples concentrate near the center of the distribution rather than covering it. GVD-CFM obtains the second-lowest relative GVD-FID ($1.022$) together with the highest EvaGeM $\alpha$-precision, $\beta$-recall and F1 of all models, reaching $88\%$ of the real-data EvaGeM F1. The raw-EEG generators obtain EvaGeM F1 below $0.25$ once their outputs are converted to GVD trajectories; we can see that realistic raw EEG does not imply realistic connectivity dynamics.

\subsection{Data Quality}
Table~\ref{tab:main_quality_5ds} shows that GVD-CFM gives the best overall balance between coverage and temporal fidelity. It achieves the highest held-out temporal-correlation agreement ($0.910$) and the second-highest lag-ACF agreement ($0.997$, against $0.998$ for Vanilla-Diffusion). Its dynamic-energy, dynamic-fraction and diversity ratios are the closest to $1$ of all generators, and its training-manifold coverage is the highest. The GVD-space controls show why joint spectral modeling matters. GVD-DDPM and Window-DIFFEO-CFM, which model temporal windows directly, obtain temporal-correlation agreement below $0.06$; GVD-cVAE retains temporal correlation but collapses dynamic energy to $0.260$ of the real value. GVD-CFM produces no exact copies of training trajectories (Appendix~\ref{app:additional_ablations}).

\paragraph{Static support versus generated dynamics.}
Every GVD window carries the trial-level support $W$, so class information could reside in static structure alone. The generated trajectories carry dynamics beyond it. The temporal diagnostics above remove each trial's temporal mean and are insensitive to $W$; a permutation test rejects exchangeability of the generated windows at $p=0.002$, the smallest attainable value with 500 permutations, on every dataset and seed; and on real data the full GVD trajectory outperforms the static support alone in held-out classification, while a support modulated by Gaussian noise falls to near chance (Figure~\ref{fig:temporal_resolution_ablation}).

\begin{table*}[t]
\centering
\caption{DCT and temporal-branch ablation. \emph{No DCT} trains the Transformer on temporal log-svec coordinates; \emph{DCT, spectral only} is GVD-CFM without the temporal branch; the last row is the full GVD-CFM. All values are dataset-balanced means over five datasets and three generator seeds; per-dataset values are in Tables~\ref{tab:appendix_dct_ablation_5ds} and~\ref{tab:appendix_dct_ablation_temporal_5ds}, and paired per-dataset tests in Table~\ref{tab:appendix_paired_ablation}. Ratios have ideal value 1. $\pm$ as in Table~\ref{tab:main_generative_5ds}. Best values are \textbf{bold}; second-best values are \underline{\textit{underlined italics}}.}
\label{tab:main_dct_ablation_5ds}
\small
\setlength{\tabcolsep}{3.8pt}
\resizebox{\textwidth}{!}{%
\begin{tabular}{lcccccccc}
\toprule
Variant
& Rel. GVD-FID $\downarrow$
& Eva F1 $\uparrow$
& CAS AUC $\uparrow$
& CAS F1 $\uparrow$
& Temp. corr. $\uparrow$
& Lag-ACF $\uparrow$
& Energy $\to1$
& Dyn. frac. $\to1$ \\
\midrule
No DCT
& 1.092{\scriptsize$\pm$0.029}
& 0.540{\scriptsize$\pm$0.028}
& 0.772{\scriptsize$\pm$0.007}
& 0.704{\scriptsize$\pm$0.007}
& 0.404{\scriptsize$\pm$0.007}
& \underline{\textit{0.666}}{\scriptsize$\pm$0.069}
& \underline{\textit{0.973}}{\scriptsize$\pm$0.016}
& \underline{\textit{1.011}}{\scriptsize$\pm$0.017} \\

DCT, spectral only
& \underline{\textit{1.025}}{\scriptsize$\pm$0.006}
& \underline{\textit{0.600}}{\scriptsize$\pm$0.019}
& \textbf{0.792}{\scriptsize$\pm$0.002}
& \underline{\textit{0.721}}{\scriptsize$\pm$0.004}
& \textbf{0.919}{\scriptsize$\pm$0.002}
& \textbf{0.997}{\scriptsize$\pm$0.001}
& \textbf{0.975}{\scriptsize$\pm$0.011}
& \textbf{1.010}{\scriptsize$\pm$0.010} \\

GVD-CFM (DCT + temporal branch)
& \textbf{1.022}{\scriptsize$\pm$0.019}
& \textbf{0.613}{\scriptsize$\pm$0.028}
& \underline{\textit{0.790}}{\scriptsize$\pm$0.008}
& \textbf{0.725}{\scriptsize$\pm$0.006}
& \underline{\textit{0.910}}{\scriptsize$\pm$0.004}
& \textbf{0.997}{\scriptsize$\pm$0.001}
& 0.953{\scriptsize$\pm$0.010}
& 0.987{\scriptsize$\pm$0.013} \\
\bottomrule
\end{tabular}}
\end{table*}

\subsection{Ablations}
\label{sec:timecontext_ablation}

\paragraph{The DCT gives large and consistent gains.}
Table~\ref{tab:main_dct_ablation_5ds} separates the spectral representation from the temporal branch. Moving from temporal coordinates to the DCT chart raises temporal-correlation agreement from $0.404$ to $0.919$ and lag-ACF agreement from $0.666$ to $0.997$, in line with Theorem~\ref{thm:spectral_cfm_main}. Both improvements hold on all five datasets (Welch $t$ from $4.8$ to $80.9$ and from $3.4$ to $13.0$; Table~\ref{tab:appendix_paired_ablation}). Under the same spectral-only network, a random orthogonal basis does not reproduce this gain, while the empirical KLT performs comparably to the DCT (Table~\ref{tab:basis_ablation_summary}). The benefit therefore comes from temporal decorrelation rather than from orthogonality alone.

\paragraph{The temporal branch corrects amplitude.}
Adding the temporal branch changes the dataset-balanced means by amounts comparable to their seed variation (EvaGeM F1 $0.600\to0.613$, CAS F1 $0.721\to0.725$, CAS AUC $0.792\to0.790$, temporal correlation $0.919\to0.910$), so we treat these differences as ties. Its intended effect is on amplitude: dynamic energy decreases on every dataset and moves towards $1$ where the spectral-only network overshoots, as we show in
Appendix \ref{app:time_context_visual}.

\paragraph{The stable support improves downstream utility.}
We compared GVD-CFM with the same generator trained on ridge-regularized window covariances $J_b+10^{-6}I$, since without the support a ridge is needed to make the windows SPD (per-dataset values in Table~\ref{tab:appendix_support_ablation_5ds}; $\pm$ as in Table~\ref{tab:main_generative_5ds}). The stable support raises CAS AUC on every dataset; the dataset-balanced CAS AUC rises from $0.736{\scriptstyle\pm0.005}$ to $0.790{\scriptstyle\pm0.008}$, CAS F1 from $0.679{\scriptstyle\pm0.005}$ to $0.725{\scriptstyle\pm0.006}$ and EvaGeM F1 from $0.436{\scriptstyle\pm0.015}$ to $0.613{\scriptstyle\pm0.028}$.

\subsection{Regional Physiological Plausibility}
\label{sec:regional_main}
\begin{figure*}[t]
    \centering
    \includegraphics[width=0.94\textwidth]{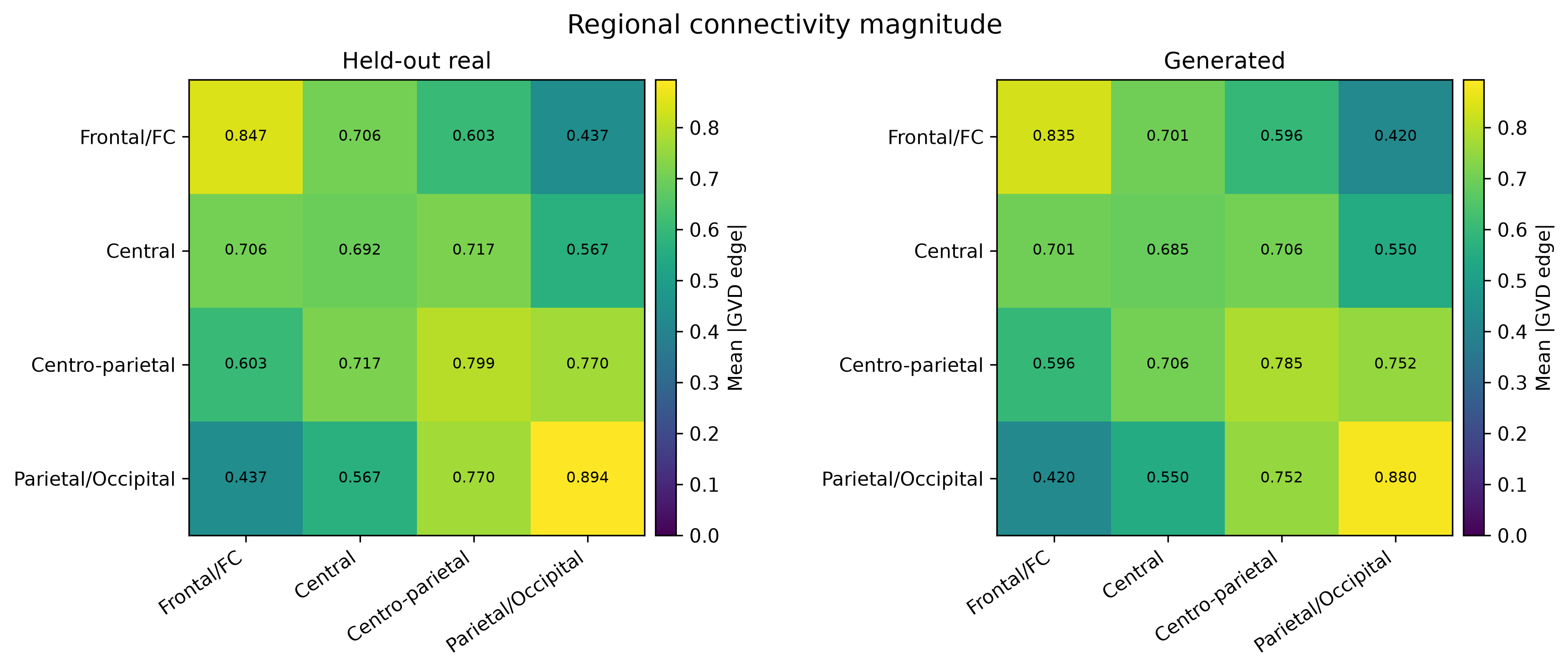}
    \caption{
    \textbf{Regional connectivity magnitude on BNCI2014\_001.}
    Mean absolute GVD edge magnitude between broad scalp regions for held-out
    real trials and generated trials. The generated data preserve the relative
    regional organization of the real EEG, including strong within-region
    Frontal/FC and Parietal/Occipital connectivity, strong
    Central--Centro-parietal coupling, and weaker long-range
    Frontal/FC--Parietal/Occipital interactions.
    }
    \label{fig:regional_gvd_matrix}
\end{figure*}

Global metrics do not show whether generated connectivity is organized in space and time as in real EEG. We grouped the 22 electrodes of BNCI2014\_001 into Frontal/FC, Central, Centro-parietal and Parietal/Occipital regions and computed the mean GVD edge magnitude for every pair of regions (Appendix~\ref{app:regional_plausibility}). Generated trajectories keep the regional organization of held-out data. Within-region connectivity is $0.835$ against $0.847$ for real data in Frontal/FC and $0.880$ against $0.894$ in Parietal/Occipital, Central--Centro-parietal coupling is $0.706$ against $0.717$, and the long-range Frontal/FC--Parietal/Occipital coupling stays weak ($0.420$ against $0.437$). Figure~\ref{fig:regional_gvd_dynamics} shows that the time course of these interactions is also preserved, and that its changes occur together across central, centro-parietal, frontal and posterior interactions. The generated dynamics are therefore organized around the sensorimotor regions expected to take part in motor imagery \citep{pfurtscheller1999erd}.

\subsection{Continuous-Grid Decoding and Efficiency}

The generated DCT coefficients define a band-limited cosine trajectory, so they can be decoded on any temporal grid without retraining. Training the spectral-only network once at $B_{\mathrm{train}}=25$ windows and decoding at $M=100$ changes generated-to-real AUC only from $0.803$ to $0.802$; at $M=200$ it remains $0.794$ with lag-ACF agreement $0.959$, and every decoded matrix stays SPD. At $M=400$ temporal-correlation agreement falls to $0.506$, the bandwidth limit of $25$ modes (Appendix~\ref{app:continuous_resolution_full}). GVD-CFM trains in about five minutes per dataset and achieves the best efficiency trade off (Appendix Figure~\ref{fig:efficiency_ablation}).

\section{Conclusions and Limitations}
We introduced GVD-CFM, a conditional flow matching model for generating high-temporal-resolution DFC trajectories from EEG. When the stable support is positive definite, GVD trajectories lie on a product Riemannian manifold even when individual window covariances are rank deficient. This allows Euclidean flow matching under a global diffeomorphism to be exactly equivalent to Riemannian flow matching on the original manifold. We further showed that applying the DCT approximately decorrelates the temporal structure of these trajectories, leading to large and consistent improvements in temporal fidelity. Across five EEG datasets and seven baselines, GVD-CFM provides the best overall trade-off across the reported metrics. The generated trajectories remain on the manifold, show no evidence of memorizing the training data, and preserve both the regional organization and temporal evolution of motor-imagery connectivity. Together, these results suggest that directly modeling connectivity trajectories is a practical alternative to generating raw EEG when the downstream quantity of interest is dynamic functional connectivity.

Our work also has several limitations. First, the number of log-Euclidean coordinates per window grows quadratically with the number of channels, which may limit scalability to substantially higher-density recordings. Second, GVD-CFM ultimately depends on the quality of the underlying GVD estimator. Future work could explore using phase- or coherence-based node functions \citep{smith2019gvsa,roy2025hodgefast}, and could explore alternative diffeomorphisms with different geometric or computational properties \citep{lin2019cholesky,david2019correlation,thanwerdas2022thesis}. It would also be useful to test whether the same framework generalizes beyond motor-imagery EEG and to other modalities where dynamic connectivity is of interest, particularly fMRI \citep{hutchison2013dfc,allen2014tracking,preti2017dynamic}. More broadly, extending GVD-CFM to larger networks, frequency-resolved connectivity, and other forms of neural dynamics would help establish how well the approach scales beyond the setting studied here.

\section*{Reproducibility Statement}

All five datasets are public and are obtained through MOABB
\citep{jayaram2018moabb,aristimunha2023moabb}. Appendix~\ref{app:experimental_details}
gives the preprocessing, the GVD construction, the log-Euclidean and DCT
chart, the architecture and optimizer settings of GVD-CFM, including the
temporal branch and the minibatch coupling, the configuration of every
baseline, and the definition of every metric. Training and sampling are stated
as Algorithms~\ref{alg:gvd_cfm_training} and~\ref{alg:gvd_cfm_sampling}. All
reported experiments use three generator seeds and three classifier seeds,
stated in the appendix, and exact train--test duplication is checked by
hashing before any model is fitted. Code is included in the supplementary material.

\section*{Use of Large Language Models}

Large language models were used for proofreading, notation consistency, and
formatting. They were also used to check correctness of mathematical proofs.  All scientific claims, mathematical arguments, experimental
choices, implementations, and reported results remain the responsibility of
the authors.


\small

\appendix

\section{Experimental Details}
\label{app:experimental_details}

\subsection{Datasets and evaluation protocol}

Experiments were conducted on five motor-imagery EEG datasets distributed
through MOABB \citep{jayaram2018moabb,aristimunha2023moabb}. The main
benchmark uses
\texttt{BNCI2014\_001} \citep{tangermann2012bci,brunner2008graz},
\texttt{BNCI2014\_002} \citep{steyrl2016randomforests},
\texttt{BNCI2015\_001} \citep{faller2012autocalibration},
\texttt{Shin2017A},
and
\texttt{Zhou2016} \citep{zhou2016trial}, with all configured subjects for each
dataset. The corresponding subject counts were 9, 14, 12, 29, and 4,
respectively.

All reported experiments used a fixed cross-session or cross-run evaluation
protocol. For each subject, when multiple sessions were available, the final
session was reserved for testing and all preceding sessions were used for
training. When multiple sessions were unavailable but multiple runs were
present, the final run was held out instead. If neither structure was available,
a stratified 50/50 split was used as a fallback. The data split was fixed across
generator seeds.

All representation statistics, normalization parameters, model parameters,
source-distribution statistics, and downstream classifiers were estimated
using training data only. Exact duplicate trials between the training and
held-out sets were explicitly checked, and execution was terminated if
train--test leakage was detected.

All generative experiments were repeated using three generator seeds,
\begin{equation}
    s_{\mathrm{gen}}\in\{1,2,3\},
\end{equation}
and CAS evaluation used three independent classifier seeds,
\begin{equation}
    s_{\mathrm{CAS}}\in\{9001,9002,9003\}.
\end{equation}

\subsection{EEG preprocessing}

Only EEG channels were retained; non-EEG channels were discarded. Signals
were converted to microvolts, band-pass filtered from 4 to 38~Hz, and
resampled to 128~Hz \citep{gramfort2013mne}. Event-aligned trials were then
extracted using the complete event interval provided by each dataset.

Before GVD construction, every EEG trial was standardized independently for
each channel over the complete temporal duration of the trial. For channel
$c$,
\begin{equation}
    \widetilde u_{c,t}
    =
    \frac{
        u_{c,t}-\mu_c
    }{
        \max(\sigma_c,10^{-6})
    },
\end{equation}
where
\begin{equation}
    \mu_c
    =
    \frac{1}{T}
    \sum_{t=1}^{T}u_{c,t},
\end{equation}
and
\begin{equation}
    \sigma_c
    =
    \sqrt{
        \frac{1}{T}
        \sum_{t=1}^{T}
        (u_{c,t}-\mu_c)^2
    }.
\end{equation}

Standardization was performed over time within each channel and trial.

Two motor-imagery classes were retained for each binary experiment.

\subsection{Graph-variate dynamic connectivity construction}

Each standardized EEG trial was transformed into a high-resolution
graph-variate dynamic (GVD) connectivity trajectory. Let
\begin{equation}
    U=
    \begin{bmatrix}
        u_1 & \cdots & u_T
    \end{bmatrix}
    \in\mathbb{R}^{d\times T},
\end{equation}
where
\begin{equation}
    u_t\in\mathbb{R}^{d}
\end{equation}
is the vector of standardized channel amplitudes at EEG sample $t$.

\paragraph{Stable trial-level support.}

The long-term support matrix was the signed whole-trial Pearson correlation
matrix,
\begin{equation}
    W
    =
    \frac{1}{T}
    UU^\top.
    \label{eq:gvd_support}
\end{equation}

Because every channel has already been centered and normalized over the
complete trial, Equation~\ref{eq:gvd_support} is the signed channel
correlation matrix under the population-standard-deviation convention used
by the implementation.

No absolute-value operation, additive ridge, or nearest-SPD projection was
applied to the support matrix in the reported configuration. Because every
channel is standardized over the trial, $W$ has unit diagonal and
$\operatorname{rank}(W)=\operatorname{rank}(U)$, which is full whenever the
$T\gg d$ samples span $\mathbb{R}^d$. The strict positive definiteness of every
$W$ is verified in double precision before construction of the trajectory, and
execution terminates if the check fails. No trial in any dataset failed.

\paragraph{Sample-resolution instantaneous interaction.}

At every original EEG sample, the instantaneous interaction matrix was
defined as the rank-one outer product
\begin{equation}
    J_t
    =
    u_tu_t^\top.
    \label{eq:instantaneous_outer}
\end{equation}

The sample-resolution graph-variate matrix was then
\begin{equation}
    \Delta_t
    =
    W\odot J_t,
    \label{eq:sample_gvd}
\end{equation}
where $\odot$ denotes the Hadamard product.

Using
\begin{equation}
    D_t=\operatorname{diag}(u_t),
\end{equation}
Equation~\ref{eq:sample_gvd} can equivalently be written as
\begin{equation}
    \Delta_t
    =
    D_t W D_t.
    \label{eq:gvd_congruence}
\end{equation}

\paragraph{Temporal aggregation.}

Each trial was partitioned into
\begin{equation}
    B=100
\end{equation}
disjoint full-coverage temporal bins. Their boundaries were
\begin{equation}
    e_b
    =
    \left\lfloor
        \frac{bT}{B}
    \right\rfloor,
    \qquad
    b=0,\ldots,B,
\end{equation}
so that every original EEG sample belongs to exactly one bin and the complete
trial is covered.

Let
\begin{equation}
    \mathcal{I}_b
    =
    \{e_{b-1},\ldots,e_b-1\}
\end{equation}
denote the sample indices in bin $b$. The reported GVD trajectory was obtained
by averaging the already Hadamard-modulated sample-resolution matrices (this is $\overline{\Delta}_b$ in the main text; we drop the bar in the appendix):
\begin{equation}
    \Delta_b
    =
    \frac{1}{|\mathcal{I}_b|}
    \sum_{t\in\mathcal{I}_b}
    \left(
        W\odot u_tu_t^\top
    \right).
    \label{eq:binned_gvd}
\end{equation}

Since $W$ is constant within a trial, this is equivalently
\begin{equation}
    \Delta_b
    =
    W\odot
    \left(
        \frac{1}{|\mathcal{I}_b|}
        \sum_{t\in\mathcal{I}_b}
        u_tu_t^\top
    \right),
\end{equation}
which is the form used in the implementation.

No local re-centering was performed inside a temporal bin, because centering
and scaling had already been carried out over the complete trial. Likewise,
no covariance ridge or additive GVD ridge was used.

The resulting matrices were symmetrized numerically and their minimum
eigenvalues were evaluated in double precision. If a trajectory failed the
strict positive-definiteness check, execution terminated rather than applying
an additive ridge or post-hoc nearest-SPD correction.

The canonical GVD representation used throughout the benchmark is therefore
\begin{equation}
    \boxed{
    \Delta_b
    =
    \frac{1}{|\mathcal{I}_b|}
    \sum_{t\in\mathcal{I}_b}
    W\odot u_tu_t^\top
    }.
    \label{eq:canonical_gvd}
\end{equation}

\subsection{Log-Euclidean trajectory representation}

Each SPD GVD matrix was mapped to the log-Euclidean chart
\citep{arsigny2007logeuclidean}:
\begin{equation}
    z_b
    =
    \operatorname{svec}
    \left(
        \log\Delta_b
    \right)
    \in\mathbb{R}^{m},
\end{equation}
where
\begin{equation}
    m=\frac{d(d+1)}{2}.
\end{equation}

The $\operatorname{svec}$ operator contains the lower-triangular entries of a
symmetric matrix, with off-diagonal elements multiplied by $\sqrt{2}$. This
preserves the Frobenius inner product under vectorization.

The complete trajectory was stacked as
\begin{equation}
    Z
    =
    \begin{bmatrix}
        z_1^\top\\
        \vdots\\
        z_B^\top
    \end{bmatrix}
    \in
    \mathbb{R}^{B\times m}.
\end{equation}

For every log-svec feature $j$, the normalization statistics were estimated
using the real training trajectories only and pooled across training trials
and temporal bins:
\begin{equation}
    \mu_j
    =
    \frac{1}{N_{\mathrm{tr}}B}
    \sum_{n=1}^{N_{\mathrm{tr}}}
    \sum_{b=1}^{B}
    Z_{n,b,j},
\end{equation}
and
\begin{equation}
    \sigma_j
    =
    \operatorname{Std}_{n,b}
    \left[
        Z_{n,b,j}
    \right].
\end{equation}

A minimum scale of $10^{-6}$ was used,
\begin{equation}
    \sigma_j
    \leftarrow
    \max(\sigma_j,10^{-6}),
\end{equation}
and the standardized coordinates were
\begin{equation}
    \widetilde Z_{n,b,j}
    =
    \frac{
        Z_{n,b,j}-\mu_j
    }{
        \sigma_j
    }.
    \label{eq:log_svec_standardization}
\end{equation}

The same train-estimated statistics were used for all direct GVD-space
generators.

\paragraph{Full spectral representation.}

For GVD-CFM, a full orthonormal DCT-II was subsequently applied along the
temporal axis:
\begin{equation}
    \bar Z
    =
    C_B\widetilde Z,
    \qquad
    C_B^\top C_B
    =
    C_BC_B^\top
    =
    I_B.
    \label{eq:dct_representation}
\end{equation}

All $B=100$ DCT modes were retained. The DCT therefore performs no
dimensionality reduction; it is an invertible orthogonal reparameterization
of the complete temporal trajectory.

The inverse transformation is
\begin{align}
    \widetilde Z
    &=
    C_B^\top\bar Z,\\
    Z
    &=
    \widetilde Z\odot\sigma+\mu,\\
    \Delta_b
    &=
    \exp
    \left(
        \operatorname{svec}^{-1}(z_b)
    \right).
\end{align}

For generated trajectories, extreme log-eigenvalues were stabilized before
matrix exponentiation. The lower and upper bounds were estimated exclusively
from the real training log-spectrum using the $0.001$ and $0.999$ quantiles,
respectively, and each bound was expanded by a margin of $0.5$. This
stabilization was applied only in the log domain and did not modify the real
training or held-out trajectories.

\subsection{Generative models}

The benchmark contained eight generators:
\begin{enumerate}
    \item GVD-CFM;
    \item GVD-cVAE;
    \item GVD-DDPM;
    \item Window-DIFFEO-CFM \citep{collas2025diffeocfm};
    \item JET \citep{wang2026jet};
    \item Vanilla Diffusion / DDIM \citep{ho2020ddpm,song2021ddim};
    \item EEGGAN-2025 \citep{williams2025eeggan,williams2023augmenting};
    \item conditional VAE \citep{kingma2014vae,sohn2015cvae};
\end{enumerate}

The first four models generate GVD trajectories directly. The remaining four generate raw multichannel EEG, after which each synthetic EEG trial is independently transformed into a GVD trajectory using its own support matrix $W$ and its own window covariances $J_b$. The quantitative tables report GVD-CFM and seven comparators.

All trainable generators were trained for 1000 epochs. The training batch size was scaled with the number of available training trials:
\begin{equation}
N_{\mathrm{batch}}
=
\max\left(
1,
\min\left(
N,
\operatorname{round}
\left[
64\frac{N}{1000}
\right]
\right)
\right).
\end{equation}
This gives a batch size of 64 for 1000 training trials and keeps the number of optimizer batches per epoch approximately constant across datasets.


%


\subsection{GVD-CFM}

GVD-CFM generates complete dynamic GVD trajectories in the composite log-Euclidean/DCT coordinate system. The network operates on
\begin{equation}
\bar Z\in\mathbb{R}^{B\times m},
\end{equation}
where each of the $B$ tokens corresponds to one DCT mode and each token contains the $m$ log-\texttt{svec} connectivity coordinates.

The velocity network consists of a spectral Transformer and a small temporal branch that share the conditioning vector. The spectral branch is an AdaLN Transformer \citep{vaswani2017attention,peebles2023dit} with model dimension 256, 8 attention heads, and 6 Transformer blocks acting on the $B$ DCT tokens. Each block contains multi-head self-attention followed by a feed-forward network with expansion factor 4 and GELU activation \citep{hendrycks2016gelu}. The input and output projections map between the $m$-dimensional GVD coordinate and the 256-dimensional Transformer state.

The temporal branch receives $C_B^\top z_\tau$, the inverse-DCT view of the current state, as $B$ window tokens. It projects each window to the same 256-dimensional width, adds a window-position embedding, and applies 2 AdaLN Transformer blocks of the same width and number of heads, conditioned on the same flow-time and class vector. Its output is multiplied by $C_B$ to return to DCT-token alignment and fused with the spectral hidden states through the tokenwise sigmoid gate of Equation~\ref{eq:gated_fusion} before the shared output projection. No auxiliary loss is applied to the temporal branch. Including both branches, the network has between $9.89$ and $10.18$ million parameters, depending on the number of EEG channels.

DCT-mode identity is represented using a continuous learned embedding. Flow time is encoded using 16 random Fourier frequencies \citep{tancik2020fourier} followed by two fully connected SiLU layers. The class condition is represented by a learned embedding and added to the flow-time representation before adaptive layer-normalization modulation.
\paragraph{Source distribution.}

The default GVD-CFM source is an isotropic standard normal distribution in the complete DCT trajectory space. Specifically, for each generated trajectory,
\begin{equation}
z_0 \sim \mathcal{N}(0,I),
\end{equation}
where $z_0$ has the same dimensionality as the vectorized DCT representation of the target GVD trajectory. Equivalently,
\begin{equation}
z_0 = \epsilon,
\qquad
\epsilon \sim \mathcal{N}(0,I).
\end{equation}

The source distribution is independent of the class label and is not estimated from the training data. Class information is instead supplied to the conditional flow model through the class-conditioning mechanism. Thus, all classes share the same standard Gaussian source, while the learned conditional vector field transports samples toward the corresponding class-conditional distribution of GVD trajectories.
\paragraph{Minibatch coupling.}
Source and target samples are paired classwise. For every class present in a training minibatch, a fresh standard-normal draw of the same size is matched to that class's target trajectories by an entropic optimal-transport plan computed with Sinkhorn iterations \citep{cuturi2013sinkhorn} on the squared Euclidean distance between complete DCT-coordinate trajectories, and the minibatch is re-paired according to this plan \citep{tong2024improving}. Since the orthonormal DCT is an isometry, this is the squared distance between standardized log-Euclidean trajectories. Coupling within class ensures that every source sample is paired with a target of the label on which the velocity is conditioned.
\paragraph{Flow-matching objective.}

For each training example,
\begin{equation}
\tau\sim\mathcal U(0,1),
\end{equation}
$(z_0,z_1)$ is drawn from the classwise minibatch coupling, and a straight conditional probability path is used:
\begin{equation}
z_\tau
=
(1-\tau)z_0+\tau z_1,
\end{equation}
with target velocity
\begin{equation}
u_\tau
=
z_1-z_0.
\end{equation}
The model is optimized using an $\ell_2$ conditional flow-matching objective,
\begin{equation}
\mathcal L_{\mathrm{CFM}}
=
\left\|
v_\theta(z_\tau,\tau,y)
-
u_\tau
\right\|_2^2.
\end{equation}
AdamW \citep{loshchilov2019adamw} is used with learning rate $5\times10^{-4}$ and weight decay $10^{-4}$. Gradient norms are clipped at 1.0. Training minibatches are sampled with inverse class-frequency weighting. 
\paragraph{Sampling.}

Sampling starts from the standard-normal source and integrates the learned velocity field from flow time 0 to 1. The default sampler is fourth-order Runge--Kutta with 50 uniform integration steps, so each trajectory requires 200 evaluations of the velocity network. Sampling is performed in batches of at most 2048 trajectories.

After integration, DCT coefficients are transformed back to temporal log-\texttt{svec} coordinates using the inverse DCT, reversed through the training-set affine standardization, stabilized in the log-spectrum, and exponentiated to obtain SPD GVD trajectories.

\subsection{Direct GVD-space control models}

Three additional generators operate directly on the same GVD targets as GVD-CFM:
GVD-cVAE, GVD-DDPM, and Window-DIFFEO-CFM.

These controls never generate raw EEG. They therefore isolate the contribution of the generative model from that of the GVD representation itself. All three use the same training GVD matrices, log-Euclidean $\operatorname{svec}(\log(\cdot))$ representation, training-set standardization, generated log-spectrum stabilization, and SPD decoding procedure as GVD-CFM. However, unlike GVD-CFM, they operate directly on the temporal sequence of standardized log-Euclidean GVD coordinates and do not transform the trajectories into DCT modes.

GVD-cVAE and GVD-DDPM receive exactly the same training-set standardized coordinates as GVD-CFM. Thus, each trial is represented directly as a sequence of standardized log-\texttt{svec} GVD vectors, and no temporal DCT or inverse-DCT operation is used by either baseline.

\subsubsection{GVD-cVAE}

GVD-cVAE is a conditional VAE \citep{kingma2014vae,sohn2015cvae} defined on the complete temporal GVD trajectory in standardized log-Euclidean coordinates.

The sequence of temporal log-\texttt{svec} vectors is flattened into a single trial-level representation and concatenated with a one-hot class vector. The encoder contains two fully connected layers of width 512 with GELU activations. The latent representation has dimension 64 and is parameterized by separate mean and log-variance heads.

The decoder concatenates the sampled latent representation with the one-hot class vector and applies two width-512 GELU layers followed by a linear output layer spanning the complete temporal GVD trajectory.

The objective is
\begin{equation}
\mathcal L_{\mathrm{GVD-cVAE}}
=
\mathcal L_{\mathrm{MSE}}
+
10^{-3}\mathcal L_{\mathrm{KL}}.
\end{equation}
The model uses AdamW with learning rate $10^{-3}$, weight decay $10^{-4}$, and gradient-norm clipping at 5.0.

At generation time,
\begin{equation}
z\sim\mathcal N(0,I_{64}),
\end{equation}
is sampled and passed to the class-conditional decoder. The generated temporal log-Euclidean trajectory is inverse-standardized and mapped directly back to a sequence of SPD GVD matrices through the same matrix-exponential decoding path used by GVD-CFM. 

\subsubsection{GVD-DDPM}

GVD-DDPM is an $\epsilon$-prediction diffusion model \citep{ho2020ddpm} operating directly on the complete temporal sequence of standardized log-Euclidean GVD coordinates.

Its noise-prediction network is a pre-norm Transformer encoder with width 256, 8 attention heads, and 6 layers, matching the width, number of heads, and depth of the GVD-CFM spectral Transformer. Each layer contains multi-head self-attention followed by a feed-forward network with expansion factor 4 and GELU activation, without dropout. The Transformer tokens correspond to temporal GVD windows. Each token is a learned projection of one standardized log-\texttt{svec} window, to which a Fourier embedding of the diffusion step, a window-position embedding, and a learned class embedding are added. Consequently, the model operates directly on the temporal log-\texttt{svec} trajectory.

The diffusion process contains 200 steps with a linear variance schedule
\begin{equation}
\beta_1=10^{-4},
\qquad
\beta_{200}=2\times10^{-2}.
\end{equation}
At a randomly selected diffusion step $\ell$, noise
\begin{equation}
\epsilon\sim\mathcal N(0,I)
\end{equation}
is added according to the standard forward diffusion process. The network is trained with
\begin{equation}
\mathcal L_{\mathrm{DDPM}}
=
\left\|
\epsilon_\theta(x_\ell,\ell,y)
-
\epsilon
\right\|_2^2.
\end{equation}
AdamW uses learning rate $2\times10^{-4}$, weight decay $10^{-4}$, and gradient clipping at 5.0. Training and sampling are performed in single precision.

Generation uses full ancestral DDPM sampling over all 200 diffusion steps. The resulting temporal coordinates are inverse-standardized and decoded directly through the log-Euclidean inverse map to obtain a sequence of SPD GVD matrices.

\subsubsection{Window-DIFFEO-CFM}

Window-DIFFEO-CFM \citep{collas2025diffeocfm} is a per-window flow-matching control defined in the same standardized log-Euclidean GVD space. Unlike GVD-CFM and the other whole-trajectory controls, it does not model the complete temporal trajectory jointly and has no communication between different temporal windows.

Each standardized log-\texttt{svec} GVD window is treated as an independent training sample. The network is a conditional MLP with one hidden layer of width 128 and SELU activation. Its input contains:
\begin{enumerate}
    \item the current noisy GVD window;
    \item a one-hot class vector;
    \item the normalized temporal window location $\xi_b\in[0,1]$;
    \item the flow time.
\end{enumerate}

For each window,
\begin{equation}
x_0\sim\mathcal N(0,I),
\end{equation}
and the conditional interpolation path is
\begin{equation}
x_\tau=(1-\tau)x_0+\tau x_1.
\end{equation}
The model is trained with the MSE velocity objective
\begin{equation}
\mathcal L_{\mathrm{Window-CFM}}
=
\left\|
v_\theta(x_\tau,\tau,y,\xi_b)
-
(x_1-x_0)
\right\|_2^2.
\end{equation}
AdamW uses learning rate $10^{-3}$ with zero weight decay and gradient clipping at 5.0. Sampling uses RK4 with 50 steps independently for each temporal window.

This baseline therefore tests whether matching the marginal distribution of each temporal GVD window independently is sufficient, in contrast to jointly modeling the full temporal trajectory with cross-window context.

\subsection{Raw-EEG generators}

The remaining four methods generate raw multichannel EEG. Their generated signals are not compared directly with GVD-CFM in raw-signal space. Instead, every synthetic EEG trial is passed through the same GVD construction used for real data.

For every generated trial $U^{(g)}$, its own support matrix is computed:
\begin{equation}
W^{(g)}
=
\operatorname{corr}
\left(
U^{(g)}
\right),
\end{equation}
and its own window covariance sequence
\begin{equation}
J^{(g)}_1,\ldots,J^{(g)}_B
\end{equation}
is constructed. The final generated GVD trajectory is
\begin{equation}
\Delta^{(g)}_b
=
W^{(g)}\odot J^{(g)}_b.
\end{equation}
Thus, no real-data $W$ or $J_b$ is reused for a raw-EEG baseline.

\subsubsection{JET}

JET \citep{wang2026jet} is evaluated using the official Y-Research-SBU implementation and the JiT-B/16 configuration. Raw EEG trials are padded at the input boundary to a length divisible by a patch size of 200 samples. The number of input tokens is therefore determined by the number of EEG channels multiplied by the number of temporal patches.

The benchmark retains the official JET denoiser, objective, EMA updates, and sampling equations. The configuration uses a class-conditional model, label dropout probability 0.1, $P_{\mathrm{mean}}=-0.8$, $P_{\mathrm{std}}=0.8$, Gaussian noise, and the mixed loss configuration from the release. The enabled auxiliary loss terms are statistical loss with weight 1.0, total-variation loss with weight 0.1, and correlation loss with weight 0.1; the STFT loss weight is 0.

Two EMA decay factors, 0.9999 and 0.9996, are retained from the implementation. Sampling uses the Heun solver with 50 steps.

The JET base learning rate is $5\times10^{-5}$ and is scaled by
\begin{equation}
\mathrm{lr}
=
5\times10^{-5}
\frac{N_{\mathrm{batch}}}{256}.
\end{equation}
AdamW uses $\beta=(0.9,0.95)$ and zero weight decay.

\subsubsection{Vanilla Diffusion / DDIM}

The diffusion baseline follows the released Song, Meng, and Ermon DDPM/DDIM implementation \citep{ho2020ddpm,song2021ddim}. The original model uses two-dimensional convolution for images; only the convolutional operators are ported to one-dimensional convolution so that the model operates directly on multichannel EEG.

No class embedding is introduced because the released architecture is unconditional. Instead, one diffusion model is trained separately for each class.

The forward diffusion process uses 1000 steps and a linear beta schedule from
\begin{equation}
10^{-4}
\quad\text{to}\quad
2\times10^{-2}.
\end{equation}
The U-Net uses base width 128, channel multipliers
\begin{equation}
(1,2,2,2),
\end{equation}
two residual blocks per resolution, and dropout 0.1.

Training uses standard DDPM noise prediction with learning rate $2\times10^{-4}$. Sampling uses the generalized DDIM sampler with 50 sampling steps and
\begin{equation}
\eta=0,
\end{equation}
corresponding to deterministic DDIM sampling.

Input EEG is scaled to the training-derived $[-1,1]$ range before diffusion and transformed back afterward.

\subsubsection{EEGGAN-2025}

EEGGAN-2025 \citep{williams2025eeggan,williams2023augmenting} uses the official AutoResearch EEG-GAN generator, discriminator, and \texttt{GANTrainer.batch\_train} update procedure.

The benchmark reproduces the source preprocessing in memory. EEG is arranged as trial $\times$ time $\times$ channel and normalized using global training-set min--max normalization. A one-step class-condition prefix is prepended and repeated over channels.

The source configuration uses temporal patch size 20, hidden dimension 16, four model layers, latent dimension 128, five critic iterations, and gradient-penalty coefficient 10. Generator and discriminator learning rates are both $10^{-4}$.

The released Adam optimizer settings
\begin{equation}
\beta=(0,0.9)
\end{equation}
are retained. A compatibility wrapper is used only to express both beta values as floating-point numbers in recent PyTorch versions.

Generated normalized EEG is mapped back to the original training-data amplitude range before GVD construction.

\subsubsection{Conditional raw-EEG VAE}

The raw-EEG conditional VAE \citep{kingma2014vae,sohn2015cvae} operates on a flattened complete EEG trial.

The encoder concatenates the flattened signal with a learned class embedding and uses two fully connected layers of width 512 with SiLU activations. The latent dimension is 64. The decoder uses two width-512 SiLU layers before mapping back to the complete multichannel EEG trial.

The objective is
\begin{equation}
\mathcal L_{\mathrm{cVAE}}
=
\mathcal L_{\mathrm{MSE}}
+
10^{-3}\mathcal L_{\mathrm{KL}}.
\end{equation}
AdamW is used with learning rate $10^{-3}$ and weight decay $10^{-4}$.

Synthetic EEG is generated from
\begin{equation}
z\sim\mathcal N(0,I_{64}),
\end{equation}
conditioned on the desired class.

\subsection{Full-dimensional GVD Fr\'echet distance}

All generative models are compared in the same GVD feature space \citep{dowson1982frechet,heusel2017fid}. No PCA or other dimensionality-reduction transform is applied.

For a trajectory
\begin{equation}
Z=[z_1,\ldots,z_B],
\end{equation}
the Fr\'echet feature contains both all log-\texttt{svec} positions and all first temporal increments:
\begin{equation}
\phi(Z)
=
\left[
z_1,\ldots,z_B,
z_2-z_1,\ldots,z_B-z_{B-1}
\right].
\end{equation}
The feature dimension is therefore
\begin{equation}
(2B-1)m.
\end{equation}
For two sets of features with Gaussian moments
$(\mu_1,\Sigma_1)$ and $(\mu_2,\Sigma_2)$,
the Fr\'echet distance is
\begin{equation}
d_F
=
\|\mu_1-\mu_2\|_2^2
+
\operatorname{tr}
\left(
\Sigma_1+\Sigma_2
-
2
(\Sigma_1^{1/2}
\Sigma_2
\Sigma_1^{1/2})^{1/2}
\right).
\end{equation}
The full feature covariance is not explicitly constructed. The Bures cross term is evaluated through the algebraically equivalent sample-space nuclear norm, allowing the exact full-coordinate distance to be computed without reducing feature dimension.

Distances are computed separately by class and macro-averaged. Real-training to held-out-real distance is retained as the real--real reference. Relative Fr\'echet is computed by dividing the generated-to-held-out-real distance by the corresponding real-training-to-held-out-real reference within each class before averaging.

\subsection{Classification accuracy score}

Class-conditional utility is assessed using a train-synthetic-test-real protocol \citep{ravuri2019cas,esteban2017rcgan}.

Each GVD trajectory is encoded as the complete flattened log-\texttt{svec} sequence. A ridge of $10^{-6}I$ is added before the logarithmic chart for the classifier representation. No dimensionality reduction is applied.

The shared CAS classifier is a two-layer MLP. Each hidden layer has width 256 and consists of
\begin{equation}
\text{Linear}
\rightarrow
\text{BatchNorm}
\rightarrow
\text{GELU}
\rightarrow
\text{Dropout}(0.1).
\end{equation}
The classifier uses AdamW with learning rate $10^{-3}$ and weight decay $10^{-4}$. Training examples are class-balanced by weighted sampling and weighted cross-entropy. The maximum batch size is 2048.

Training runs for at most 120 epochs. The state with the lowest training loss is retained and optimization stops after 15 epochs without improvement.

For each generator, the classifier is trained using synthetic trajectories and tested only on the held-out real partition. ROC-AUC and weighted F1 are reported. A classifier trained on real training trajectories and evaluated on the same held-out test set provides the real-data reference.

\subsection{Novelty and memorization analysis}

The novelty analysis is performed in the complete standardized flattened
log-\texttt{svec} trajectory space. Let
\[
x_i \in \mathbb{R}^{Bm}
\]
denote the standardized flattened feature vector of a real training trajectory
and let
\[
g_j \in \mathbb{R}^{Bm}
\]
denote the corresponding representation of a generated trajectory, where
$m=d(d+1)/2$ is the number of log-\texttt{svec} coordinates per window.
All nearest-neighbor calculations described below are performed
class-conditionally.

Exact copying is tested by rounding complete unstandardized log-\texttt{svec}
feature vectors to six decimal places and comparing their hashes with those
of the real training set.

Nearest-neighbor diagnostics include real leave-one-out distances,
generated-to-training distances, generated leave-one-out distances, and
held-out-real-to-training distances. A generated sample is marked as a
near copy when its distance to the closest same-class real training sample is
below the $5$th percentile of the corresponding real leave-one-out
nearest-neighbor distribution.

\paragraph{Manifold precision.}
Following the $k$-nearest-neighbor support construction used in the
precision--recall family of generative-model metrics, we use $k=5$.
For each real training sample $x_i$, let
\[
\rho_i^{(r)}
=
d_k\!\left(
x_i,\,
\{x_\ell:\ell\neq i,\;y_\ell=y_i\}
\right)
\]
be the Euclidean distance to its $k$th nearest same-class real training
neighbor.

A generated sample $g$ is counted as lying on the estimated real-data
manifold when it falls inside at least one same-class real support ball:
\[
\mathbb{I}_{\mathrm{prec}}(g)
=
\mathbb{I}
\left[
\exists\, i:
y_i=y_g,
\;
\|g-x_i\|_2
\leq
\rho_i^{(r)}
\right].
\]
The reported manifold precision is
\[
\mathrm{Precision}
=
\frac{1}{N_g}
\sum_{j=1}^{N_g}
\mathbb{I}_{\mathrm{prec}}(g_j).
\]
Higher values indicate that a larger fraction of generated trajectories lies
inside the empirical support of the real training distribution:
\[
\mathrm{Precision}\uparrow.
\]

\paragraph{Diversity ratio.}
Diversity is measured from leave-one-out nearest-neighbor spacing.
For every real training sample,
\[
d_i^{(r)}
=
\min_{\ell\neq i,\;y_\ell=y_i}
\|x_i-x_\ell\|_2,
\]
and for every generated trajectory,
\[
d_j^{(g)}
=
\min_{\ell\neq j,\;y_\ell=y_j}
\|g_j-g_\ell\|_2.
\]
The diversity ratio is
\[
R_{\mathrm{div}}
=
\frac{
\operatorname{median}_j d_j^{(g)}
}{
\operatorname{median}_i d_i^{(r)}
}.
\]
The ideal value is therefore one:
\[
\mathrm{Diversity}\rightarrow 1.
\]
Values below one indicate that generated trajectories are more tightly
clustered than the real training trajectories, whereas values above one
indicate that generated samples are more dispersed.

\paragraph{Training coverage.}
For each generated trajectory, let
\[
n(g_j)
=
\arg\min_{i:\,y_i=y_j}
\|g_j-x_i\|_2
\]
denote its nearest same-class real training trajectory. Training coverage is
defined as the fraction of real training trajectories that are selected as
the nearest neighbor of at least one generated trajectory:
\[
\mathrm{Coverage}
=
\frac{
\left|
\left\{
n(g_j):
j=1,\ldots,N_g
\right\}
\right|
}{
N_{\mathrm{train}}
}.
\]
Coverage therefore measures how broadly the generator distributes samples
across the empirical training set rather than repeatedly concentrating around
a small subset:
\[
\mathrm{Coverage}\uparrow.
\]

\subsection{Temporal-coherence analysis}

Temporal coherence is measured in log-\texttt{svec} coordinates after
removing the temporal mean of every feature within each trajectory.
For trajectory $i$, let
\[
z_{i,b}\in\mathbb{R}^{m},
\qquad
b=1,\ldots,B,
\]
denote its log-\texttt{svec} coordinates and define
\[
\bar z_i
=
\frac{1}{B}
\sum_{b=1}^{B} z_{i,b},
\qquad
\widetilde z_{i,b}
=
z_{i,b}-\bar z_i.
\]
Removing the within-trial temporal mean prevents the stable connectivity
level of a trial from dominating the dynamic comparison.

For real and generated trajectories, population time-by-time correlation
matrices are computed from the centred trajectories. Their agreement is
reported as the Pearson correlation between corresponding upper-triangular
entries and as their mean absolute error.

Temporal autocorrelation is additionally evaluated at positive lags up to
$16$ windows. The real and generated lag-autocorrelation curves are compared
using Pearson correlation and mean absolute error.

\subsection{Dynamic-signal measures}

Two permutation-based diagnostics test whether generated trajectories contain
nontrivial temporal dynamics.

For a trajectory
\[
Z_i
=
[z_{i,1},\ldots,z_{i,B}],
\qquad
z_{i,b}\in\mathbb{R}^{m},
\]
define its temporal mean
\[
\bar z_i
=
\frac{1}{B}
\sum_{b=1}^{B} z_{i,b}.
\]

\paragraph{Dynamic energy.}
The dynamic energy of trajectory $i$ is
\[
E_{\mathrm{dyn}}(Z_i)
=
\frac{1}{Bm}
\sum_{b=1}^{B}
\left\|
z_{i,b}-\bar z_i
\right\|_2^2.
\]
This measures the absolute amount of within-trajectory temporal variation
after removing the stable temporal mean.

The table reports the generated-to-real median ratio
\[
R_E
=
\frac{
\operatorname{median}_{i\in\mathrm{gen}}
E_{\mathrm{dyn}}(Z_i)
}{
\operatorname{median}_{i\in\mathrm{real}}
E_{\mathrm{dyn}}(Z_i)
}.
\]
Consequently, the ideal value is
\[
\mathrm{Energy}\rightarrow 1.
\]
Values below one indicate insufficient temporal variation, whereas values
above one indicate excessive temporal variation relative to the real
trajectories.

\paragraph{Dynamic fraction.}
The total log-\texttt{svec} energy of trajectory $i$ is
\[
E_{\mathrm{tot}}(Z_i)
=
\frac{1}{Bm}
\sum_{b=1}^{B}
\|z_{i,b}\|_2^2.
\]
The fraction of trajectory energy attributable to temporal variation is
\[
F_{\mathrm{dyn}}(Z_i)
=
\frac{
E_{\mathrm{dyn}}(Z_i)
}{
\max(E_{\mathrm{tot}}(Z_i),\varepsilon)
},
\]
where $\varepsilon$ is a numerical safeguard.

The reported dynamic-fraction ratio is
\[
R_F
=
\frac{
\operatorname{median}_{i\in\mathrm{gen}}
F_{\mathrm{dyn}}(Z_i)
}{
\operatorname{median}_{i\in\mathrm{real}}
F_{\mathrm{dyn}}(Z_i)
}.
\]
Again, the ideal value is
\[
\mathrm{Dyn.\ fraction}\rightarrow 1.
\]
This quantity differs from dynamic energy because it normalizes temporal
variation by the overall magnitude of the trajectory.

\paragraph{Adjacent-window step energy.}
Temporal smoothness is quantified by the mean squared displacement between
successive windows:
\[
E_{\mathrm{adj}}(Z_i)
=
\frac{1}{(B-1)m}
\sum_{b=1}^{B-1}
\left\|
z_{i,b+1}-z_{i,b}
\right\|_2^2.
\]
The complementary diagnostic table reports the generated-to-real ratio
\[
R_{\mathrm{adj}}
=
\frac{
\operatorname{median}_{i\in\mathrm{gen}}
E_{\mathrm{adj}}(Z_i)
}{
\operatorname{median}_{i\in\mathrm{real}}
E_{\mathrm{adj}}(Z_i)
},
\]
so the target value is
\[
\mathrm{Adjacent}\rightarrow 1.
\]
Values below one correspond to trajectories that are smoother than the real
data, while values above one indicate excessive frame-to-frame variation.

\subsection{Efficiency measurements}

Training time and generation time are measured separately using synchronized CUDA wall-clock timing.

Generation efficiency is measured using a batch of exactly 32 model-native samples. One warm-up call is performed before timing, and the reported latency is based on repeated synchronized measurements.

For direct GVD-space models, the measured generation latency already produces a GVD trajectory.

For raw-EEG generators, raw EEG generation and the subsequent EEG-to-GVD transformation are timed separately. The latter contains:
\begin{enumerate}
    \item per-trial EEG standardization;
    \item computation of the generated trial support $W$;
    \item computation of all window covariances $J_b$;
    \item construction of $\Delta_b=W\odot J_b$.
\end{enumerate}

The benchmark therefore reports both native raw-EEG generation latency and end-to-end latency required to obtain a GVD trajectory.

Peak allocated and reserved CUDA memory are recorded during training and sampling.

For the full GVD-CFM, 1000 training epochs take $304$~s on average across the five main datasets and three seeds (from $128$~s on Zhou2016 to $470$~s on Shin2017A), generating a batch of 32 trajectories with 50 RK4 steps takes $3.92$~s, and peak allocated training memory is $1.77$~GB on average.

\subsection{Implementation}

The benchmark is implemented in PyTorch \citep{paszke2019pytorch} and requires a CUDA-capable GPU. The reference runner was optimized for an NVIDIA A100.

Automatic mixed precision is enabled. BFloat16 is used when supported by the GPU and FP16 is used otherwise. TF32 matrix multiplication and cuDNN benchmarking are enabled. GVD log-Euclidean matrix logarithms and exponentials use GPU eigendecomposition \citep{golub2013matrix} where available, while the SPD calculations themselves are retained in double precision where required for numerical stability.

Fused AdamW \citep{loshchilov2019adamw} is used when supported. The GVD-CFM and compatible baseline networks are optionally compiled using \texttt{torch.compile} with the \texttt{reduce-overhead} mode. Training arrays are kept GPU-resident where practical.

External baselines are loaded from their public source repositories. The benchmark uses the official JET repository, the official Song--Meng--Ermon DDIM release, and the AutoResearch EEG-GAN repository. Compatibility changes required by modern Python or PyTorch are restricted to syntax, datatype, and execution issues; model architectures and reported training objectives are not intentionally altered.

\subsection{Inverse of the composite chart}
\label{app:chart_inverse}

With $\Psi$ as defined in Section~\ref{sec:chart}, its inverse is
\begin{align}
    \widetilde{Z}
    &=
    C_B^\top\bar{Z},\\
    Z
    &=
    \widetilde{Z}\odot\mathbf{1}\sigma^\top+\mathbf{1}\mu^\top,\\
    \overline{\Delta}_b
    &=
    \exp\left(
        \operatorname{svec}^{-1}(z_b)
    \right).
\end{align}

\subsection{Velocity network and training path}
\label{app:velocity_network}

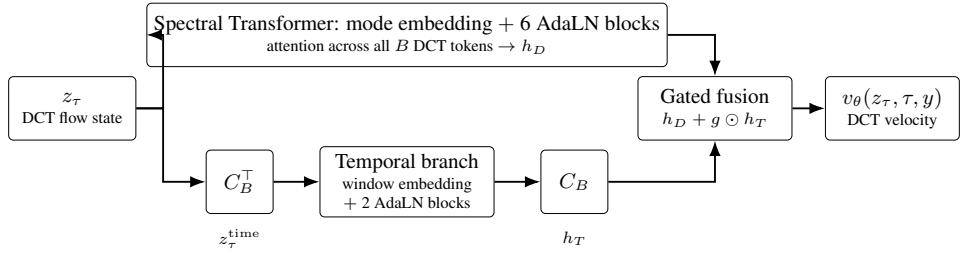
\begin{figure}[t]
\centering
\resizebox{0.9\linewidth}{!}{%
\begin{tikzpicture}[font=\small,>=Latex,
  box/.style={draw,rounded corners=2pt,align=center,minimum height=0.95cm,inner sep=3pt,fill=white},
  arr/.style={->,thick}]
\node[box,minimum width=1.9cm] (z) at (0,0) {$z_\tau$\\[-1pt]{\scriptsize DCT flow state}};
\node[box,minimum width=5.6cm] (spec) at (5.0,1.1) {Spectral Transformer: mode embedding $+$ 6 AdaLN blocks\\[-1pt]{\scriptsize attention across all $B$ DCT tokens $\rightarrow h_D$}};
\node[box,minimum width=1.0cm] (idct) at (2.5,-1.1) {$C_B^\top$};
\node[box,minimum width=2.6cm] (time) at (5.0,-1.1) {Temporal branch\\[-1pt]{\scriptsize window embedding}\\[-2pt]{\scriptsize $+$ 2 AdaLN blocks}};
\node[box,minimum width=1.0cm] (dct) at (7.5,-1.1) {$C_B$};
\node[box,minimum width=2.3cm] (fuse) at (9.6,0) {Gated fusion\\[-1pt]{\scriptsize $h_D+g\odot h_T$}};
\node[box,minimum width=2.0cm] (out) at (12.25,0) {$v_\theta(z_\tau,\tau,y)$\\[-1pt]{\scriptsize DCT velocity}};
\draw[arr] (z.east) -- ++(0.4,0) |- (spec.west);
\draw[arr] (z.east) -- ++(0.4,0) |- (idct.west);
\draw[arr] (idct) -- (time);
\draw[arr] (time) -- (dct);
\draw[arr] (spec.east) -| (fuse.north);
\draw[arr] (dct.east) -| (fuse.south);
\draw[arr] (fuse) -- (out);
\node[font=\scriptsize] at (2.5,-1.95) {$z_\tau^{\mathrm{time}}$};
\node[font=\scriptsize] at (7.5,-1.95) {$h_T$};
\end{tikzpicture}}
\caption{Velocity network of GVD-CFM. Only the DCT state $z_\tau$ is integrated. The spectral Transformer attends over all DCT modes. The temporal branch reads $z_\tau^{\mathrm{time}}=C_B^\top z_\tau$, and its output is returned to DCT alignment by $C_B$ before gated fusion. Both are conditioned on $c_\tau$ through adaptive LayerNorm.}
\label{fig:dualview}
\end{figure}

Each DCT token of Section~\ref{sec:transformer} enters the spectral Transformer
with initial hidden representation
\begin{equation}
    h_k^{(0)}
    =
    P_{\mathrm{in}}\bar{z}_k+e_k^{\mathrm{mode}},
\end{equation}
where $P_{\mathrm{in}}$ is a learned projection and
$e_k^{\mathrm{mode}}$ identifies the DCT mode. Each window of the time view
$z^{\mathrm{time}}=C_B^\top\bar{Z}$ enters the temporal branch as
\begin{equation}
    r_b^{(0)}
    =
    P^{\mathrm{time}}_{\mathrm{in}}z^{\mathrm{time}}_b+e_b^{\mathrm{win}},
\end{equation}
where $e_b^{\mathrm{win}}$ identifies the window position. The model is
class-conditional, with conditioning vector
\begin{equation}
    c_\tau
    =
    e_{\mathrm{flow}}(\tau)
    +
    e_{\mathrm{class}}(y),
\end{equation}
where $\tau\in[0,1]$ is flow time. Flow time is encoded using Fourier
features followed by an MLP \citep{tancik2020fourier}.

Writing $S_\theta$ for the spectral Transformer and $T_\theta$ for the
temporal branch, both conditioned on $c_\tau$, a single forward pass computes
(Figure~\ref{fig:dualview})
\begin{equation}
    h_D=S_\theta(\bar{Z},c_\tau),
    \qquad
    h_T=C_B\,T_\theta\!\left(C_B^\top\bar{Z},c_\tau\right),
\end{equation}
\begin{equation}
    v_\theta(\bar{Z},\tau,y)
    =
    P_{\mathrm{out}}
    \left(
        h_D+\sigma\!\left(W_g[h_D;h_T]+b_g\right)\odot h_T
    \right),
\end{equation}
which is the joint velocity
\begin{equation}
    v_\theta:
    \mathbb{R}^{B\times m}
    \times[0,1]
    \times\mathcal{Y}
    \longrightarrow
    \mathbb{R}^{B\times m}
\end{equation}
of the complete trajectory.

During training of Section~\ref{sec:training}, we sample
\begin{equation}
    z_0\sim\mathcal{N}(0,I),
    \qquad
    z_1\sim\Psi_{\#}q(\cdot\mid y),
    \qquad
    \tau\sim\mathcal{U}[0,1].
\end{equation}
The pair $(z_0,z_1)$ is then re-drawn from the classwise minibatch Sinkhorn
coupling $\pi_y$. We use the linear conditional path
\begin{equation}
    z_\tau
    =
    (1-\tau)z_0+\tau z_1,
    \qquad
    u_\tau=z_1-z_0.
\end{equation}

To generate a trajectory from class $y$ we sample $z(0)\sim\mathcal{N}(0,I)$
and solve
\begin{equation}
    \frac{\mathrm{d}z(\tau)}{\mathrm{d}\tau}
    =
    v_{\theta^\star}(z(\tau),\tau,y),
    \qquad
    \tau\in[0,1].
\end{equation}

\subsection{Training and sampling algorithms}

\begin{algorithm}[H]
\caption{GVD-CFM training}
\label{alg:gvd_cfm_training}
\begin{algorithmic}[1]
\REQUIRE Label distribution $\pi_{\mathcal{Y}}$, training distribution
$q$, and diffeomorphism $\Psi$
\ENSURE Trained parameters $\theta^\star$
\STATE Initialize $\theta$
\WHILE{not converged}
    \STATE Sample a class-weighted minibatch
    $\{(\overline{\boldsymbol{\Delta}}^{(i)},y^{(i)})\}_{i=1}^{N}$ from $q$
    \STATE Set
    $z_1^{(i)}\gets\Psi(\overline{\boldsymbol{\Delta}}^{(i)})$
    \STATE Sample $z_0^{(i)}\sim\mathcal{N}(0,I)$
    \FOR{each class $y$ in the minibatch}
        \STATE Re-pair $\{z_0^{(i)}\}$ with
        $\{z_1^{(i)}:y^{(i)}=y\}$ using the Sinkhorn plan $\pi_y$
    \ENDFOR
    \STATE Sample $\tau^{(i)}\sim\mathcal{U}[0,1]$
    \STATE Set $z_\tau^{(i)}\gets(1-\tau^{(i)})z_0^{(i)}+\tau^{(i)} z_1^{(i)}$
    \STATE Set $u_\tau^{(i)}\gets z_1^{(i)}-z_0^{(i)}$
    \STATE Evaluate $v_\theta(z_\tau^{(i)},\tau^{(i)},y^{(i)})$
    \STATE Compute
    \[
    \mathcal{L}\gets
    \frac{1}{NBm}
    \sum_{i=1}^{N}
    \left\|
        v_\theta(z_\tau^{(i)},\tau^{(i)},y^{(i)})-u_\tau^{(i)}
    \right\|_F^2
    \]
    \STATE Update
    $\theta\gets
    \operatorname{OptimizerStep}
    (\theta,\nabla_\theta\mathcal{L})$
\ENDWHILE
\STATE \textbf{return} $\theta^\star$
\end{algorithmic}
\end{algorithm}

\begin{algorithm}[H]
\caption{GVD-CFM sampling}
\label{alg:gvd_cfm_sampling}
\begin{algorithmic}[1]
\REQUIRE Class label $y$, parameters $\theta^\star$, integration
steps $L$, and diffeomorphism $\Psi$
\ENSURE Generated trajectory
$\widehat{\boldsymbol{\Delta}}$
\STATE Set $h\gets1/L$
\STATE Sample $z_0\sim\mathcal{N}(0,I)$
\FOR{$\ell=0,\ldots,L-1$}
    \STATE Set $\tau_\ell\gets\ell h$
    \STATE Set
    $z_{\ell+1}\gets
    \operatorname{RK4Step}
    (v_{\theta^\star},z_\ell,\tau_\ell,y,h)$
\ENDFOR
\STATE Set
$\widehat{\boldsymbol{\Delta}}\gets\Psi^{-1}(z_L)$
\STATE \textbf{return}
$\widehat{\boldsymbol{\Delta}}$
\end{algorithmic}
\end{algorithm}

\section{Extended Proofs}
\label{app:extended_proofs}

\paragraph{Geometric overview.}
The log-Euclidean map $\phi:\mathbb{S}_{++}^{d}\rightarrow\mathbb{R}^{m}$, $\phi(X)=\operatorname{svec}(\log X)$, is a global coordinate system on the SPD manifold. The matrix logarithm is a diffeomorphism from $\mathbb{S}_{++}^{d}$ onto $\mathbb{S}^{d}$ and $\operatorname{svec}$ is a linear isometry, so $\phi$ is a global diffeomorphism. The log-Euclidean metric is the pullback of the Euclidean inner product through $\phi$,
\[
    g_X(\xi,\eta)=\bigl\langle D\phi_X[\xi],D\phi_X[\eta]\bigr\rangle_{\mathbb{R}^{m}},
    \qquad \xi,\eta\in T_X\mathbb{S}_{++}^{d},
\]
so $\phi$ is an isometry by construction, and the straight line $z_\tau=(1-\tau)\phi(X_0)+\tau\phi(X_1)$ maps through $\phi^{-1}$ to the log-Euclidean geodesic between $X_0$ and $X_1$. The construction requires $X\succ0$: a positive-semidefinite matrix with a zero eigenvalue has no finite logarithm and lies outside the domain of $\phi$.

For a trajectory of $B$ matrices the state space is the product manifold $\mathcal{M}_B=(\mathbb{S}_{++}^{d})^{B}$, with tangent space $\prod_{b=1}^{B}T_{X_b}\mathbb{S}_{++}^{d}$ and product metric $g^{\mathrm{prod}}(\xi,\eta)=\sum_{b=1}^{B}g_{X_b}(\xi_b,\eta_b)$. Applying $\phi$ to every factor gives the trajectory chart $\Phi(X_{1:B})=(\phi(X_1),\ldots,\phi(X_B))\in\mathbb{R}^{B\times m}$, which is a global diffeomorphism with $g^{\mathrm{prod}}(\xi,\eta)=\langle D\Phi[\xi],D\Phi[\eta]\rangle_F$. If a single window $X_b$ is singular, the corresponding factor of $\Phi$ is undefined and the trajectory lies outside the domain of the chart. Strict positive definiteness (Proposition~\ref{prop:gvd_spd}) must therefore be established before this geometry can be used.

The chart, including the DCT and the standardization, is a global diffeomorphism and an isometry for its pullback metric. In addition, the Euclidean conditional flow-matching loss in these coordinates equals the Riemannian loss on $\mathcal{M}_B$, because the manifold norm of a velocity residual equals the Frobenius norm of its pushforward. Finally, integrating the flow in Euclidean coordinates and decoding by $\Phi^{-1}$ gives the manifold flow, and every decoded window is SPD because the matrix exponential of a symmetric matrix is SPD; the same holds at every step of an explicit Runge--Kutta integrator.

\subsection{Graph-variate signal analysis}
\label{app:gvsa}

GVSA \citep{smith2019gvsa} defines the modulated connectivity of
Section~\ref{sec:gvd} entrywise as
\begin{equation}
\theta_{ij}(t)=
\begin{cases}
W_{ij}F_{\mathcal{V}}\!\left(x_i(t),x_j(t)\right),
    & i\neq j,\\
0,  & i=j,
\end{cases}
\end{equation}
where $W_{ij}$ describes the stable, long-term relationship between nodes $i$
and $j$, while $F_{\mathcal{V}}$ measures their instantaneous connectivity.
For correlation-based GVSA,
\begin{equation}
F_{\mathcal{V}}\!\left(x_i(t),x_j(t)\right)
=
\left|
\left(x_i(t)-\bar{x}_i\right)
\left(x_j(t)-\bar{x}_j\right)
\right|,
\end{equation}
where $\bar{x}_i$ is the temporal mean of node $i$. Substituting this into
the entrywise definition yields the sample-resolution connectivity matrix of
Section~\ref{sec:gvd}.

Our implementation makes two deliberate modifications to this classical
formulation in order to preserve the geometry required by the generative
model. First, although the original correlation-based definition uses the
absolute value of the instantaneous product, we retain the sign of both the
long-term support and the instantaneous interaction. Let
\begin{equation}
u_i(t)=x_i(t)-\bar{x}_i,
\end{equation}
and write
\begin{equation}
J_t = u_tu_t^\top,
\qquad
\Delta_t = W\odot J_t.
\end{equation}
Thus, rather than replacing the instantaneous interaction by
$|u_i(t)u_j(t)|$, we use the signed outer product. When the support is written
as a correlation matrix $W$ and $D_t=\operatorname{diag}(u_t)$, this gives
\begin{equation}
\Delta_t
=
W\odot u_tu_t^\top
=
D_tWD_t.
\end{equation}
This signed construction preserves the congruence structure of the support
matrix and is therefore compatible with positive-semidefinite, and in the
experimentally used binned case positive-definite, GVD trajectories. In
contrast, taking entrywise absolute values destroys this exact congruence
relationship and is not required for the downstream geometric construction.

Second, the original GVSA definition sets the diagonal entries to zero. We
retain the diagonal during GVD construction because it contains the
instantaneous node-energy terms and is necessary for treating each
connectivity state as a full symmetric positive-definite matrix. After
temporal binning, we therefore use
\begin{equation}
\Delta_b
=
\frac{1}{|\mathcal I_b|}
\sum_{t\in\mathcal I_b}
W\odot u_tu_t^\top,
\end{equation}
with the diagonal left intact. All GVD matrices used in our experiments are
explicitly verified to be positive definite before entering the
log-Euclidean representation; no ridge or nearest-SPD projection is applied
to the canonical GVD construction. If a zero-diagonal graph representation
is desired for visualization or conventional network analysis, the diagonal
can be removed post hoc without altering the matrices used for geometric
learning.

\subsection{Proof of Proposition~\ref{prop:gvd_spd} for single samples}
\label{app:proof_gvd_spd}

\begin{figure}[t]
    \centering
    \includegraphics[width=\columnwidth]{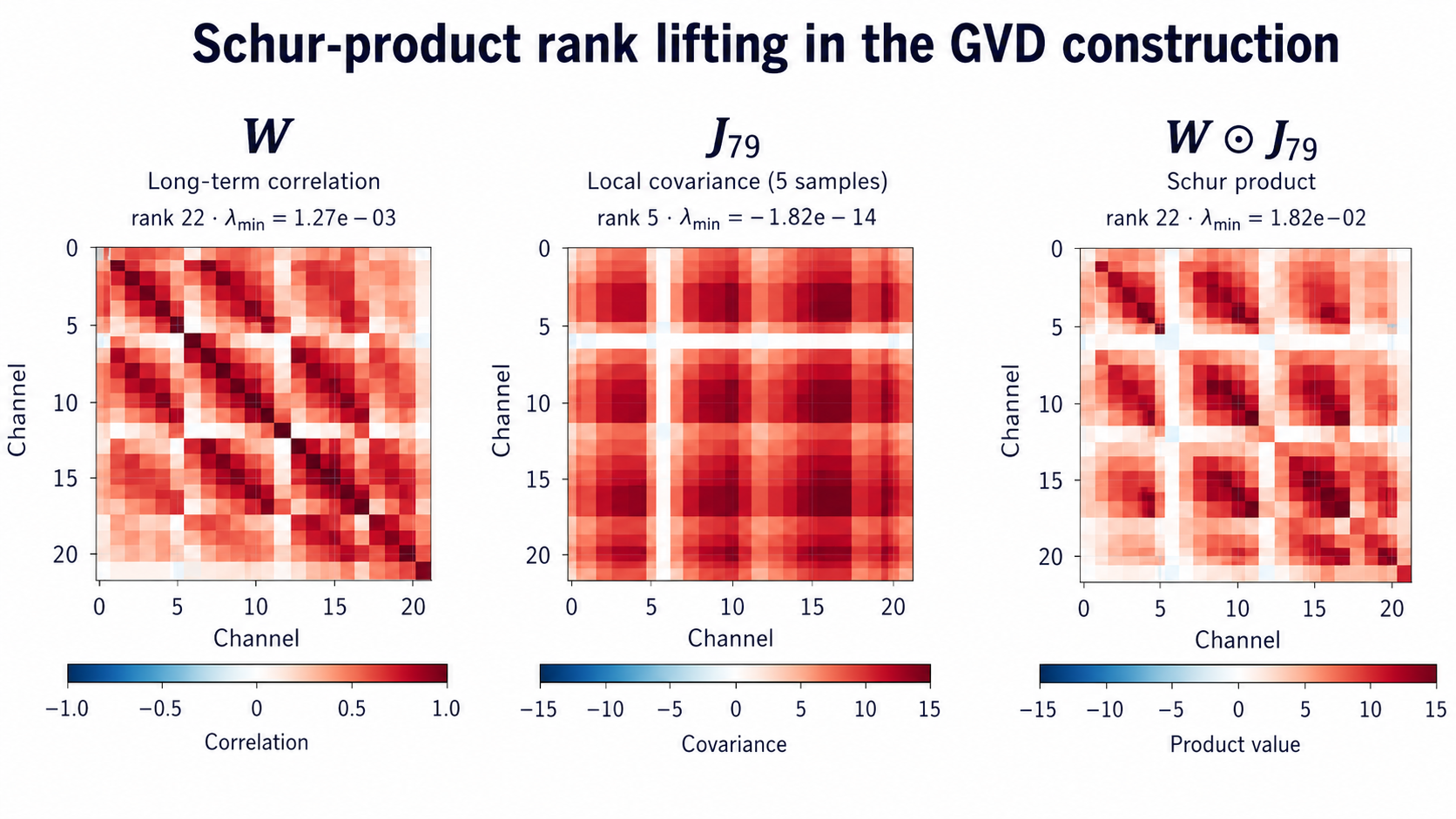}
    \caption{Schur-product rank lifting induced by the long-term support.
    The instantaneous covariance $J_b$ is positive semidefinite but
    rank-deficient because each temporal window contains fewer samples than
    EEG channels. The long-term support $W$ is positive definite, and the
    Hadamard composition $W \odot J_b$ preserves positive semidefiniteness
    while lifting the example to full rank. No ridge regularization is added
    to $J_b$ or $W \odot J_b$.}
    \label{fig:rank_lifting}
\end{figure}

\begin{proof}
Since $J_t=u_tu_t^\top\succeq0$ and $W\succ0$, the Schur product
theorem \citep{horn2012matrix,schur1911bemerkungen} guarantees that
\begin{equation}
    \Delta_t=W\odot J_t\succeq0.
\end{equation}
To establish strict positive definiteness, let
$D_t=\operatorname{diag}(u_t)$. Entrywise,
\begin{equation}
    [W\odot(u_tu_t^\top)]_{ij}
    =
    W_{ij}u_{i,t}u_{j,t},
\end{equation}
and therefore
\begin{equation}
    \Delta_t=D_tWD_t.
\end{equation}
If every component of $u_t$ is nonzero, then $D_t$ is invertible.
Hence, for every nonzero $v\in\mathbb{R}^{d}$,
\begin{equation}
    v^\top\Delta_t v
    =
    (D_tv)^\top W(D_tv)
    >0,
\end{equation}
because $D_tv\neq0$ and $W\succ0$. Thus,
$\Delta_t\in\mathbb{S}_{++}^{d}$. Moreover, congruence by an
invertible matrix preserves rank, giving
\begin{equation}
    \operatorname{rank}(\Delta_t)
    =
    \operatorname{rank}(D_tWD_t)
    =
    \operatorname{rank}(W)
    =
    d.
\end{equation}
Therefore every $\Delta_t$ lies on the SPD manifold.
\end{proof}

If some component of $u_t$ is exactly zero, $D_t$ is singular and
$\Delta_t$ is positive semidefinite rather than positive definite. We do not encounter this in practice however.


\subsection{Further results}

\begin{proposition}[Equivalence of DCT-coordinate flow matching and product log-Euclidean Riemannian flow matching]
\label{prop:dct_le_rfm_equiv}

Let
\[
\mathcal M_B
=
\left(\mathbb S_{++}^{d}\right)^B,
\qquad
m=\frac{d(d+1)}{2},
\]
and define the product log-Euclidean coordinate map
\[
\Phi(\Delta_{1:B})
=
\begin{bmatrix}
\operatorname{svec}(\log \Delta_1)^\top\\
\vdots\\
\operatorname{svec}(\log \Delta_B)^\top
\end{bmatrix}
\in \mathbb R^{B\times m},
\]
where $\operatorname{svec}$ is chosen to preserve the Frobenius inner
product on symmetric matrices. Let
\[
\mathcal C(Z)=C_B Z,
\qquad
C_B^\top C_B=I_B,
\]
be the orthonormal temporal DCT, and define the composite map
\[
\Psi
=
\mathcal C\circ\Phi.
\]
Equip $\mathcal M_B$ with the product log-Euclidean metric
\[
g_{\mathrm{LE}}^{\mathrm{prod}}.
\]
Then $\Psi$ is a global Riemannian isometry from
\[
(\mathcal M_B,g_{\mathrm{LE}}^{\mathrm{prod}})
\]
to Euclidean space
\[
(\mathbb R^{B\times m},\langle\cdot,\cdot\rangle_F).
\]
Consequently, for any endpoints
$\Delta^0,\Delta^1\in\mathcal M_B$, with
\[
Z_i=\Psi(\Delta^i),\qquad i\in\{0,1\},
\]
the Euclidean conditional path
\[
Z_\tau=(1-\tau)Z_0+\tau Z_1
\]
is the $\Psi$-image of the product log-Euclidean geodesic
\[
\Gamma_\tau
=
\Psi^{-1}(Z_\tau).
\]
Moreover, if
\[
U^E=Z_1-Z_0
\]
denotes the Euclidean conditional velocity and
\[
U_\tau^{\mathcal M}
=
D\Psi^{-1}_{Z_\tau}[U^E]
\]
the corresponding manifold tangent velocity, then for any manifold
vector field $V_\theta$ whose coordinate representation is
\[
v_\theta(Z,\tau)
=
D\Psi_{\Psi^{-1}(Z)}
\left[
V_\theta(\Psi^{-1}(Z),\tau)
\right],
\]
we have the pointwise identity
\[
\boxed{
\left\|
V_\theta(\Gamma_\tau,\tau)
-
U_\tau^{\mathcal M}
\right\|_{g_{\mathrm{LE}}^{\mathrm{prod}}}^{2}
=
\left\|
v_\theta(Z_\tau,\tau)
-
(Z_1-Z_0)
\right\|_F^{2}.
}
\]
Hence the Euclidean $L_2$ conditional flow-matching objective in
full-mode DCT coordinates is exactly equal to the corresponding
Riemannian flow-matching objective on the product SPD manifold:
\[
\boxed{
\mathcal L_{\mathrm{CFM}}
=
\mathcal L_{\mathrm{RFM}}.
}
\]
\end{proposition}

\begin{proof}
For each SPD factor, the matrix logarithm
\[
\log:\mathbb S_{++}^{d}\rightarrow\mathbb S^{d}
\]
is a global smooth diffeomorphism. Since $\operatorname{svec}$ is a
linear isomorphism from $\mathbb S^d$ to $\mathbb R^m$, the product map
\[
\Phi:\mathcal M_B\rightarrow\mathbb R^{B\times m}
\]
is also a global smooth diffeomorphism.

By definition of the log-Euclidean metric, the single-factor map
\[
\Delta\mapsto\operatorname{svec}(\log\Delta)
\]
is an isometry from $\mathbb S_{++}^{d}$ equipped with the
log-Euclidean metric to $\mathbb R^m$ equipped with its Euclidean
metric. Therefore, for tangent vectors
\[
\xi=(\xi_1,\ldots,\xi_B),
\qquad
\eta=(\eta_1,\ldots,\eta_B)
\]
at $\Delta=(\Delta_1,\ldots,\Delta_B)$,
\[
g_{\mathrm{LE},\Delta}^{\mathrm{prod}}(\xi,\eta)
=
\left\langle
D\Phi_\Delta[\xi],
D\Phi_\Delta[\eta]
\right\rangle_F.
\]
Now consider the temporal DCT map
\[
\mathcal C(Z)=C_BZ.
\]
Because $C_B$ is orthogonal,
\[
C_B^\top C_B=I_B,
\]
and therefore, for arbitrary
$X,Y\in\mathbb R^{B\times m}$,
\[
\langle C_BX,C_BY\rangle_F
=
\operatorname{tr}(X^\top C_B^\top C_BY)
=
\operatorname{tr}(X^\top Y)
=
\langle X,Y\rangle_F.
\]
Thus $\mathcal C$ is a Euclidean isometry.

Since
\[
\Psi=\mathcal C\circ\Phi,
\]
its differential satisfies
\[
D\Psi_\Delta
=
D\mathcal C_{\Phi(\Delta)}
\circ
D\Phi_\Delta.
\]
Hence
\begin{align}
\left\langle
D\Psi_\Delta[\xi],
D\Psi_\Delta[\eta]
\right\rangle_F
&=
\left\langle
C_BD\Phi_\Delta[\xi],
C_BD\Phi_\Delta[\eta]
\right\rangle_F\\
&=
\left\langle
D\Phi_\Delta[\xi],
D\Phi_\Delta[\eta]
\right\rangle_F\\
&=
g_{\mathrm{LE},\Delta}^{\mathrm{prod}}(\xi,\eta).
\end{align}
Therefore $\Psi$ is a global Riemannian isometry.

Now let
\[
Z_\tau=(1-\tau)Z_0+\tau Z_1.
\]
Euclidean straight lines are geodesics, and an isometry maps geodesics
to geodesics. Thus
\[
\Gamma_\tau
=
\Psi^{-1}(Z_\tau)
\]
is the corresponding product log-Euclidean geodesic on $\mathcal M_B$.

Differentiating
\[
\Gamma_\tau=\Psi^{-1}(Z_\tau)
\]
gives
\[
\dot\Gamma_\tau
=
D\Psi^{-1}_{Z_\tau}[\dot Z_\tau].
\]
Since
\[
\dot Z_\tau=Z_1-Z_0,
\]
we obtain
\[
U_\tau^{\mathcal M}
=
D\Psi^{-1}_{Z_\tau}[Z_1-Z_0].
\]
Because $\Psi$ is a diffeomorphism,
\[
D\Psi_{\Gamma_\tau}
\circ
D\Psi^{-1}_{Z_\tau}
=
\operatorname{Id},
\]
and hence
\[
D\Psi_{\Gamma_\tau}
[U_\tau^{\mathcal M}]
=
Z_1-Z_0.
\]
By definition of the coordinate representation of the model vector
field,
\[
D\Psi_{\Gamma_\tau}
[
V_\theta(\Gamma_\tau,\tau)
]
=
v_\theta(Z_\tau,\tau).
\]
Therefore, using linearity of the differential,
\begin{align}
&D\Psi_{\Gamma_\tau}
\left[
V_\theta(\Gamma_\tau,\tau)
-
U_\tau^{\mathcal M}
\right]
\\
&\qquad=
v_\theta(Z_\tau,\tau)
-
(Z_1-Z_0).
\end{align}
Since $\Psi$ is a Riemannian isometry,
\[
\|\xi\|_{g_{\mathrm{LE}}^{\mathrm{prod}}}^{2}
=
\|D\Psi[\xi]\|_F^{2}.
\]
Applying this to the velocity residual yields
\[
\left\|
V_\theta(\Gamma_\tau,\tau)
-
U_\tau^{\mathcal M}
\right\|_{g_{\mathrm{LE}}^{\mathrm{prod}}}^{2}
=
\left\|
v_\theta(Z_\tau,\tau)
-
(Z_1-Z_0)
\right\|_F^{2}.
\]
The equality holds pointwise for every endpoint pair and every
$\tau$. Taking expectations therefore gives
\[
\mathcal L_{\mathrm{RFM}}(\theta)
=
\mathcal L_{\mathrm{CFM}}(\theta).
\]
\end{proof}

\begin{corollary}[Standardized chart]
\label{cor:standardized_chart}
Let $S(Z)=(Z-\mathbf{1}\mu^\top)D_\sigma^{-1}$ with $D_\sigma=\operatorname{diag}(\sigma_1,\ldots,\sigma_m)$, $\sigma_j>0$, and let $\Psi_\sigma=\mathcal{C}\circ S\circ\Phi$ be the chart used in practice. Define on $\mathcal{M}_B$ the metric
\[
g_\sigma(\xi,\eta)
=
\sum_{b=1}^{B}
\bigl\langle
D_\sigma^{-1}\,\mathrm{D}\phi_{\Delta_b}[\xi_b],\,
D_\sigma^{-1}\,\mathrm{D}\phi_{\Delta_b}[\eta_b]
\bigr\rangle,
\qquad \phi=\operatorname{svec}\circ\log .
\]
Then (i) $\Psi_\sigma$ is a global diffeomorphism and a Riemannian isometry from $(\mathcal{M}_B,g_\sigma)$ to $(\mathbb{R}^{B\times m},\langle\cdot,\cdot\rangle_F)$, so Proposition~\ref{prop:dct_le_rfm_equiv} holds verbatim with $\Psi_\sigma$ and $g_\sigma$ in place of $\Psi$ and $g^{\mathrm{prod}}_{\mathrm{LE}}$; (ii) the geodesics of $g_\sigma$ and $g^{\mathrm{prod}}_{\mathrm{LE}}$ coincide as parametrized curves; (iii) for a manifold field $V$ and target $U$,
$\|V-U\|^2_{g_\sigma}=\sum_b\|D_\sigma^{-1}\mathrm{D}\phi_{\Delta_b}[V_b-U_b]\|^2$, a fixed diagonal reweighting of the log-Euclidean residual, so the population minimizer over measurable fields is the same conditional expectation under both metrics.
\end{corollary}

\begin{proof}
(i) $S$ is an invertible affine map, so $\Psi_\sigma$ is a diffeomorphism, and its differential is $\mathrm{D}\Psi_\sigma[\xi]=C_B\,\mathrm{D}\Phi[\xi]\,D_\sigma^{-1}$. Since $C_B$ is orthogonal, $\langle\mathrm{D}\Psi_\sigma[\xi],\mathrm{D}\Psi_\sigma[\eta]\rangle_F=g_\sigma(\xi,\eta)$, which is the definition of an isometry; the proof of Proposition~\ref{prop:dct_le_rfm_equiv} uses only this property. (ii) In log coordinates both metrics are constant, so their geodesics are affinely parametrized straight lines in these coordinates, and $S$ maps straight lines to straight lines. (iii) The expression follows from the definition of $g_\sigma$. The conditional expectation minimizes the expected squared residual under any fixed positive-definite quadratic form, so both objectives share it as population minimizer.
\end{proof}

\begin{remark}
The proposition places no restriction on how $v_\theta$ is computed from $Z$
beyond measurability. It therefore holds for the network with the temporal
branch, whose temporal input $C_B^\top Z$ is a fixed linear function of $Z$, and it holds
for any coupling of the endpoints, including the classwise minibatch Sinkhorn
coupling used in training, because the identity is pointwise in the endpoint
pair.
\end{remark}
\begin{proposition}[Spectral bounds for graph-variate connectivity]
\label{prop:gvd_spectral_bound}
Let $u_1,\ldots,u_n\in\mathbb{R}^d$ denote the observations in a
temporal bin, and define
\begin{equation}
    J
    =
    \frac{1}{n}\sum_{t=1}^{n}u_tu_t^\top,
    \qquad
    \Delta
    =
    W\odot J,
\end{equation}
where $W\in\mathbb{S}_{++}^d$. Define the channelwise energy within
the bin by
\begin{equation}
    q_i
    :=
    \frac{1}{n}\sum_{t=1}^{n}u_{i,t}^2,
    \qquad i=1,\ldots,d,
\end{equation}
and let
\begin{equation}
    Q := \operatorname{diag}(q_1,\ldots,q_d).
\end{equation}
Then $\Delta$ satisfies the Loewner-order bounds
\begin{equation}
    \lambda_{\min}(W)\,Q
    \preceq
    \Delta
    \preceq
    \lambda_{\max}(W)\,Q.
    \label{eq:gvd_loewner_bound}
\end{equation}
Consequently,
\begin{equation}
    \lambda_{\min}(\Delta)
    \geq
    \lambda_{\min}(W)\min_i q_i,
    \label{eq:gvd_lambda_min_bound}
\end{equation}
and
\begin{equation}
    \lambda_{\max}(\Delta)
    \leq
    \lambda_{\max}(W)\max_i q_i.
    \label{eq:gvd_lambda_max_bound}
\end{equation}
If $q_i>0$ for every channel, then in particular
\begin{equation}
    \lambda_{\min}(\Delta)>0,
\end{equation}
and its spectral condition number obeys
\begin{equation}
    \kappa_2(\Delta)
    \leq
    \kappa_2(W)
    \frac{\max_i q_i}{\min_i q_i}.
    \label{eq:gvd_condition_bound}
\end{equation}
Thus the distance of a GVD matrix from the boundary of the SPD cone
is controlled jointly by the smallest eigenvalue of the long-term
support and the least energetic channel in the local bin.
\end{proposition}

\begin{proof}
For each sample $u_t$, let
\begin{equation}
    D_t := \operatorname{diag}(u_t).
\end{equation}
Using the identity
\begin{equation}
    W\odot(u_tu_t^\top)
    =
    D_tWD_t,
\end{equation}
and linearity of the Hadamard product, we can write
\begin{equation}
    \Delta
    =
    W\odot
    \left(
        \frac{1}{n}\sum_{t=1}^{n}u_tu_t^\top
    \right)
    =
    \frac{1}{n}
    \sum_{t=1}^{n}D_tWD_t.
    \label{eq:gvd_sum_congruences}
\end{equation}

Since $W\in\mathbb{S}_{++}^d$,
\begin{equation}
    \lambda_{\min}(W)I
    \preceq
    W
    \preceq
    \lambda_{\max}(W)I.
\end{equation}
Congruence preserves the Loewner order, so for every $t$,
\begin{equation}
    \lambda_{\min}(W)D_t^2
    \preceq
    D_tWD_t
    \preceq
    \lambda_{\max}(W)D_t^2.
\end{equation}
Averaging over the $n$ samples gives
\begin{equation}
    \lambda_{\min}(W)
    \left(
        \frac{1}{n}\sum_{t=1}^{n}D_t^2
    \right)
    \preceq
    \Delta
    \preceq
    \lambda_{\max}(W)
    \left(
        \frac{1}{n}\sum_{t=1}^{n}D_t^2
    \right).
\end{equation}
Because
\begin{equation}
    \frac{1}{n}\sum_{t=1}^{n}D_t^2
    =
    \operatorname{diag}
    \left(
        \frac{1}{n}\sum_{t=1}^{n}u_{1,t}^2,
        \ldots,
        \frac{1}{n}\sum_{t=1}^{n}u_{d,t}^2
    \right)
    =
    Q,
\end{equation}
we obtain
\begin{equation}
    \lambda_{\min}(W)Q
    \preceq
    \Delta
    \preceq
    \lambda_{\max}(W)Q,
\end{equation}
which proves \eqref{eq:gvd_loewner_bound}.

Applying the Rayleigh--Ritz characterization to the lower bound,
for every unit vector $x$,
\begin{align}
    x^\top\Delta x
    &\geq
    \lambda_{\min}(W)x^\top Qx \\
    &=
    \lambda_{\min}(W)
    \sum_{i=1}^{d}q_i x_i^2 \\
    &\geq
    \lambda_{\min}(W)\min_i q_i.
\end{align}
Taking the minimum over all $\|x\|_2=1$ therefore yields
\begin{equation}
    \lambda_{\min}(\Delta)
    \geq
    \lambda_{\min}(W)\min_i q_i.
\end{equation}

Likewise, the upper Loewner bound gives, for every unit vector $x$,
\begin{align}
    x^\top\Delta x
    &\leq
    \lambda_{\max}(W)x^\top Qx \\
    &\leq
    \lambda_{\max}(W)\max_i q_i,
\end{align}
and hence
\begin{equation}
    \lambda_{\max}(\Delta)
    \leq
    \lambda_{\max}(W)\max_i q_i.
\end{equation}

If every $q_i>0$, the lower bound is strictly positive because
$W\succ0$, so $\lambda_{\min}(\Delta)>0$. Finally,
\begin{align}
    \kappa_2(\Delta)
    &=
    \frac{\lambda_{\max}(\Delta)}
         {\lambda_{\min}(\Delta)} \\
    &\leq
    \frac{
        \lambda_{\max}(W)\max_i q_i
    }{
        \lambda_{\min}(W)\min_i q_i
    } \\
    &=
    \kappa_2(W)
    \frac{\max_i q_i}{\min_i q_i},
\end{align}
which proves \eqref{eq:gvd_condition_bound}.
\end{proof}
\begin{proposition}[DCT-II diagonalizes discrete temporal variation]
\label{prop:dct_temporal_laplacian}
Let $D_B \in \mathbb{R}^{(B-1)\times B}$ denote the first-order
temporal difference operator,
\[
D_B =
\begin{bmatrix}
-1 & 1 & 0 & \cdots & 0 \\
0 & -1 & 1 & \cdots & 0 \\
\vdots & & \ddots & \ddots & \vdots \\
0 & \cdots & 0 & -1 & 1
\end{bmatrix},
\]
and let
\[
L_B = D_B^\top D_B
\]
be the corresponding path-graph Laplacian. Let
$C_B \in \mathbb{R}^{B\times B}$ be the orthonormal DCT-II matrix with
entries
\[
[C_B]_{k,b}
=
\alpha_k
\cos\left(
\frac{\pi k}{B}\left(b+\frac{1}{2}\right)
\right),
\qquad
k,b=0,\ldots,B-1,
\]
where
\[
\alpha_0 = \frac{1}{\sqrt{B}},
\qquad
\alpha_k = \sqrt{\frac{2}{B}},
\quad k\geq 1.
\]
Then the DCT-II basis diagonalizes $L_B$:
\[
C_B L_B C_B^\top
=
\Lambda_B,
\]
where
\[
\Lambda_B
=
\operatorname{diag}(\lambda_0,\ldots,\lambda_{B-1}),
\qquad
\lambda_k
=
4\sin^2\left(\frac{\pi k}{2B}\right).
\]
In particular,
\[
0=\lambda_0 < \lambda_1 < \cdots < \lambda_{B-1}<4.
\]
\end{proposition}

\begin{proof}
Let $q_k \in \mathbb{R}^B$ denote the $k$-th DCT-II basis vector,
\[
q_k(b)
=
\alpha_k
\cos\left(
\frac{\pi k}{B}\left(b+\frac{1}{2}\right)
\right),
\qquad
b=0,\ldots,B-1.
\]
For an interior index $b=1,\ldots,B-2$,
\[
(L_B q_k)_b
=
2q_k(b)-q_k(b-1)-q_k(b+1).
\]
Let $\theta_k=\pi k/B$. Using
\[
\cos(a-\theta_k)+\cos(a+\theta_k)
=
2\cos(a)\cos(\theta_k),
\]
we obtain
\[
(L_B q_k)_b
=
2(1-\cos\theta_k)q_k(b).
\]
Since
\[
2(1-\cos\theta_k)
=
4\sin^2\left(\frac{\theta_k}{2}\right),
\]
it follows that
\[
(L_B q_k)_b
=
4\sin^2\left(\frac{\pi k}{2B}\right)q_k(b).
\]
The two boundary rows satisfy the same identity, and hence
\[
L_B q_k = \lambda_k q_k,
\qquad
\lambda_k
=
4\sin^2\left(\frac{\pi k}{2B}\right).
\]
Because the DCT-II basis is orthonormal, its basis vectors form a
complete orthonormal eigenbasis of $L_B$, which gives
\[
C_B L_B C_B^\top = \Lambda_B.
\]
The ordering of the eigenvalues follows from the strict monotonicity of
$\sin(x)$ on $[0,\pi/2)$.
\end{proof}

\begin{theorem}[Exact spectral decomposition of log-Euclidean temporal variation]
\label{thm:gvd_temporal_spectral_energy}
Let
\[
\Delta_{1:B}
=
(\Delta_1,\ldots,\Delta_B)
\in
\left(\mathbb{S}_{++}^{d}\right)^B
\]
be a GVD trajectory. Define its log-Euclidean coordinates by
\[
z_b
=
\operatorname{svec}(\log \Delta_b)
\in
\mathbb{R}^{m},
\qquad
m=\frac{d(d+1)}{2},
\]
and stack them as
\[
Z
=
\begin{bmatrix}
z_1^\top \\
\vdots \\
z_B^\top
\end{bmatrix}
\in
\mathbb{R}^{B\times m}.
\]
Let
\[
\bar Z = C_B Z
\]
be the full orthonormal DCT-II representation, and denote the $k$-th
DCT coefficient vector by
\[
\bar z_k^\top
=
[\bar Z]_{k,:},
\qquad
k=0,\ldots,B-1.
\]
Define the discrete temporal variation of the GVD trajectory under the
log-Euclidean metric as
\[
\mathcal{V}_{\mathrm{LE}}(\Delta_{1:B})
=
\sum_{b=1}^{B-1}
d_{\mathrm{LE}}^2(\Delta_{b+1},\Delta_b),
\]
where
\[
d_{\mathrm{LE}}(A,B)
=
\|\log A-\log B\|_F.
\]
Then
\[
\boxed{
\mathcal{V}_{\mathrm{LE}}(\Delta_{1:B})
=
\sum_{k=0}^{B-1}
\lambda_k
\|\bar z_k\|_2^2
}
\]
with
\[
\lambda_k
=
4\sin^2\left(\frac{\pi k}{2B}\right).
\]
Equivalently,
\[
\boxed{
\mathcal{V}_{\mathrm{LE}}(\Delta_{1:B})
=
\sum_{k=1}^{B-1}
4\sin^2\left(\frac{\pi k}{2B}\right)
\|\bar z_k\|_2^2
}
\]
since $\lambda_0=0$.

Hence, the DCT-II provides an exact orthogonal decomposition of the
log-Euclidean temporal variation of a GVD trajectory. The zero-frequency
mode contributes no temporal variation, while the weighting
$\lambda_k$ increases monotonically with DCT mode index $k$.
\end{theorem}

\begin{proof}
Since $\operatorname{svec}$ preserves the Frobenius inner product on
symmetric matrices,
\[
d_{\mathrm{LE}}^2(\Delta_{b+1},\Delta_b)
=
\|\log\Delta_{b+1}-\log\Delta_b\|_F^2
=
\|z_{b+1}-z_b\|_2^2.
\]
Therefore,
\[
\mathcal{V}_{\mathrm{LE}}(\Delta_{1:B})
=
\sum_{b=1}^{B-1}
\|z_{b+1}-z_b\|_2^2.
\]
Using the first-difference matrix $D_B$ from
Proposition~\ref{prop:dct_temporal_laplacian},
\[
\mathcal{V}_{\mathrm{LE}}(\Delta_{1:B})
=
\|D_B Z\|_F^2.
\]
Hence,
\[
\mathcal{V}_{\mathrm{LE}}(\Delta_{1:B})
=
\operatorname{tr}
\left(
Z^\top D_B^\top D_B Z
\right)
=
\operatorname{tr}
\left(
Z^\top L_B Z
\right).
\]
Since the DCT-II matrix is orthonormal,
\[
Z=C_B^\top \bar Z.
\]
Substituting this expression gives
\[
\begin{aligned}
\mathcal{V}_{\mathrm{LE}}(\Delta_{1:B})
&=
\operatorname{tr}
\left(
\bar Z^\top
C_B L_B C_B^\top
\bar Z
\right).
\end{aligned}
\]
By Proposition~\ref{prop:dct_temporal_laplacian},
\[
C_B L_B C_B^\top
=
\Lambda_B.
\]
Thus
\[
\mathcal{V}_{\mathrm{LE}}(\Delta_{1:B})
=
\operatorname{tr}
\left(
\bar Z^\top
\Lambda_B
\bar Z
\right).
\]
Because $\Lambda_B$ is diagonal,
\[
\mathcal{V}_{\mathrm{LE}}(\Delta_{1:B})
=
\sum_{k=0}^{B-1}
\lambda_k
\|\bar z_k\|_2^2.
\]
Finally, substituting
\[
\lambda_k
=
4\sin^2\left(\frac{\pi k}{2B}\right)
\]
gives the stated result.
\end{proof}

\paragraph{Interpretation.}
Theorem~\ref{thm:gvd_temporal_spectral_energy} gives the DCT-token
representation a direct geometric interpretation. The quantity
\[
E_k
=
4\sin^2\left(\frac{\pi k}{2B}\right)
\|\bar z_k\|_2^2
\]
is exactly the contribution of DCT mode $k$ to the discrete
log-Euclidean temporal variation of the GVD trajectory. The DC mode
$k=0$ describes time-invariant connectivity structure and contributes
zero temporal variation, whereas higher-order modes receive
progressively larger temporal-variation weights. Thus, the DCT does
not merely reorganize the trajectory into frequency coordinates; it
diagonalizes its intrinsic temporal variation under the log-Euclidean
geometry.

\begin{corollary}[Exact log-Euclidean error of spectral truncation]
\label{cor:dct_truncation}
Let $\Delta^{(K)}_{1:B}$ be obtained by retaining only DCT modes
$k=0,\ldots,K-1$ and setting all remaining coefficients to zero before
applying the inverse DCT and log-Euclidean decoder. Then
\[
\sum_{b=1}^{B}
d_{\mathrm{LE}}^2
\left(
\Delta_b,\Delta_b^{(K)}
\right)
=
\sum_{k=K}^{B-1}
\|\bar z_k\|_2^2.
\]
Moreover,
\[
\mathcal{V}_{\mathrm{LE}}(\Delta_{1:B})
-
\mathcal{V}_{\mathrm{LE}}(\Delta^{(K)}_{1:B})
=
\sum_{k=K}^{B-1}
\lambda_k\|\bar z_k\|_2^2.
\]
\end{corollary}

\begin{proof}
The first identity follows from the isometry of the log map,
$\operatorname{svec}$, and the orthonormal DCT together with
Parseval's identity. The second follows directly from
Theorem~\ref{thm:gvd_temporal_spectral_energy}.
\end{proof}

\begin{proposition}[Stable support attenuates local connectivity perturbations]
\label{prop:support_noise_contraction}
Let
\[
W\in\mathbb{S}_{++}^{d}
\]
be a correlation matrix and define
\[
\rho_W
=
\max_{i\neq j}|W_{ij}|.
\]
Since $W$ is positive definite with unit diagonal,
\[
0\leq\rho_W<1.
\]
Let a local covariance estimate satisfy
\[
\widehat J = J+E,
\]
where $E=E^\top$. Decompose the perturbation into diagonal and
off-diagonal components,
\[
E = E_{\mathrm{diag}}+E_{\mathrm{off}},
\]
where
\[
E_{\mathrm{diag}}
=
\operatorname{diag}(E),
\qquad
\operatorname{diag}(E_{\mathrm{off}})=0.
\]
Define
\[
\Delta=W\odot J,
\qquad
\widehat\Delta=W\odot\widehat J.
\]
Then
\[
\boxed{
\|\widehat\Delta-\Delta\|_F^2
\leq
\|E_{\mathrm{diag}}\|_F^2
+
\rho_W^2
\|E_{\mathrm{off}}\|_F^2.
}
\]
In particular,
\[
\|\widehat\Delta-\Delta\|_F
\leq
\|\widehat J-J\|_F,
\]
so the stable support never amplifies a covariance perturbation in
Frobenius norm. Moreover, if the perturbation is purely off-diagonal,
\[
E_{\mathrm{diag}}=0,
\]
then
\[
\boxed{
\|\widehat\Delta-\Delta\|_F
\leq
\rho_W
\|\widehat J-J\|_F,
}
\]
which is a strict contraction whenever $E\neq0$.
\end{proposition}

\begin{proof}
By linearity of the Hadamard product,
\[
\widehat\Delta-\Delta
=
W\odot E.
\]
Because $W$ is a correlation matrix,
\[
W_{ii}=1.
\]
Therefore the diagonal perturbation is unchanged,
\[
W\odot E_{\mathrm{diag}}
=
E_{\mathrm{diag}}.
\]
For the off-diagonal component,
\[
\begin{aligned}
\|W\odot E_{\mathrm{off}}\|_F^2
&=
\sum_{i\neq j}
W_{ij}^2E_{ij}^2\\
&\leq
\rho_W^2
\sum_{i\neq j}E_{ij}^2\\
&=
\rho_W^2
\|E_{\mathrm{off}}\|_F^2.
\end{aligned}
\]
Since the diagonal and off-diagonal components have disjoint support,
they are orthogonal under the Frobenius inner product. Hence
\[
\begin{aligned}
\|\widehat\Delta-\Delta\|_F^2
&=
\|W\odot E_{\mathrm{diag}}\|_F^2
+
\|W\odot E_{\mathrm{off}}\|_F^2\\
&\leq
\|E_{\mathrm{diag}}\|_F^2
+
\rho_W^2
\|E_{\mathrm{off}}\|_F^2.
\end{aligned}
\]
Since $\rho_W<1$, the stated consequences follow.
\end{proof}

%

\section{Spectral structure of GVD-CFM}
\label{app:spectral_theory}

This appendix provides the complete derivation underlying
Theorem~\ref{thm:spectral_cfm_main}.

\subsection{Log-Euclidean trajectory coordinates}

Let $\boldsymbol{\Delta}=(\Delta_1,\ldots,\Delta_B)$ with $\Delta_b\in\mathbb{S}_{++}^{d}$, and define $z_b=\operatorname{svec}(\log\Delta_b)\in\mathbb{R}^{m}$. The complete trajectory is $Z=[z_1,\ldots,z_B]^\top\in\mathbb{R}^{B\times m}$, and the orthonormal DCT-II acts along the temporal dimension, $\bar Z=C_BZ$ with $C_B^\top C_B=I_B$. Since $\operatorname{svec}$ is linear, mode $k$ is
\begin{equation}
    \bar z_k
    =
    \sum_{b=1}^{B}(C_B)_{kb}z_b
    =
    \operatorname{svec}
    \left(
        \sum_{b=1}^{B}(C_B)_{kb}\log\Delta_b
    \right).
\end{equation}
The DCT does not act within an individual connectivity matrix. It reorganizes the temporal evolution of the complete log-connectivity trajectory.

\subsection{Temporal covariance and the KLT}

For one log-connectivity coordinate $j$, define the temporal vector
\begin{equation}
    x^{(j)}
    =
    \begin{bmatrix}
        Z_{1,j} &
        \cdots &
        Z_{B,j}
    \end{bmatrix}^{\top}.
\end{equation}

Let
\begin{equation}
    K_j
    =
    \operatorname{Cov}(x^{(j)})
\end{equation}
and define the feature-averaged temporal covariance
\begin{equation}
    K_T
    =
    \frac{1}{m}
    \sum_{j=1}^{m}
    K_j.
\end{equation}

Because $K_T$ is symmetric positive semidefinite,
\begin{equation}
    K_T
    =
    U\Lambda U^\top,
\end{equation}
with orthogonal $U$ and diagonal
\begin{equation}
    \Lambda
    =
    \operatorname{diag}(\lambda_1,\ldots,\lambda_B).
\end{equation}

\begin{proposition}[Exact KLT decorrelation]
\label{prop:appendix_klt}
Let $x$ be zero mean with covariance $K_T=U\Lambda U^\top$. Then for
\begin{equation}
    y=U^\top x,
\end{equation}
we have
\begin{equation}
    \operatorname{Cov}(y)=\Lambda.
\end{equation}
Therefore
\begin{equation}
    \operatorname{Cov}(y_k,y_\ell)=0,
    \qquad
    k\neq \ell.
\end{equation}
\end{proposition}

\begin{proof}
\begin{align}
    \operatorname{Cov}(y)
    &=
    U^\top
    \operatorname{Cov}(x)
    U\\
    &=
    U^\top K_TU\\
    &=
    U^\top U\Lambda U^\top U\\
    &=
    \Lambda.
\end{align}
\end{proof}

\begin{figure}[t]
    \centering
    \includegraphics[width=\linewidth]{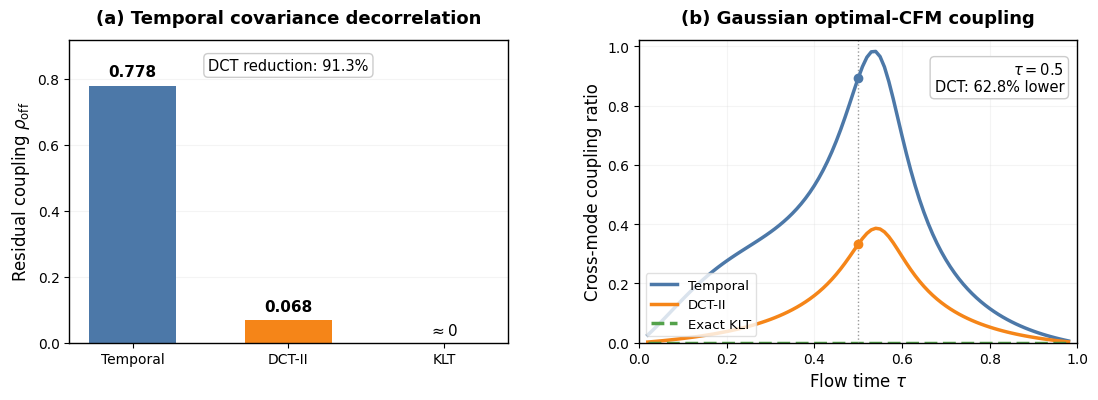}
    \caption{
    Empirical and theoretical motivation for the temporal DCT representation.
    \textbf{(a)} Residual off-diagonal temporal covariance,
    $\rho_{\mathrm{off}}(K)
    =\|K-\operatorname{diag}(K)\|_F/\|K\|_F$,
    for standardized log-GVD trajectories from BNCI2014\_001.
    The full orthonormal DCT-II reduces residual temporal covariance coupling from
    $0.778$ to $0.068$, corresponding to a $91.3\%$ reduction, while the empirical
    Karhunen--Lo\`eve transform (KLT) yields numerical zero.
    \textbf{(b)} Cross-mode coupling ratio of the population-optimal Gaussian
    conditional flow-matching field as a function of flow time $\tau$.
    The exact KLT removes second-order cross-mode coupling, while DCT-II
    substantially reduces it relative to the original temporal coordinates;
    at $\tau=0.5$, the DCT coupling ratio is $62.8\%$ lower.
    Together, these results show that DCT-II acts as an invertible,
    data-independent approximation to the temporal KLT that reduces the
    cross-mode dependency structure presented to the trajectory velocity model.
    }
    \label{fig:dct_cfm_decoupling}
\end{figure}

\begin{table*}[t]
\centering
\caption{Temporal-basis ablation. Results are dataset-balanced means over five datasets and three generator seeds. All variants retain all $B=100$ temporal coordinates. Random orthogonal uses a fixed random orthogonal basis and KLT the training-set empirical temporal Karhunen--Lo\`eve basis. The first four rows use the spectral-only velocity network and therefore isolate the basis; the last row is the full GVD-CFM. Ratios have ideal value 1. Best values are \textbf{bold}; second-best values are \underline{\textit{underlined italics}}.}
\label{tab:basis_ablation_summary}
\small
\setlength{\tabcolsep}{3.5pt}
\resizebox{\textwidth}{!}{%
\begin{tabular}{lcccccccc}
\toprule
Basis
& Rel. GVD-FID $\downarrow$
& Eva F1 $\uparrow$
& CAS AUC $\uparrow$
& CAS F1 $\uparrow$
& Temp. corr. $\uparrow$
& Lag-ACF $\uparrow$
& Energy $\to 1$
& Dyn. frac. $\to 1$ \\
\midrule
No DCT
& 1.092 & 0.540 & 0.772 & 0.704 & 0.404 & 0.666 & 0.973 & \underline{\textit{1.011}} \\
Random orthogonal
& 1.093 & 0.480 & 0.788 & 0.713 & 0.356 & 0.678 & \underline{\textit{1.017}} & 1.047 \\
KLT
& 1.033 & \underline{\textit{0.600}} & \textbf{0.795} & \textbf{0.729} & \textbf{0.951} & \textbf{0.998} & \textbf{1.004} & 1.014 \\
DCT, spectral only
& \underline{\textit{1.025}} & \underline{\textit{0.600}} & \underline{\textit{0.792}} & 0.721 & \underline{\textit{0.919}} & \underline{\textit{0.997}} & 0.975 & \textbf{1.010} \\
\midrule
GVD-CFM (DCT + temporal branch)
& \textbf{1.022} & \textbf{0.613} & 0.790 & \underline{\textit{0.725}} & 0.910 & \underline{\textit{0.997}} & 0.953 & 0.987 \\
\bottomrule
\end{tabular}}
\end{table*}

\subsection{The DCT as an approximate KLT}

The DCT-transformed temporal covariance is
\begin{equation}
    K_{\mathrm{DCT}}
    =
    C_BK_TC_B^\top.
\end{equation}

For temporally smooth or approximately stationary processes whose covariance
eigenvectors are close to cosine modes,
\begin{equation}
    C_BK_TC_B^\top
    \approx
    \Lambda.
\end{equation}

We quantify residual temporal coupling using
\begin{equation}
    \rho_{\mathrm{off}}(K)
    =
    \frac{
        \left\|
            K-\operatorname{diag}(K)
        \right\|_F
    }{
        \|K\|_F
    }.
\end{equation}

The residual coupling in the temporal and DCT bases is then
\begin{equation}
    \rho_{\mathrm{time}}
    =
    \rho_{\mathrm{off}}(K_T)
\end{equation}
and
\begin{equation}
    \rho_{\mathrm{DCT}}
    =
    \rho_{\mathrm{off}}
    \left(
        C_BK_TC_B^\top
    \right).
\end{equation}

A reduction
\begin{equation}
    \rho_{\mathrm{DCT}}
    <
    \rho_{\mathrm{time}}
\end{equation}
indicates that the DCT has reduced second-order temporal dependence.

\subsection{Separable trajectory covariance}

Let
\begin{equation}
    x=\operatorname{vec}(Z^\top)
    \in\mathbb{R}^{Bm},
\end{equation}
which stacks the windows $z_1,\ldots,z_B$, and suppose
\begin{equation}
    \Sigma
    =
    K_T\otimes K_S,
\end{equation}
where $K_S$ represents covariance between log-connectivity coordinates.

\begin{proposition}[Temporal KLT decouples temporal covariance]
\label{prop:appendix_kronecker}
If
\begin{equation}
    K_T=U\Lambda_TU^\top,
\end{equation}
then with
\begin{equation}
    Q=U^\top\otimes I_m,
\end{equation}
we have
\begin{equation}
    Q\Sigma Q^\top
    =
    \Lambda_T\otimes K_S.
\end{equation}
\end{proposition}

\begin{proof}
Using the mixed-product property of the Kronecker product,
\begin{align}
    Q\Sigma Q^\top
    &=
    (U^\top\otimes I_m)
    (K_T\otimes K_S)
    (U\otimes I_m)\\
    &=
    (U^\top K_TU)
    \otimes
    K_S\\
    &=
    \Lambda_T\otimes K_S.
\end{align}
\end{proof}

Therefore the KLT removes cross-temporal second-order coupling while leaving
the within-mode connectivity covariance $K_S$ unchanged.

Replacing $U^\top$ with $C_B$ yields the approximate DCT analogue
\begin{equation}
    \Sigma_{\mathrm{DCT}}
    =
    (C_BK_TC_B^\top)\otimes K_S.
\end{equation}

\subsection{Gaussian conditional flow matching}

Assume
\begin{equation}
    x_0\sim\mathcal{N}(0,I),
    \qquad
    x_1\sim\mathcal{N}(0,\Sigma),
\end{equation}
with $x_0$ and $x_1$ independent, and define
\begin{equation}
    x_\tau
    =
    (1-\tau)x_0+\tau x_1.
\end{equation}

The conditional target velocity is
\begin{equation}
    u_\tau=x_1-x_0.
\end{equation}

For squared-error conditional flow matching, the population-optimal field is
\begin{equation}
    v_\tau^\star(x)
    =
    \mathbb{E}
    [
        u_\tau
        \mid
        x_\tau=x
    ].
\end{equation}

\begin{proposition}[Optimal Gaussian CFM field]
\label{prop:appendix_gaussian_cfm}
The optimal field is linear:
\begin{equation}
    v_\tau^\star(x)
    =
    A_\tau x,
\end{equation}
where
\begin{equation}
    A_\tau
    =
    \left[
        \tau\Sigma-(1-\tau)I
    \right]
    \left[
        (1-\tau)^2I+\tau^2\Sigma
    \right]^{-1}.
\end{equation}
\end{proposition}

\begin{proof}
Since $(u_\tau,x_\tau)$ is jointly Gaussian,
\begin{equation}
    \mathbb{E}[u_\tau\mid x_\tau=x]
    =
    \operatorname{Cov}(u_\tau,x_\tau)
    \operatorname{Cov}(x_\tau)^{-1}x.
\end{equation}

Using independence of $x_0$ and $x_1$,
\begin{align}
    \operatorname{Cov}(x_\tau)
    &=
    (1-\tau)^2I+\tau^2\Sigma,
\end{align}
and
\begin{align}
    \operatorname{Cov}(u_\tau,x_\tau)
    &=
    \tau\Sigma-(1-\tau)I.
\end{align}

Substitution proves the result.
\end{proof}

Now let
\begin{equation}
    \Sigma
    =
    V\Lambda V^\top.
\end{equation}

\begin{corollary}[Mode-wise Gaussian flow]
\label{cor:appendix_modewise}
In coordinates
\begin{equation}
    y=V^\top x,
\end{equation}
the optimal field is diagonal:
\begin{equation}
    v_{\tau,k}^{\star}(y)
    =
    a_{\tau,k}y_k,
\end{equation}
with
\begin{equation}
    a_{\tau,k}
    =
    \frac{
        \tau\lambda_k-(1-\tau)
    }{
        (1-\tau)^2+\tau^2\lambda_k
    }.
\end{equation}
\end{corollary}

\begin{proof}
Substituting
\begin{equation}
    \Sigma=V\Lambda V^\top
\end{equation}
into Proposition~\ref{prop:appendix_gaussian_cfm} gives
\begin{equation}
    A_\tau
    =
    V
    \left[
        \tau\Lambda-(1-\tau)I
    \right]
    \left[
        (1-\tau)^2I+\tau^2\Lambda
    \right]^{-1}
    V^\top.
\end{equation}
Since the middle term is diagonal, the result follows.
\end{proof}

This proves Theorem~\ref{thm:spectral_cfm_main}. In an exact covariance
eigenbasis, the optimal Gaussian conditional flow field requires no
cross-mode coupling.

\begin{corollary}[Mode-wise Gaussian flow under optimal-transport coupling]
\label{cor:appendix_ot_modewise}
Let $x_0\sim\mathcal{N}(0,I)$ and $x_1\sim\mathcal{N}(0,\Sigma)$ with
$\Sigma=V\Lambda V^\top\succ0$, and couple them by the quadratic-cost
optimal-transport map $x_1=\Sigma^{1/2}x_0$
\citep{dowson1982frechet,villani2009optimal}. Along
$x_\tau=(1-\tau)x_0+\tau x_1$, the conditional target $u_\tau=x_1-x_0$ is a
deterministic function of $x_\tau$, and in coordinates $y=V^\top x$
\begin{equation}
    v_{\tau,k}^{\star}(y)
    =
    \frac{\sqrt{\lambda_k}-1}{(1-\tau)+\tau\sqrt{\lambda_k}}\,y_k.
\end{equation}
\end{corollary}

\begin{proof}
Let $M_\tau=(1-\tau)I+\tau\Sigma^{1/2}$. Since $\Sigma^{1/2}\succ0$, $M_\tau$
is positive definite for every $\tau\in[0,1]$, and $x_\tau=M_\tau x_0$. Hence
$x_0=M_\tau^{-1}x_\tau$ and
\begin{equation}
    u_\tau
    =
    (\Sigma^{1/2}-I)x_0
    =
    (\Sigma^{1/2}-I)M_\tau^{-1}x_\tau,
\end{equation}
so $\mathbb{E}[u_\tau\mid x_\tau=x]=(\Sigma^{1/2}-I)M_\tau^{-1}x$. Both factors
are diagonal in the basis $V$, with entries $\sqrt{\lambda_k}-1$ and
$(1-\tau)+\tau\sqrt{\lambda_k}$.
\end{proof}

Minibatch Sinkhorn coupling approximates this population coupling within each
class, with $\Sigma$ replaced by the class-conditional trajectory covariance.
Under both couplings the optimal field is diagonal in the KLT basis, so the
motivation for spectral coordinates does not depend on the coupling. The DCT is not generally the exact KLT, but whenever
\begin{equation}
    C_BK_TC_B^\top
\end{equation}
is more diagonal than $K_T$, the DCT reduces the amount of second-order
temporal interaction that must be represented by the velocity network.

\subsection{Residual coupling}

Write
\begin{equation}
    C_BK_TC_B^\top
    =
    D_T+E_T,
\end{equation}
where
\begin{equation}
    D_T
    =
    \operatorname{diag}
    \left(
        C_BK_TC_B^\top
    \right)
\end{equation}
and
\begin{equation}
    E_T
    =
    C_BK_TC_B^\top-D_T.
\end{equation}

Define
\begin{equation}
    \epsilon_{\mathrm{DCT}}
    =
    \frac{
        \|E_T\|_F
    }{
        \|C_BK_TC_B^\top\|_F
    }.
\end{equation}

For the exact KLT,
\begin{equation}
    \epsilon_{\mathrm{KLT}}=0.
\end{equation}

Hence $\epsilon_{\mathrm{DCT}}$ measures the residual second-order coupling
remaining after the fixed spectral transform.

\subsection{Orthogonal invariance of the GVD-CFM representation}

\begin{proposition}[DCT isometry]
\label{prop:appendix_dct_isometry}
For any trajectory $Z$,
\begin{equation}
    \|C_BZ\|_F
    =
    \|Z\|_F.
\end{equation}
Likewise, for any two velocity fields $V$ and $U$,
\begin{equation}
    \|C_B(V-U)\|_F^2
    =
    \|V-U\|_F^2.
\end{equation}
\end{proposition}

\begin{proof}
Since
\begin{equation}
    C_B^\top C_B=I_B,
\end{equation}
\begin{align}
    \|C_BZ\|_F^2
    &=
    \operatorname{tr}
    \left(
        Z^\top C_B^\top C_BZ
    \right)\\
    &=
    \operatorname{tr}(Z^\top Z)\\
    &=
    \|Z\|_F^2.
\end{align}
\end{proof}

Therefore the DCT does not alter the Euclidean flow-matching metric or reduce
the objective by rescaling the trajectory. Its effect is purely a change of
coordinates.

\subsection{Energy organization}

The KLT additionally orders directions by variance. If
\begin{equation}
    \lambda_1\geq\lambda_2\geq\cdots\geq\lambda_B,
\end{equation}
then the first $r$ KLT modes retain the largest possible amount of expected
second-order energy among all orthonormal $r$-dimensional projections.

When the DCT approximates the KLT, temporally persistent structure is therefore
expected to concentrate toward lower-order cosine modes. GVD-CFM does not
truncate this representation: all $B$ modes are retained. Thus this is an
organization of temporal variation rather than a dimensionality-reduction
argument.

\section{Continuous-Grid Decoding as Band-Limited Interpolation}
\label{app:continuous_resolution_full}
\label{app:continuous_cosine_basis}

\subsection{Cosine-basis decoding of GVD trajectories}

A distinctive consequence of the DCT-coordinate representation is that a generated trajectory is not restricted to the temporal grid used during training. The generated DCT coefficients define a finite cosine trajectory over normalized time, and may therefore be evaluated at an arbitrary number of temporal locations without retraining the generative model. We test this with an intentionally low native resolution: the spectral-only GVD-CFM network is trained once at $B_{\mathrm{train}}=25$ windows, and the same generated coefficients are decoded at $M\in\{25,50,100,200,400\}$, with the real reference trajectories recomputed natively at each $M$. The protocol is given in Section~\ref{app:continuous_protocol}.

Moderate upsampling preserves both discriminative and temporal structure. A $4\times$ denser grid ($M=100$) changes generated-to-real AUC only from $0.803$ to $0.802$, and at $M=200$ it remains $0.794$ with lag-ACF agreement $0.959$. Every generated matrix stays SPD at every tested resolution. Evaluating at $M>B$ is interpolation of a finite-bandwidth function: it does not increase the generated temporal bandwidth. Table~\ref{tab:continuous_resolution_macro_full} reports the full comparison and the finite-bandwidth limit that appears at $M=400$.

The composite diffeomorphism introduced in Section~\ref{sec:chart} maps a GVD trajectory
to the DCT-coordinate representation
\[
    \bar{Z}
    =
    C_B \widetilde{Z}
    \in \mathbb{R}^{B\times m},
\]
where $B$ is the temporal resolution of the GVD trajectory,
$m=d(d+1)/2$, $\widetilde{Z}$ denotes the standardized log-Euclidean
trajectory, and $C_B$ is the orthonormal DCT-II matrix. In the standard
GVD-CFM decoder, a generated coefficient tensor $\widehat{\bar{Z}}$ is
mapped back to the original $B$ temporal locations using the inverse DCT,
\[
    \widehat{\widetilde{Z}}
    =
    C_B^{\top}\widehat{\bar{Z}}.
\]
We additionally exploit the fact that the inverse DCT is an evaluation of a
finite cosine basis. Rather than restricting decoding to the original
$B$ DCT sampling locations, the generated coefficients can therefore be
interpreted as defining a continuous trajectory and evaluated on an
arbitrary temporal grid.

\paragraph{Continuous cosine representation.}
Let
\[
    \widehat{\bar{Z}}
    =
    \begin{bmatrix}
        \widehat{\mathbf{c}}_0^{\top}\\
        \vdots\\
        \widehat{\mathbf{c}}_{B-1}^{\top}
    \end{bmatrix},
    \qquad
    \widehat{\mathbf{c}}_k\in\mathbb{R}^{m},
\]
denote the DCT-coordinate trajectory generated by GVD-CFM. We associate
these coefficients with the continuous standardized tangent trajectory
\begin{equation}
    \widehat{\widetilde{\mathbf{z}}}(\xi)
    =
    \sum_{k=0}^{B-1}
    \alpha_k\,
    \widehat{\mathbf{c}}_k
    \cos(\pi k\xi),
    \qquad
    \xi\in[0,1],
    \label{eq:continuous_gvd_cosine}
\end{equation}
where $\xi$ denotes normalized EEG trajectory time and is distinct from the
flow-time variable $\tau$. The orthonormal DCT-II normalization is
\begin{equation}
    \alpha_k
    =
    \begin{cases}
        B^{-1/2}, & k=0,\\[3pt]
        \sqrt{2/B}, & k>0.
    \end{cases}
    \label{eq:continuous_gvd_dct_norm}
\end{equation}
Hence, GVD-CFM generates a finite set of frequency coefficients, but those
coefficients define a function over continuous trajectory time. The output
temporal resolution is consequently determined during decoding rather than
being restricted to the resolution at which the coefficient representation
was generated.

\paragraph{Evaluation at an arbitrary temporal resolution.}
Suppose that the generated trajectory is to be evaluated at $M$ temporal
locations. We use the centered grid
\begin{equation}
    \xi_j^{(M)}
    =
    \frac{j+\frac{1}{2}}{M},
    \qquad
    j=0,\ldots,M-1.
    \label{eq:continuous_gvd_grid}
\end{equation}
Evaluating Equation~\ref{eq:continuous_gvd_cosine} on this grid gives
\begin{equation}
    \widehat{\widetilde{\mathbf{z}}}^{(M)}_j
    =
    \sum_{k=0}^{B-1}
    \alpha_k\,
    \widehat{\mathbf{c}}_k
    \cos\left[
        \pi k
        \frac{j+\frac{1}{2}}{M}
    \right].
    \label{eq:continuous_gvd_eval}
\end{equation}
Equivalently, define the continuous DCT synthesis matrix
\begin{equation}
    A_{M\leftarrow B}
    \in
    \mathbb{R}^{M\times B},
    \qquad
    \left[A_{M\leftarrow B}\right]_{j,k}
    =
    \alpha_k
    \cos\left[
        \pi k
        \frac{j+\frac{1}{2}}{M}
    \right].
    \label{eq:continuous_gvd_synthesis_matrix}
\end{equation}
The complete $M$-point standardized trajectory is then obtained by the
single matrix operation
\begin{equation}
    \boxed{
        \widehat{\widetilde{Z}}^{(M)}
        =
        A_{M\leftarrow B}\widehat{\bar{Z}}
    }.
    \label{eq:continuous_gvd_matrix_decode}
\end{equation}
Thus,
\[
    \widehat{\bar{Z}}\in\mathbb{R}^{B\times m},
    \qquad
    \widehat{\widetilde{Z}}^{(M)}
    \in\mathbb{R}^{M\times m},
\]
and $M$ does not need to equal $B$. The same generated GVD trajectory can
therefore be sampled at its native resolution, at a denser temporal
resolution, or at a coarser temporal resolution without retraining
GVD-CFM.

\paragraph{Exact agreement with the original GVD-CFM decoder.}
The continuous formulation is a strict extension of the inverse-DCT decoder
already used by GVD-CFM. When $M=B$,
\[
    \xi_j^{(B)}
    =
    \frac{j+\frac{1}{2}}{B},
\]
and therefore
\begin{equation}
    \left[A_{B\leftarrow B}\right]_{j,k}
    =
    \alpha_k
    \cos\left[
        \frac{\pi}{B}
        \left(j+\frac{1}{2}\right)k
    \right].
\end{equation}
This is exactly the orthonormal inverse DCT-II basis, giving
\begin{equation}
    A_{B\leftarrow B}
    =
    C_B^{\top}.
    \label{eq:continuous_gvd_idct_equivalence}
\end{equation}
Consequently,
\begin{equation}
    \widehat{\widetilde{Z}}^{(B)}
    =
    C_B^{\top}\widehat{\bar{Z}},
\end{equation}
which recovers the original GVD-CFM decoding operation exactly. Continuous
basis decoding therefore does not alter the learned representation or
introduce an additional generative model; it generalizes the temporal
evaluation of the existing DCT representation.

\paragraph{Mapping the continuous trajectory back to the SPD manifold.}
After continuous basis evaluation, the same inverse log-Euclidean mapping
used by GVD-CFM is applied independently at each temporal location. Using
the training-set coordinate statistics $\mu,\sigma\in\mathbb{R}^{m}$,
\begin{equation}
    \widehat{\mathbf{z}}^{(M)}_j
    =
    \widehat{\widetilde{\mathbf{z}}}^{(M)}_j
    \odot \sigma
    +
    \mu.
    \label{eq:continuous_gvd_unstandardize}
\end{equation}
The corresponding symmetric tangent matrix is
\begin{equation}
    \widehat{X}^{(M)}_j
    =
    \operatorname{svec}^{-1}
    \left(
        \widehat{\mathbf{z}}^{(M)}_j
    \right),
\end{equation}
and the GVD matrix is reconstructed as
\begin{equation}
    \widehat{\Delta}^{(M)}_j
    =
    \exp\left(
        \widehat{X}^{(M)}_j
    \right),
    \qquad
    j=0,\ldots,M-1.
    \label{eq:continuous_gvd_spd}
\end{equation}
Since $\widehat{X}^{(M)}_j$ is symmetric, its matrix exponential is strictly
positive definite. Hence,
\begin{equation}
    \widehat{\Delta}^{(M)}_j
    \in
    \mathbb{S}_{++}^{d}
    \qquad
    \forall j,
\end{equation}
so changing the temporal evaluation resolution does not compromise the
manifold constraint.

\paragraph{Interpretation.}
The DCT-token representation used by GVD-CFM can therefore be viewed not
only as a convenient Euclidean coordinate system for a discrete product-SPD
trajectory, but also as the coefficient representation of a continuous
cosine trajectory. All $B$ generated modes are retained; no spectral
truncation is introduced. Continuous decoding changes only the set of
temporal locations at which this same generated trajectory is evaluated.

This separates \emph{generative resolution} from \emph{sampling resolution}:
GVD-CFM learns the joint distribution of $B$ cosine modes, while the
resulting trajectory may be evaluated at any desired number $M$ of temporal
locations. In particular, choosing $M>B$ produces a denser realization of
the generated dynamic connectivity trajectory while preserving the same
underlying DCT coefficients and the SPD geometry of every reconstructed
GVD matrix.

\subsection{Method}
\label{app:continuous_protocol}

The DCT representation used by GVD-CFM permits a generated coefficient trajectory to be evaluated on a temporal grid different from the one used during training. We evaluate this property directly by deliberately training the generative model at a low temporal resolution and decoding its outputs on progressively denser grids.

All experiments in this section use
\[
B_{\mathrm{train}}=25.
\]
For a generated DCT-coordinate trajectory
\[
\widehat{\bar Z}
=
\begin{bmatrix}
\widehat c_0^\top\\
\vdots\\
\widehat c_{B-1}^\top
\end{bmatrix}
\in\mathbb{R}^{B\times m},
\qquad B=25,
\]
the continuous standardized tangent trajectory is
\[
\widehat z(\xi)
=
\sum_{k=0}^{B-1}
\alpha_k\widehat c_k\cos(\pi k\xi),
\qquad
\xi\in[0,1],
\]
where
\[
\alpha_k=
\begin{cases}
B^{-1/2}, & k=0,\\[2mm]
\sqrt{2/B}, & k>0.
\end{cases}
\]
For a target resolution $M$, we evaluate this function on the centered grid
\[
\xi_j^{(M)}
=
\frac{j+\frac12}{M},
\qquad
j=0,\ldots,M-1.
\]
Equivalently,
\[
\widehat Z^{(M)}
=
A_{M\leftarrow B}\widehat{\bar Z},
\]
with
\[
[A_{M\leftarrow B}]_{j,k}
=
\alpha_k
\cos\left[
\pi k\frac{j+\frac12}{M}
\right].
\]
The resulting standardized tangent coordinates are inverse-standardized,
\[
\widehat z_j^{(M)}
=
\widehat{\widetilde z}_j^{(M)}\odot\sigma+\mu,
\]
mapped back to symmetric matrices using $\operatorname{svec}^{-1}$,
\[
\widehat X_j^{(M)}
=
\operatorname{svec}^{-1}
\left(
\widehat z_j^{(M)}
\right),
\]
and finally reconstructed on the SPD manifold through
\[
\widehat\Delta_j^{(M)}
=
\exp\left(
\widehat X_j^{(M)}
\right).
\]
Because $\widehat X_j^{(M)}$ is symmetric,
\[
\widehat\Delta_j^{(M)}
\in\mathcal{S}_{++}^d
\qquad
\forall j,M.
\]
We evaluate
\[
M\in\{25,50,100,200,400\},
\]
corresponding to $1\times$, $2\times$, $4\times$, $8\times$, and $16\times$ the training-grid density.

A single generated coefficient tensor $\widehat{\bar Z}$ is reused across all values of $M$. Differences across resolutions therefore arise only from evaluating the same learned continuous cosine trajectory on different temporal grids. No GVD-CFM model is retrained for any target resolution.

For each $M$, the corresponding real reference trajectories are recomputed directly from raw EEG using $M$ GVD windows. The experiment therefore compares
\[
\text{GVD-CFM trained at }B=25
\quad\longrightarrow\quad
\text{continuous decode at }M
\]
against
\[
\text{real EEG}
\quad\longrightarrow\quad
\text{native GVD construction at }M.
\]
The same pooled cross-session/cross-run train--test protocol used in the main benchmark is retained. Results are reported for BNCI2014-001, BNCI2014-002, BNCI2015-001, Shin2017A, and Zhou2016.

\subsection{Dataset-balanced results}

\begin{table*}[t]
\centering
\caption{
Dataset-balanced continuous-resolution results on BNCI2014\_001, BNCI2014\_002, BNCI2015\_001, Shin2017A, and Zhou2016.
A single spectral-only GVD-CFM is trained at $B_{\mathrm{train}}=25$ for each dataset and decoded at the indicated target resolution $M$ without retraining.
}
\label{tab:continuous_resolution_macro_full}
\resizebox{\textwidth}{!}{
\begin{tabular}{c|ccccccccc}
\toprule
$M$
& Rel.\ Fr\'echet
& Temp.\ corr.
& Lag-ACF
& Dyn.\ energy
& SPD
& Gen$\rightarrow$Real AUC
& Gen$\rightarrow$Real F1
& Real$\rightarrow$Real AUC
& Real$\rightarrow$Real F1 \\
\midrule
25
& 1.098
& 0.944
& 0.995
& 0.978
& 1.000
& 0.803
& 0.739
& 0.838
& 0.764 \\
50
& 0.983
& 0.853
& 0.960
& 0.595
& 1.000
& 0.804
& 0.739
& 0.838
& 0.761 \\
100
& 0.919
& 0.860
& 0.960
& 0.350
& 1.000
& 0.802
& 0.734
& 0.831
& 0.756 \\
200
& 0.880
& 0.780
& 0.959
& 0.199
& 1.000
& 0.794
& 0.715
& 0.832
& 0.760 \\
400
& 0.872
& 0.506
& 0.822
& 0.105
& 1.000
& 0.785
& 0.647
& 0.810
& 0.732 \\
\bottomrule
\end{tabular}}
\end{table*}

The principal observation is that substantial temporal densification remains possible without retraining. Relative to native-resolution decoding,
\[
\mathrm{AUC}_{25\rightarrow25}=0.803,
\qquad
\mathrm{F1}_{25\rightarrow25}=0.739,
\]
while $4\times$ denser decoding gives
\[
\mathrm{AUC}_{25\rightarrow100}=0.802,
\qquad
\mathrm{F1}_{25\rightarrow100}=0.734.
\]
Thus,
\[
\Delta\mathrm{AUC}=-0.001,
\qquad
\Delta\mathrm{F1}=-0.005.
\]
At the same resolution, temporal-correlation agreement remains $0.860$ and lag-ACF agreement remains $0.960$.

Even at $M=200$, corresponding to $8\times$ denser evaluation, generated-to-real AUC remains $0.794$ and lag-ACF agreement remains $0.959$. The corresponding temporal-correlation agreement is $0.780$.

The real-to-real classifier provides useful context for these changes. At $M=100$, its dataset-balanced AUC is $0.831$, compared with $0.802$ for generated-to-real classification, a gap of
\[
0.029.
\]
At $M=200$, the corresponding values are $0.832$ and $0.794$, respectively.

All generated matrices remain SPD at every tested resolution:
\[
\text{SPD validity}=1.000
\]
for every dataset and every $M$.

At the highest tested resolution, $M=400$, generated-to-real AUC remains $0.785$, whereas temporal-correlation agreement decreases to $0.506$. This is consistent with the finite-bandwidth limitation of evaluating a trajectory represented by only $25$ learned cosine modes on a substantially denser temporal grid.

\subsection{Per-dataset results}

\begin{table*}[t]
\centering
\caption{
Per-dataset continuous-resolution results for a single spectral-only GVD-CFM trained at $B_{\mathrm{train}}=25$.
The same generated DCT coefficient tensor is evaluated at $M\in\{25,50,100,200,400\}$ without retraining, while the corresponding real GVD trajectories are recomputed natively at each target resolution.
}
\label{tab:continuous_resolution_per_dataset}
\resizebox{\textwidth}{!}{
\begin{tabular}{ll|cccccccc}
\toprule
Dataset
& $M$
& Rel.\ Fr\'echet
& Temp.\ corr.
& Lag-ACF
& Dyn.\ energy
& Gen$\rightarrow$Real AUC
& Gen$\rightarrow$Real F1
& Real$\rightarrow$Real AUC
& Real$\rightarrow$Real F1 \\
\midrule

BNCI2014-001
& 25  & 1.163 & 0.977 & 0.995 & 1.041 & 0.814 & 0.724 & 0.863 & 0.775 \\
& 50  & 1.021 & 0.875 & 0.977 & 0.622 & 0.824 & 0.748 & 0.866 & 0.784 \\
& 100 & 0.979 & 0.890 & 0.958 & 0.360 & 0.816 & 0.727 & 0.857 & 0.783 \\
& 200 & 0.958 & 0.778 & 0.959 & 0.195 & 0.806 & 0.704 & 0.821 & 0.739 \\
& 400 & 0.943 & 0.440 & 0.558 & 0.103 & 0.778 & 0.497 & 0.797 & 0.710 \\
\midrule

BNCI2014-002
& 25  & 1.008 & 0.911 & 0.999 & 0.964 & 0.791 & 0.746 & 0.797 & 0.721 \\
& 50  & 0.841 & 0.822 & 0.984 & 0.560 & 0.792 & 0.742 & 0.802 & 0.714 \\
& 100 & 0.790 & 0.778 & 0.972 & 0.319 & 0.787 & 0.730 & 0.799 & 0.721 \\
& 200 & 0.768 & 0.643 & 0.961 & 0.184 & 0.772 & 0.714 & 0.841 & 0.771 \\
& 400 & 0.789 & 0.377 & 0.884 & 0.097 & 0.769 & 0.736 & 0.792 & 0.717 \\
\midrule

BNCI2015-001
& 25  & 1.070 & 0.963 & 0.998 & 0.882 & 0.766 & 0.697 & 0.793 & 0.714 \\
& 50  & 1.269 & 0.889 & 0.990 & 0.526 & 0.764 & 0.694 & 0.794 & 0.723 \\
& 100 & 1.246 & 0.894 & 0.988 & 0.291 & 0.764 & 0.700 & 0.779 & 0.701 \\
& 200 & 1.183 & 0.806 & 0.988 & 0.160 & 0.756 & 0.682 & 0.763 & 0.693 \\
& 400 & 1.145 & 0.439 & 0.923 & 0.080 & 0.744 & 0.681 & 0.753 & 0.682 \\
\midrule

Shin2017A
& 25  & 1.199 & 0.945 & 0.984 & 1.219 & 0.673 & 0.621 & 0.759 & 0.691 \\
& 50  & 0.898 & 0.840 & 0.896 & 0.812 & 0.675 & 0.625 & 0.752 & 0.676 \\
& 100 & 0.771 & 0.889 & 0.936 & 0.532 & 0.673 & 0.631 & 0.747 & 0.672 \\
& 200 & 0.733 & 0.886 & 0.942 & 0.325 & 0.671 & 0.628 & 0.767 & 0.690 \\
& 400 & 0.723 & 0.843 & 0.829 & 0.187 & 0.677 & 0.553 & 0.742 & 0.660 \\
\midrule

Zhou2016
& 25  & 1.052 & 0.926 & 1.000 & 0.783 & 0.968 & 0.907 & 0.977 & 0.917 \\
& 50  & 0.888 & 0.838 & 0.953 & 0.451 & 0.965 & 0.889 & 0.973 & 0.910 \\
& 100 & 0.806 & 0.850 & 0.946 & 0.245 & 0.965 & 0.882 & 0.971 & 0.907 \\
& 200 & 0.756 & 0.787 & 0.945 & 0.130 & 0.966 & 0.850 & 0.969 & 0.902 \\
& 400 & 0.761 & 0.430 & 0.918 & 0.062 & 0.955 & 0.772 & 0.963 & 0.887 \\
\bottomrule
\end{tabular}}
\end{table*}

The per-dataset results show that resolution transfer is not driven by a single dataset. BNCI2014-001 preserves generated-to-real AUC from $0.814$ at $M=25$ to $0.816$ at $M=100$ and $0.806$ at $M=200$. BNCI2014-002 similarly remains near $0.79$ through $M=100$, while BNCI2015-001 changes only from $0.766$ at $M=25$ to $0.764$ at $M=100$.

Shin2017A also shows stable discriminative performance under substantial densification. Its generated-to-real AUC is
\[
0.673,\;0.675,\;0.673,\;0.671,\;0.677
\]
at $M=25,50,100,200,400$, respectively. Thus, the classifier-level utility of the generated trajectories is essentially unchanged across the entire range of target resolutions. Temporal-correlation agreement is $0.945$ at native resolution, $0.889$ at $M=100$, and $0.886$ at $M=200$. Lag-ACF agreement remains $0.936$ at $M=100$ and $0.942$ at $M=200$.

Zhou2016 shows particularly strong preservation of discriminative structure:
\[
0.968,\;0.965,\;0.965,\;0.966
\]
generated-to-real AUC at $M=25,50,100,200$, respectively, and remains at $0.955$ even at $M=400$.

Across the five datasets, the effect of increasing $M$ is therefore more apparent in the temporal-dynamics diagnostics than in generated-to-real classification. Dataset-balanced generated-to-real AUC changes only from $0.803$ at $M=25$ to $0.802$ at $M=100$, $0.794$ at $M=200$, and $0.785$ at $M=400$. In contrast, the dataset-balanced dynamic-energy ratio decreases from $0.978$ at $M=25$ to $0.350$ at $M=100$, $0.199$ at $M=200$, and $0.105$ at $M=400$.

This behavior is expected from finite-bandwidth cosine decoding. Increasing $M$ evaluates the same $B_{\mathrm{train}}=25$ learned cosine modes on a denser temporal grid; it does not introduce additional high-frequency modes. Consequently, continuous-grid decoding can preserve class-discriminative and broad temporal structure under substantial densification, while increasingly fine-scale dynamic amplitudes cannot match native high-resolution GVD trajectories indefinitely.

At $M=400$, corresponding to $16\times$ the training-grid density, the limitation is visible in the dataset-balanced temporal-correlation agreement of $0.506$ and dynamic-energy ratio of $0.105$. Nevertheless, generated-to-real AUC remains $0.785$, lag-ACF agreement remains $0.822$, and every generated matrix remains SPD. This supports interpreting the procedure as continuous-resolution evaluation of a finite-bandwidth trajectory rather than recovery of temporal frequencies absent from the original $25$-mode representation.
\subsection{Interpretation and limitation}

The experiment demonstrates that output temporal resolution is not fixed by the grid used during GVD-CFM training. In particular, a model trained on only $25$ temporal positions can be evaluated at $100$ or $200$ positions while retaining much of its discriminative and temporal structure.

However, increasing $M$ does not create additional temporal bandwidth. The learned trajectory contains only the $B=25$ cosine modes generated by the model:
\[
\widehat z(\xi)
=
\sum_{k=0}^{24}
\alpha_k\widehat c_k\cos(\pi k\xi).
\]
Hence, evaluating this function on a denser grid provides a finer sampling of the same finite-dimensional trajectory rather than synthesizing additional high-frequency modes.

This distinction is visible at $M=400$. Although generated-to-real AUC remains $0.785$ and SPD validity remains perfect, temporal-correlation agreement falls to $0.506$. This regime corresponds to $16\times$ denser temporal evaluation than training and exposes the finite-bandwidth limitation of the $25$-mode representation.

We therefore use the term \emph{continuous-resolution decoding} to mean that a generated GVD trajectory can be evaluated on arbitrary temporal grids without retraining, while explicitly not claiming that arbitrarily dense evaluation recovers temporal frequencies absent from the learned coefficient representation.

\section{Regional and Physiological Plausibility of Generated GVD Dynamics}
\label{app:regional_plausibility}

Here, we ask whether GVD-CFM reproduces physiologically structured spatiotemporal connectivity rather than only aggregate statistics. We performed the analysis on BNCI2014\_001 by grouping its 22 scalp electrodes into four broad regions: Frontal/FC, Central, Centro-parietal and Parietal/Occipital. For each pair of regions we computed the mean magnitude of the corresponding GVD edges for held-out real and generated trials, and followed representative regional interactions across all $B=100$ windows (Figure~\ref{fig:regional_gvd_dynamics} in the main text). Finally, we compared a generated single-trial sequence with its nearest held-out real trial.

\begin{figure}[t]
\centering
\includegraphics[width=0.82\textwidth]{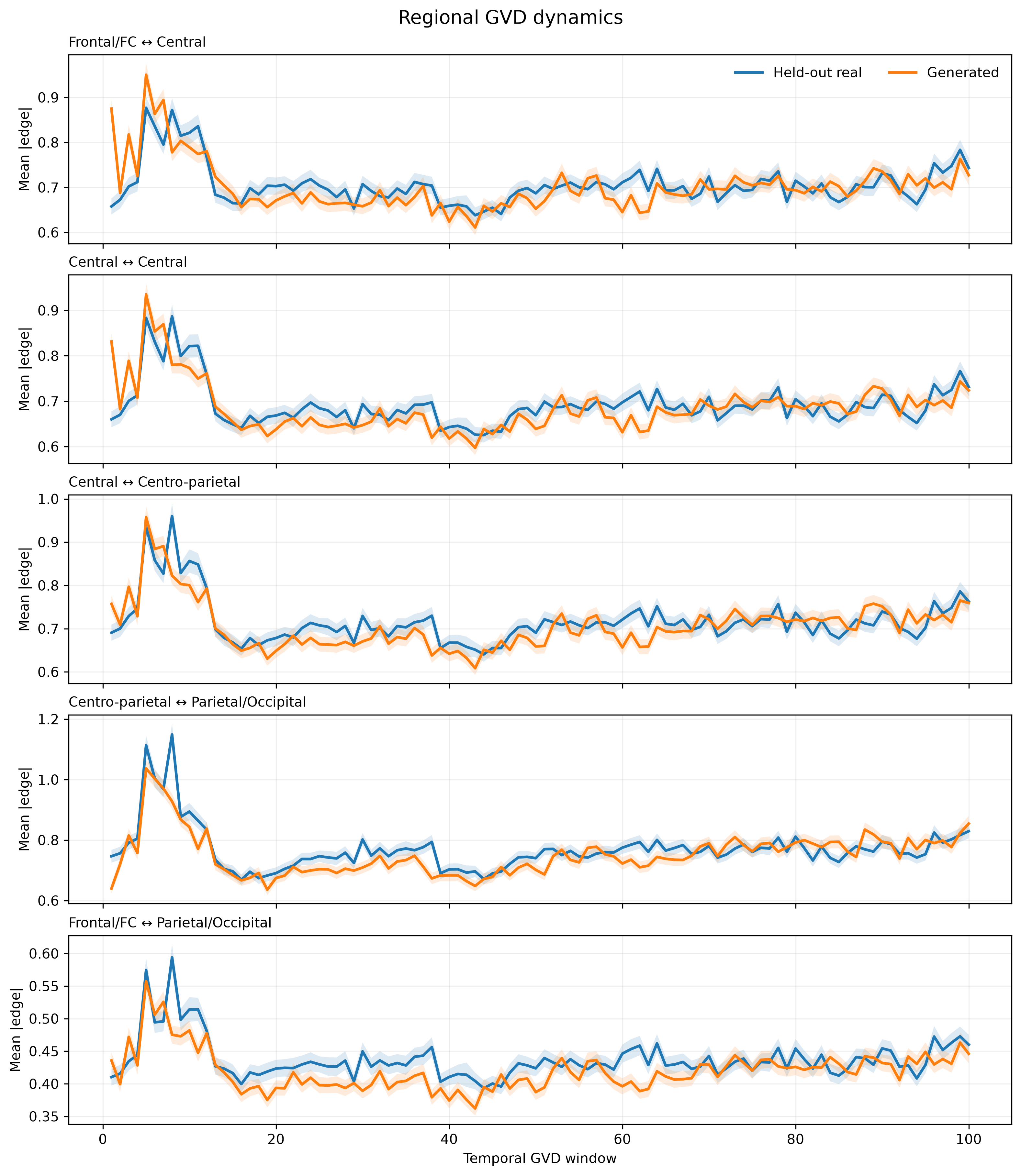}
\caption{\textbf{Regional GVD dynamics on BNCI2014\_001.} Mean GVD edge magnitude of five representative regional interactions over the $B=100$ windows for held-out real and GVD-CFM trajectories. The generated trajectories reproduce the early transient, the fall to a lower connectivity regime, the mid-trial reduction and the gradual recovery of the real data, and these changes are aligned across regions. Shaded regions show the variability around each mean trajectory.}
\label{fig:regional_gvd_dynamics}
\end{figure}
\paragraph{Regional connectivity organization.}
Figure~\ref{fig:regional_gvd_matrix} shows that the generated data preserve the large-scale regional organization of held-out EEG. Mean Frontal/FC connectivity is $0.847$ in held-out real data and $0.835$ in generated data, and Parietal/Occipital connectivity is $0.894$ and $0.880$. Central--Centro-parietal connectivity is similarly preserved ($0.717$ real against $0.706$ generated), while the weaker Frontal/FC--Parietal/Occipital interaction remains weak ($0.437$ against $0.420$). GVD-CFM therefore preserves the relative spatial organization of interactions across the scalp, with strong local and adjacent-region coupling and weaker long-range frontal-to-posterior coupling. Motor-imagery EEG is classically associated with strong modulation of sensorimotor activity over central and neighboring centro-parietal regions, with distributed involvement of frontal and posterior areas \citep{pfurtscheller1999erd}. The preserved central and centro-parietal structure is consistent with this.

\paragraph{Temporal regional dynamics.}
Across the Frontal/FC--Central, Central--Central, Central--Centro-parietal, Centro-parietal--Parietal/Occipital and Frontal/FC--Parietal/Occipital interactions of Figure~\ref{fig:regional_gvd_dynamics}, real and generated data both show a pronounced early transient followed by a fall to a lower connectivity regime. Around the middle of the trial several interactions fall further, after which connectivity gradually recovers and continues to fluctuate. These changes occur together across related regions; the generated interactions do not fluctuate independently around a static mean. Such non-stationarity is compatible with the sequence of preparation, imagery and recovery within a motor-imagery trial.

\paragraph{Single-trial connectivity reconfiguration.}
Population-level agreement could in principle arise from a generator that learned only an average connectivity template. Figure~\ref{fig:single_trial_matrix_sequence} therefore compares one generated trial with its nearest held-out real trial over ten consecutive intervals. The real sequence reconfigures substantially over time, with periods of strong, spatially distributed connectivity alternating with weaker or more concentrated interactions. The generated sequence shows a comparable degree of restructuring and passes through a series of distinct channel-level patterns rather than holding a fixed matrix. The two sequences are not identical, which is consistent with generation of a new trajectory from the learned distribution rather than reproduction of a training trial.

\begin{figure*}[t]
    \centering
    \includegraphics[width=0.99\textwidth]{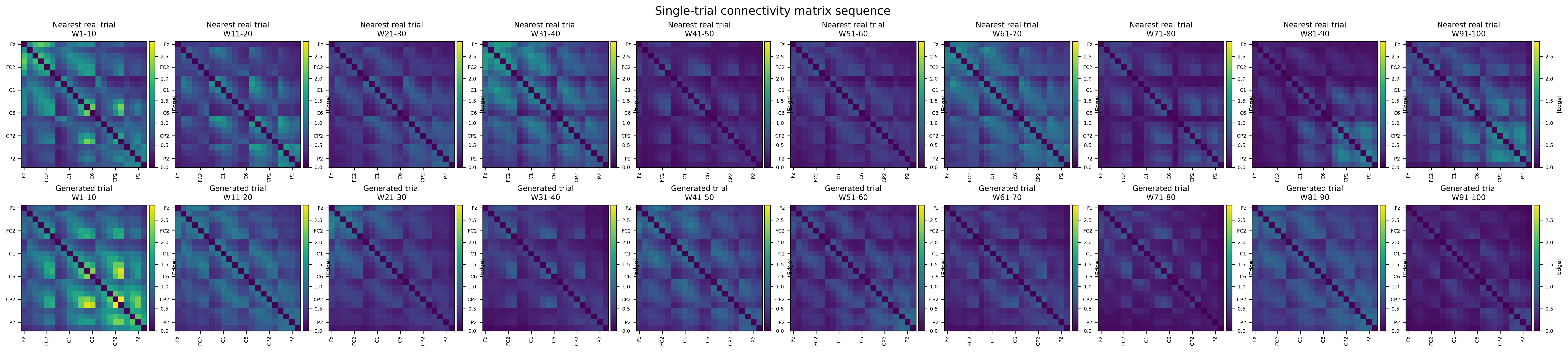}
    \caption{
    \textbf{Single-trial GVD connectivity-matrix sequence on BNCI2014\_001.}
    One generated trial and its nearest held-out real trial are shown over ten
    consecutive temporal intervals. Both sequences exhibit substantial
    time-dependent reconfiguration of channel-level connectivity rather than
    a static graph template.
    }
    \label{fig:single_trial_matrix_sequence}
\end{figure*}

\paragraph{Node-strength dynamics.}
The scalp node-strength maps of Figure~\ref{fig:single_trial_scalp_sequence} give a more interpretable view of the same trials. In the held-out trial, different intervals emphasize central, centro-parietal, posterior or more broadly distributed patterns. The generated trial shows similarly heterogeneous organization: high-strength regions shift across windows rather than remaining locked to one set of electrodes, prominent modulation recurs around central and centro-parietal locations, and frontal and posterior contributions vary over time. The model is therefore not scaling a fixed graph uniformly; the relative contribution of electrode groups changes over time.

\begin{figure*}[t]
    \centering
    \includegraphics[width=0.99\textwidth]{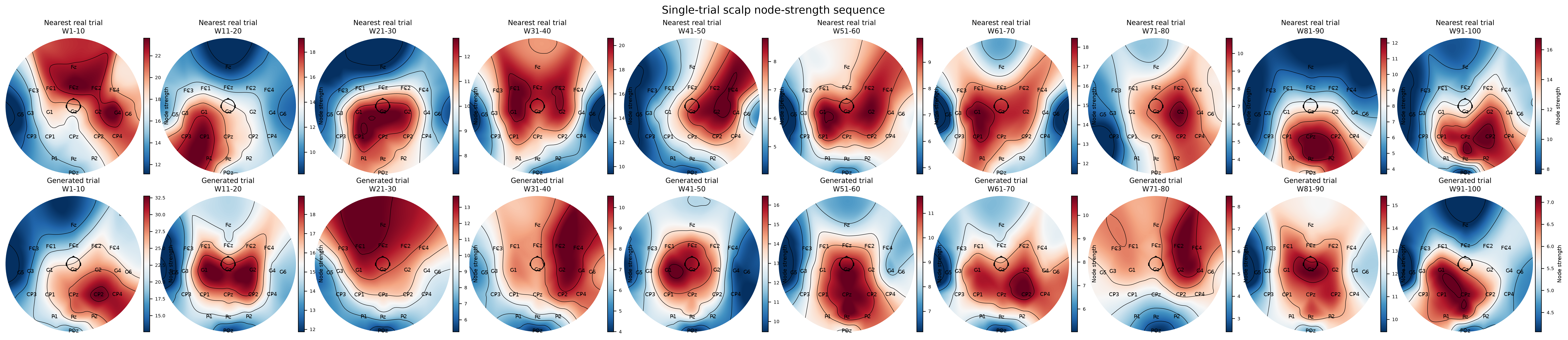}
    \caption{
    \textbf{Single-trial scalp node-strength sequence on BNCI2014\_001.}
    Scalp node-strength maps for the same generated trial and nearest
    held-out real trial across ten temporal intervals. Both sequences display
    spatially heterogeneous and time-varying patterns involving frontal,
    central, centro-parietal, and posterior electrodes. The generated sequence
    preserves broad sensorimotor-centered spatial dynamics without reproducing
    the real trial exactly.
    }
    \label{fig:single_trial_scalp_sequence}
\end{figure*}

\paragraph{Interpretation.}
These analyses show agreement at different levels: the time-averaged regional connectivity, its evolution over the trial, the reconfiguration of single-trial connectivity matrices, and the spatial distribution of node strength over time, which repeatedly involves the central and centro-parietal scalp regions relevant to motor imagery. We interpret this as evidence of regional physiological plausibility. 

\section{Additional Ablations}
\label{app:additional_ablations}

\subsection{Efficiency}

\begin{figure}[t]
    \centering
    \includegraphics[width=\linewidth]{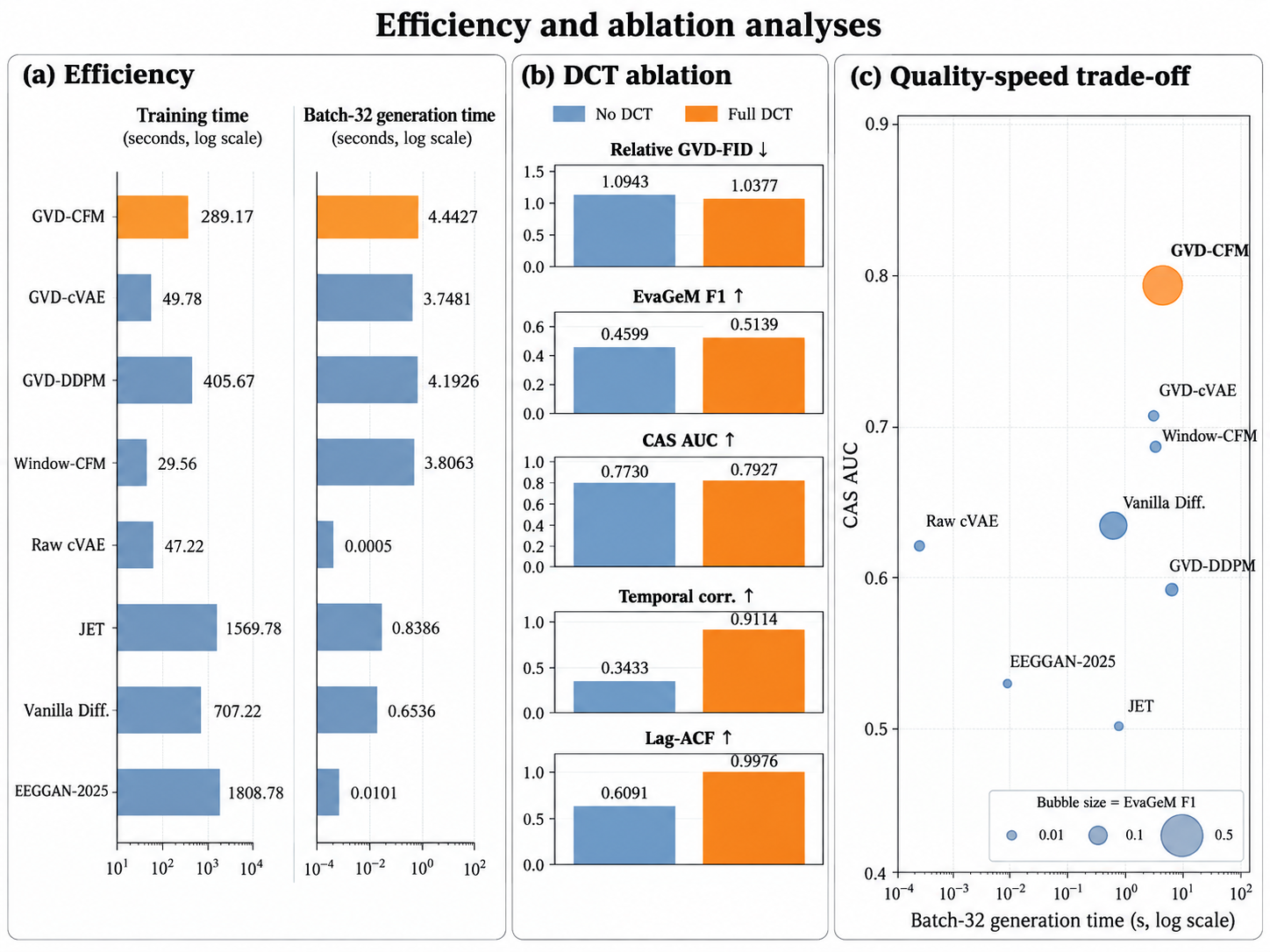}
    \caption{
    Efficiency and ablation analysis of GVD-CFM.
    \textbf{(a)} Average training time and batch-32 generation latency across five datasets and three generator seeds, shown on logarithmic axes. GVD-CFM is substantially faster to train than the large raw-signal generative baselines while retaining practical sampling cost.
    \textbf{(b)} DCT ablation with the spectral-only velocity network, comparing the full orthonormal DCT-II representation against direct temporal log-svec coordinates. The DCT improves relative GVD-FID, EvaGeM F1, CAS AUC, temporal-correlation agreement, and lag-autocorrelation agreement, with the largest gains observed in temporal structure.
    \textbf{(c)} Quality--speed trade-off across generators. The horizontal axis shows batch-32 generation time, the vertical axis shows CAS AUC, and marker size is proportional to EvaGeM F1. GVD-CFM combines strong downstream utility with substantially stronger trajectory-level distributional overlap than competing methods at practical generation cost.
    }
    \label{fig:efficiency_ablation}
\end{figure}

Figure~\ref{fig:efficiency_ablation} summarizes training and generation cost. For GVD-CFM, 1000 training epochs take $304$~s on average across the five main datasets and three seeds, and generating a batch of 32 trajectories with 50 RK4 steps takes $3.92$~s (Appendix~\ref{app:experimental_details}).
\subsection{Temporal-Resolution Audit of Stable-Support Information}
\label{app:support_resolution_audit}

\begin{figure}[t]
    \centering
    \includegraphics[width=\columnwidth,keepaspectratio]{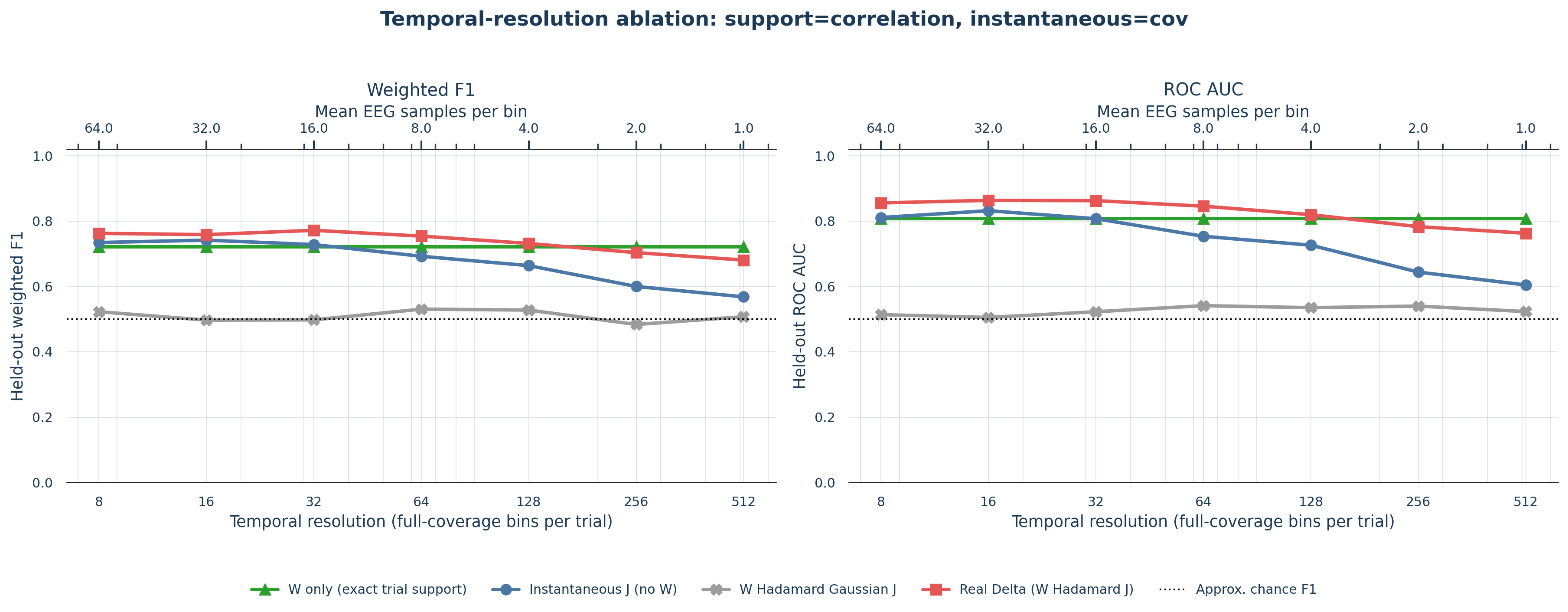}
    \caption{
        \textbf{Temporal-resolution ablation of graph-variate connectivity.}
        Weighted F1 (left) and ROC AUC (right) on the held-out evaluation
        split as the number of full-coverage bins increases from $B=8$ to
        sample resolution ($B=512$). The upper axis gives the mean number of
        EEG samples per bin. We compare the trial-level support $W$, the
        instantaneous covariance $J$ without support, the control
        $W\odot J^{\mathrm{Gaussian}}$, and the real graph-variate trajectory
        $\Delta=W\odot J$. The Gaussian control remains close to chance, so
        arbitrary temporal modulation of $W$ does not reproduce the
        discriminative information of the real trajectory. At fine
        resolutions $J$ degrades as its estimates become noisier, whereas
        modulation by the stable support is far more robust.
    }
    \label{fig:temporal_resolution_ablation}
\end{figure}

\begin{table}[t]
\centering
\caption{Temporal-resolution support audit (Figure~\ref{fig:temporal_resolution_ablation}). Held-out classification of real trajectories with the CAS classifier, dataset-balanced over five datasets and three seeds, as a function of the number of windows $B$. \emph{Stable support} is $W\odot J_t$, \emph{no support} is $J_t+\epsilon I$ with $\epsilon=10^{-6}$, \emph{support only} is the static $W$ and \emph{Gaussian dynamics} is $W\odot J_t^{\mathrm{Gaussian}}$. Best value in each column is \textbf{bold}.}
\label{tab:support_resolution_audit}
\small
\setlength{\tabcolsep}{4.2pt}
\begin{tabular}{lccccccc}
\toprule
Representation & $B=8$ & $16$ & $32$ & $64$ & $128$ & $256$ & $512$ \\
\midrule
\multicolumn{8}{l}{\textit{CAS AUC} $\uparrow$} \\
Stable support & \textbf{0.841} & \textbf{0.838} & \textbf{0.840} & \textbf{0.833} & \textbf{0.837} & \textbf{0.827} & 0.807 \\
No support & 0.816 & 0.797 & 0.765 & 0.745 & 0.756 & 0.738 & 0.690 \\
Support only & 0.822 & 0.822 & 0.822 & 0.822 & 0.822 & 0.822 & \textbf{0.822} \\
Gaussian dynamics & 0.512 & 0.518 & 0.527 & 0.545 & 0.574 & 0.558 & 0.528 \\
\midrule
\multicolumn{8}{l}{\textit{CAS weighted F1} $\uparrow$} \\
Stable support & \textbf{0.770} & \textbf{0.772} & \textbf{0.766} & \textbf{0.764} & \textbf{0.759} & \textbf{0.747} & 0.738 \\
No support & 0.747 & 0.735 & 0.698 & 0.684 & 0.693 & 0.673 & 0.640 \\
Support only & 0.747 & 0.747 & 0.747 & 0.747 & 0.747 & \textbf{0.747} & \textbf{0.747} \\
Gaussian dynamics & 0.505 & 0.512 & 0.513 & 0.526 & 0.550 & 0.535 & 0.506 \\
\bottomrule
\end{tabular}
\end{table}

To determine whether the discriminative information in GVD comes from the stable support $W$, from the instantaneous term $J_t$, or from their interaction, we evaluated four representations of real data over $B\in\{8,16,32,64,128,256,512\}$ windows with the CAS classifier and the cross-session protocol of the main benchmark: the stable-support trajectory $W\odot J_t$; the no-support trajectory $J_t+\epsilon I$; the static support $W$ alone; and $W\odot J_t^{\mathrm{Gaussian}}$, in which the node activity is replaced by independent Gaussian noise.

Table~\ref{tab:support_resolution_audit} shows that the stable-support representation keeps CAS AUC between $0.83$ and $0.84$ and weighted F1 between $0.76$ and $0.77$ from $B=8$ to $B=128$, and only falls to $0.807$ and $0.738$ at sample resolution. The no-support representation is informative at coarse resolution (AUC $0.816$ at $B=8$) but falls steadily as $B$ increases, to $0.745$ at $B=64$ and $0.690$ at $B=512$. This is the expected behavior of a covariance estimated from fewer and fewer samples. The static support alone reaches AUC $0.822$ and weighted F1 $0.747$, so $W$ itself carries substantial class information. It does not explain the full representation, however: $W\odot J_t$ exceeds $W$ alone at every resolution up to $B=256$, and replacing the real node activity with Gaussian noise drops performance to near chance (AUC $0.51$--$0.57$). The temporal modulation must therefore carry real structure from the EEG. At $B=512$ the full representation falls slightly below the support alone, so at sample resolution the instantaneous noise outweighs the added dynamic information; the loss is still much smaller than for the no-support representation. Essentially, the stable support acts as a variance-reducing structural prior that keeps informative dynamic modulation while avoiding the instability of an independent covariance estimate in every short window.

\subsection{Amplitude Bias and the Temporal Branch}
\label{app:time_context_visual}

GVD-CFM transports the trajectory in DCT coordinates, but the decoder exponentiates each window separately. Remark~\ref{rem:appendix_amplitude} shows that errors which are unbiased in log coordinates inflate the expected power of the decoded windows. Figure~\ref{fig:time_branch_pathology} shows this failure mode. The generated trajectory reproduces the broad shape of the real mean GVD edge trajectory, including the early transient and the subsequent recovery, but after the initial period it stays above the real trajectory for much of the trial. The global temporal shape is preserved while the time-local amplitude is biased.

\begin{figure*}[t]
    \centering
    \includegraphics[width=0.96\textwidth]{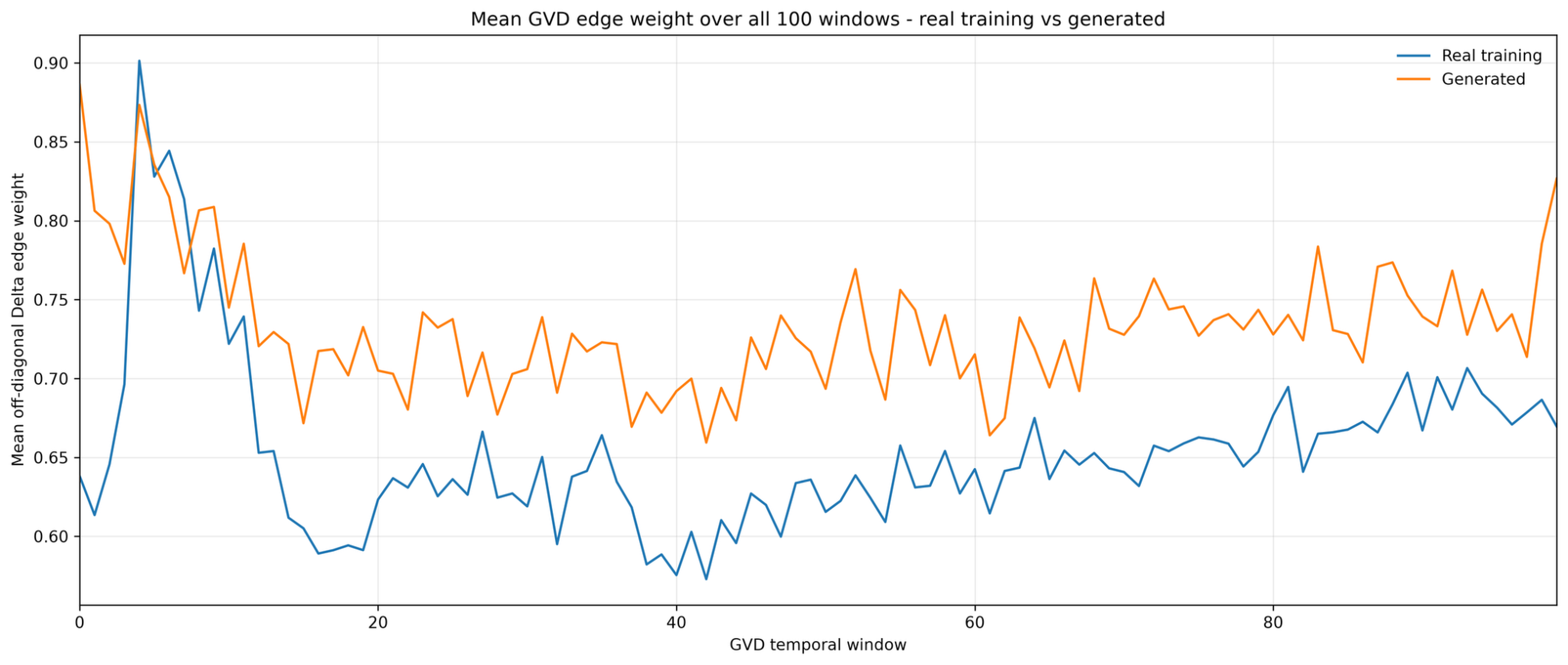}
    \caption{
    \textbf{Temporal amplitude bias in generated GVD dynamics.}
    Mean off-diagonal GVD edge weight over the $B=100$ windows for real
    training trials and generated trials. The generated sequence captures the
    broad temporal organization, including the early transient, but shows a
    persistent positive amplitude offset over much of the later trajectory.
    }
    \label{fig:time_branch_pathology}
\end{figure*}

The temporal branch of Section~\ref{sec:transformer} addresses this bias without introducing a second dynamical state. It reads the current DCT state $z_\tau$ in the window basis, $C_B^\top z_\tau$, processes it with two AdaLN blocks, and returns the result to DCT alignment before the gated fusion (Figure~\ref{fig:dualview}). The velocity is still predicted in DCT coordinates and only the DCT state is integrated, but the network now sees directly the quantity that the decoder exponentiates. Figure~\ref{fig:time_branch_alignment} shows a generated mean GVD edge trajectory from the full model that follows the real one through the initial transient, the subsequent decline and the middle and later portions of the trial. Quantitatively, the branch reduces dynamic energy on every dataset and moves it towards $1$ where the spectral-only network overshoots, while leaving the aggregate fidelity and utility metrics essentially unchanged (Table~\ref{tab:main_dct_ablation_5ds}).

\begin{figure*}[t]
    \centering
    \includegraphics[width=0.96\textwidth]{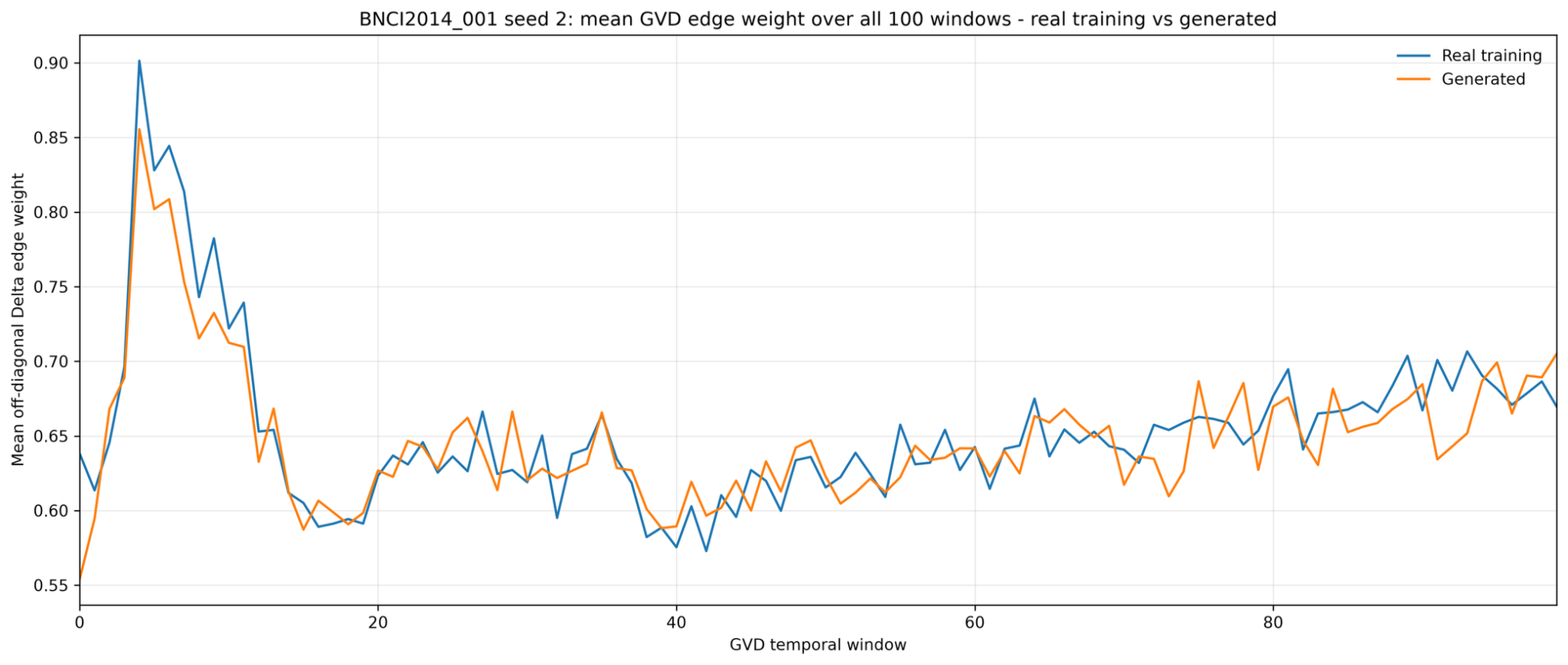}
    \caption{
    \textbf{Temporal alignment of the full GVD-CFM model.}
    Mean off-diagonal GVD edge weight over all $B=100$ windows for real training
    trials and generated trials on BNCI2014\_001. The generated trajectory
    follows the real sequence through the initial transient, the subsequent
    reduction, the intermediate fluctuations and the later recovery.
    }
    \label{fig:time_branch_alignment}
\end{figure*}
\begin{remark}[Why spectral errors inflate decoded amplitudes]
\label{rem:appendix_amplitude}
Let $\widehat{\bar Z}=\bar Z^\star+E$ be a generated DCT-coordinate trajectory with error
$E$. Window $b$ is decoded as
$\widehat{\Delta}_b=\exp\bigl(\operatorname{svec}^{-1}(\widehat{z}_b)\bigr)$
with $\widehat{z}_b=z_b^\star+\sigma\odot(C_B^\top E)_b$, so the log-coordinate
error of every window is a superposition of the errors of all modes. The map
$X\mapsto\operatorname{tr}\exp(X)$ is convex on symmetric matrices. If the
window error has zero mean, Jensen's inequality gives
\begin{equation}
    \mathbb{E}\bigl[\operatorname{tr}\widehat{\Delta}_b\bigr]
    \geq
    \operatorname{tr}\exp\bigl(\operatorname{svec}^{-1}(z_b^\star)\bigr),
\end{equation}
so unbiased errors in the log chart inflate the expected total power of each
decoded window. The spectral Transformer sees $E$ only mode by mode, whereas the
temporal branch sees $C_B^\top\widehat{\bar Z}$, the quantity that is
exponentiated window by window.
\end{remark}

\subsection{Nearest-Neighbor Memorization Diagnostic}

Table~\ref{tab:memorisation_final} compares the ratio $M=d_1/d_2$ of the first and second nearest same-class training neighbors of generated samples with the same ratio for held-out real samples. A generator that copies training data would have $M$ well below the held-out value. For GVD-CFM the two agree to within $0.001$, and no generator produces exact copies. GVD-CFM has a near-copy rate of $0.16$, against $0.98$ for GVD-cVAE, whose samples collapse towards the training data.

\begin{table}[t]
\centering
\caption{Nearest-neighbor memorization analysis. $M=d_1/d_2$ compares the first
and second nearest same-class training neighbors. The held-out-real
reference provides the target local-neighbor geometry; for GVD-CFM it is
computed in the same run as its samples.}
\label{tab:memorisation_final}
\resizebox{\linewidth}{!}{
\begin{tabular}{lrrrrrrr}
\toprule
Method &
$M_{\mathrm{gen}}$ &
$M_{\mathrm{held}}$ &
$|\Delta M| \downarrow$ &
$M^2_{\mathrm{gen}}$ &
$M^2_{\mathrm{held}}$ &
Near-copy $\downarrow$ &
Exact-copy $\downarrow$ \\
\midrule
\textbf{GVD-CFM}
& 0.99400 & 0.99501 & 0.00101
& 0.98803 & 0.99004 & 0.16155 & 0.00000 \\
GVD-cVAE
& 0.99316 & 0.99516 & 0.00200
& 0.98636 & 0.99034 & 0.98171 & 0.00000 \\
GVD-DDPM
& 0.99671 & 0.99516 & 0.00156
& 0.99344 & 0.99034 & 0.18854 & 0.00000 \\
Window-DIFFEO-CFM
& 0.99678 & 0.99516 & 0.00162
& 0.99357 & 0.99034 & 0.00000 & 0.00000 \\
cVAE
& 0.99555 & 0.99516 & \textbf{0.00040}
& 0.99113 & 0.99034 & 0.01807 & 0.00000 \\
JET
& 0.99298 & 0.99516 & 0.00218
& 0.98605 & 0.99034 & 0.00000 & 0.00000 \\
Vanilla-Diffusion
& 0.99397 & 0.99516 & 0.00119
& 0.98799 & 0.99034 & 0.00005 & 0.00000 \\
EEGGAN-2025
& 0.99169 & 0.99516 & 0.00347
& 0.98351 & 0.99034 & 0.00000 & 0.00000 \\
\bottomrule
\end{tabular}}
\end{table}

\subsection{Ablation: Removing Generated Eigenvalue Clipping}
\label{app:ablation_no_eig_clip}

To assess whether the performance of GVD-CFM depends on post-generation eigenvalue stabilization, we repeated the full generative benchmark with generated log-eigenvalue clipping disabled. All other components were unchanged: the full GVD-CFM with its temporal branch, the standard normal source, classwise Sinkhorn coupling, full orthonormal DCT coordinates, RK4 integration with 50 steps and the canonical stable-support GVD representation with $B=100$ windows. We evaluated GVD-CFM, GVD-DDPM and GVD-cVAE on all five datasets with three generator seeds per dataset.

Table~\ref{tab:no_eig_clip} reports the dataset-balanced results. Disabling clipping has a negligible effect on GVD-CFM. Its relative GVD Fr\'echet distance is $1.016$, EvaGeM F1 $0.609$, CAS AUC $0.792$ and CAS weighted F1 $0.724$, against $1.022$, $0.613$, $0.790$ and $0.725$ with clipping (Table~\ref{tab:main_generative_5ds}). Temporal agreement remains high (temporal correlation $0.942$, lag-ACF $0.998$), and the dynamic-energy and dynamic-fraction ratios remain close to $1$ ($0.975$ and $0.982$). The reported performance of GVD-CFM is therefore not driven by clipping, which acts as a numerical safeguard.

The ablation also shows the different failure modes of the GVD-space controls. Without clipping, GVD-DDPM becomes unstable on some datasets, with a relative GVD Fr\'echet distance of $1.796\pm1.533$ across datasets and a dynamic-energy ratio of $1.840\pm2.007$. GVD-cVAE again attains a low Fr\'echet distance with near-zero EvaGeM F1 ($0.003$) and strongly attenuated dynamics (dynamic-energy ratio $0.230$), so a favorable second-order distance does not imply faithful recovery of the dynamic distribution. Novelty metrics are unaffected: GVD-CFM has $M^2=0.988$ against a held-out value of $0.990$, no exact copies, and training coverage $0.227$, against $0.039$ for GVD-DDPM and $0.103$ for GVD-cVAE.

\begin{table*}[t]
\centering
\caption{Dataset-balanced results over five EEG datasets and three generator seeds per dataset with generated log-eigenvalue clipping disabled. GVD-CFM is the full model with the temporal branch. Values are mean $\pm$ standard deviation across datasets after averaging generator seeds within each dataset.}
\label{tab:no_eig_clip}
\resizebox{\textwidth}{!}{
\begin{tabular}{lccc}
\toprule
\textbf{Metric} &
\textbf{GVD-CFM} &
\textbf{GVD-DDPM} &
\textbf{GVD-cVAE} \\
\midrule
Relative GVD-FID + increments
& $1.0158 \pm 0.0692$
& $1.7955 \pm 1.5328$
& $0.8131 \pm 0.0789$ \\

Relative GVD-FID, positions only
& $1.0040 \pm 0.0453$
& $1.7354 \pm 1.2118$
& $0.9014 \pm 0.0708$ \\

EvaGeM F1 + increments
& $0.6086 \pm 0.2440$
& $0.0480 \pm 0.0372$
& $0.0028 \pm 0.0034$ \\

EvaGeM F1, positions only
& $0.6442 \pm 0.1778$
& $0.0046 \pm 0.0058$
& $0.0035 \pm 0.0048$ \\

CAS AUC
& $0.7920 \pm 0.0973$
& $0.6071 \pm 0.1317$
& $0.7334 \pm 0.0983$ \\

CAS weighted F1
& $0.7240 \pm 0.0963$
& $0.5759 \pm 0.1020$
& $0.6766 \pm 0.0840$ \\

Temporal correlation agreement
& $0.9419 \pm 0.0198$
& $0.7442 \pm 0.3233$
& $0.8255 \pm 0.0690$ \\

Lag-ACF agreement
& $0.9980 \pm 0.0016$
& $0.9462 \pm 0.0885$
& $0.9964 \pm 0.0011$ \\

Dynamic-energy ratio
& $0.9750 \pm 0.0774$
& $1.8396 \pm 2.0071$
& $0.2301 \pm 0.0596$ \\

Dynamic-fraction ratio
& $0.9818 \pm 0.0968$
& $1.6042 \pm 1.4582$
& $0.2587 \pm 0.0787$ \\

$M^2$
& $0.9881 \pm 0.0018$
& $0.9931 \pm 0.0027$
& $0.9855 \pm 0.0024$ \\

Held-out $M^2$
& $0.9900 \pm 0.0025$
& $0.9900 \pm 0.0025$
& $0.9900 \pm 0.0025$ \\

Exact-copy rate
& $0.0000$
& $0.0000$
& $0.0000$ \\

Training coverage
& $0.2273 \pm 0.0455$
& $0.0385 \pm 0.0189$
& $0.1026 \pm 0.0305$ \\
\bottomrule
\end{tabular}
}
\end{table*}
\subsection{Extended Tables}
\label{app:extended_tables}


\begin{table*}[t]
\centering
\caption{Full per-dataset generative performance corresponding to Table~\ref{tab:main_generative_5ds}. Entries are mean$\pm$standard deviation over three generator seeds. The real-data reference uses real training trajectories against held-out real test trajectories for EvaGeM, and the corresponding TRTR classifier reference for CAS. Real-data references are excluded from generator rankings. Best generator values within each dataset are \textbf{bold}; second-best generator values are \underline{\textit{underlined italics}}.}
\label{tab:appendix_generative_5ds}
\scriptsize
\renewcommand{\arraystretch}{0.92}
\setlength{\tabcolsep}{2.8pt}
\resizebox{\textwidth}{!}{%
\begin{tabular}{llcccccc}
\toprule
Method & Dataset & Rel. GVD-FID $\downarrow$ & Eva $\alpha$ $\uparrow$ & Eva $\beta$ $\uparrow$ & Eva F1 $\uparrow$ & CAS AUC $\uparrow$ & CAS F1 $\uparrow$ \\
\midrule

GVD-CFM & BNCI2014\_001 & 1.006$\pm$0.014 & \textbf{0.897$\pm$0.057} & \textbf{0.836$\pm$0.025} & \textbf{0.864$\pm$0.027} & \textbf{0.789$\pm$0.015} & \textbf{0.707$\pm$0.007} \\
 & BNCI2014\_002 & \underline{\textit{0.977$\pm$0.016}} & \textbf{0.837$\pm$0.046} & \textbf{0.666$\pm$0.055} & \textbf{0.740$\pm$0.022} & \textbf{0.780$\pm$0.015} & \textbf{0.709$\pm$0.005} \\
 & BNCI2015\_001 & \underline{\textit{0.978$\pm$0.005}} & \textbf{0.842$\pm$0.050} & \underline{\textit{0.719$\pm$0.090}} & \textbf{0.775$\pm$0.073} & \textbf{0.746$\pm$0.011} & \textbf{0.681$\pm$0.002} \\
 & Shin2017A & \underline{\textit{1.190$\pm$0.084}} & \underline{\textit{0.369$\pm$0.032}} & \textbf{0.696$\pm$0.049} & \textbf{0.482$\pm$0.039} & \textbf{0.672$\pm$0.029} & \textbf{0.622$\pm$0.025} \\
 & Zhou2016 & \underline{\textit{0.959$\pm$0.033}} & \textbf{0.440$\pm$0.143} & 0.133$\pm$0.081 & \underline{\textit{0.203$\pm$0.110}} & \textbf{0.965$\pm$0.003} & \textbf{0.905$\pm$0.014} \\
 & \textit{5-dataset avg.} & \underline{\textit{1.022}} & \textbf{0.677} & \textbf{0.610} & \textbf{0.613} & \textbf{0.790} & \textbf{0.725} \\

\midrule
GVD-cVAE & BNCI2014\_001 & \textbf{0.813$\pm$0.005} & 0.018$\pm$0.001 & 0.007$\pm$0.006 & 0.009$\pm$0.006 & \underline{\textit{0.742$\pm$0.033}} & \underline{\textit{0.683$\pm$0.020}} \\
 & BNCI2014\_002 & \textbf{0.734$\pm$0.008} & 0.027$\pm$0.022 & 0.001$\pm$0.002 & 0.002$\pm$0.003 & \underline{\textit{0.727$\pm$0.020}} & \underline{\textit{0.669$\pm$0.024}} \\
 & BNCI2015\_001 & \textbf{0.961$\pm$0.005} & 0.020$\pm$0.001 & 0.010$\pm$0.004 & 0.013$\pm$0.004 & \underline{\textit{0.701$\pm$0.013}} & \underline{\textit{0.644$\pm$0.007}} \\
 & Shin2017A & \textbf{0.796$\pm$0.008} & 0.008$\pm$0.013 & 0.002$\pm$0.002 & 0.000$\pm$0.000 & 0.575$\pm$0.026 & 0.552$\pm$0.016 \\
 & Zhou2016 & \textbf{0.770$\pm$0.002} & 0.012$\pm$0.004 & 0.002$\pm$0.001 & 0.003$\pm$0.002 & 0.888$\pm$0.059 & \underline{\textit{0.828$\pm$0.036}} \\
 & \textit{5-dataset avg.} & \textbf{0.815} & 0.017 & 0.004 & 0.005 & \underline{\textit{0.727}} & \underline{\textit{0.675}} \\

\midrule
GVD-DDPM & BNCI2014\_001 & \underline{\textit{0.967$\pm$0.013}} & 0.064$\pm$0.025 & 0.023$\pm$0.012 & 0.034$\pm$0.016 & 0.530$\pm$0.049 & 0.513$\pm$0.031 \\
 & BNCI2014\_002 & 1.019$\pm$0.001 & 0.110$\pm$0.011 & 0.051$\pm$0.006 & 0.070$\pm$0.008 & 0.583$\pm$0.029 & 0.541$\pm$0.005 \\
 & BNCI2015\_001 & 1.000$\pm$0.013 & 0.186$\pm$0.044 & 0.107$\pm$0.022 & 0.135$\pm$0.029 & 0.536$\pm$0.031 & 0.516$\pm$0.026 \\
 & Shin2017A & 4.778$\pm$0.014 & 0.006$\pm$0.008 & 0.000$\pm$0.000 & 0.001$\pm$0.001 & 0.529$\pm$0.004 & 0.507$\pm$0.013 \\
 & Zhou2016 & 1.004$\pm$0.007 & 0.077$\pm$0.018 & 0.014$\pm$0.004 & 0.024$\pm$0.006 & 0.815$\pm$0.032 & 0.723$\pm$0.039 \\
 & \textit{5-dataset avg.} & 1.753 & 0.089 & 0.039 & 0.053 & 0.599 & 0.560 \\

\midrule
Window-DIFFEO-CFM & BNCI2014\_001 & 1.412$\pm$0.003 & 0.004$\pm$0.003 & 0.071$\pm$0.003 & 0.008$\pm$0.005 & 0.638$\pm$0.001 & 0.616$\pm$0.003 \\
 & BNCI2014\_002 & 1.577$\pm$0.005 & 0.002$\pm$0.002 & 0.026$\pm$0.005 & 0.004$\pm$0.004 & 0.637$\pm$0.004 & 0.587$\pm$0.001 \\
 & BNCI2015\_001 & 1.311$\pm$0.009 & 0.012$\pm$0.004 & 0.050$\pm$0.001 & 0.019$\pm$0.005 & 0.662$\pm$0.010 & 0.594$\pm$0.003 \\
 & Shin2017A & 2.864$\pm$0.021 & 0.005$\pm$0.009 & \underline{\textit{0.021$\pm$0.004}} & \underline{\textit{0.006$\pm$0.010}} & \underline{\textit{0.589$\pm$0.008}} & \underline{\textit{0.563$\pm$0.010}} \\
 & Zhou2016 & 1.323$\pm$0.007 & 0.003$\pm$0.001 & \underline{\textit{0.242$\pm$0.009}} & 0.006$\pm$0.002 & 0.889$\pm$0.003 & 0.810$\pm$0.006 \\
 & \textit{5-dataset avg.} & 1.697 & 0.005 & 0.082 & 0.009 & 0.683 & 0.634 \\

\midrule
Real data reference & BNCI2014\_001 & -- & 0.911$\pm$0.001 & 0.894$\pm$0.000 & 0.902$\pm$0.001 & 0.856$\pm$0.003 & 0.763$\pm$0.008 \\
 & BNCI2014\_002 & -- & 0.673$\pm$0.001 & 0.328$\pm$0.001 & 0.441$\pm$0.001 & 0.801$\pm$0.005 & 0.733$\pm$0.007 \\
 & BNCI2015\_001 & -- & 0.926$\pm$0.001 & 0.882$\pm$0.003 & 0.904$\pm$0.001 & 0.780$\pm$0.002 & 0.703$\pm$0.004 \\
 & Shin2017A & -- & 0.639$\pm$0.003 & 0.883$\pm$0.002 & 0.741$\pm$0.002 & 0.745$\pm$0.001 & 0.676$\pm$0.009 \\
 & Zhou2016 & -- & 0.880$\pm$0.004 & 0.337$\pm$0.004 & 0.488$\pm$0.004 & 0.972$\pm$0.000 & 0.907$\pm$0.002 \\
 & \textit{5-dataset avg.} & -- & 0.806 & 0.665 & 0.695 & 0.831 & 0.756 \\

\bottomrule
\end{tabular}}
\end{table*}

\begin{table*}[t]
\centering
\caption{Full per-dataset generative performance (continued).}
\scriptsize
\renewcommand{\arraystretch}{0.92}
\setlength{\tabcolsep}{2.8pt}
\resizebox{\textwidth}{!}{%
\begin{tabular}{llcccccc}
\toprule
Method & Dataset & Rel. GVD-FID $\downarrow$ & Eva $\alpha$ $\uparrow$ & Eva $\beta$ $\uparrow$ & Eva F1 $\uparrow$ & CAS AUC $\uparrow$ & CAS F1 $\uparrow$ \\
\midrule

cVAE & BNCI2014\_001 & 1.374$\pm$0.020 & \underline{\textit{0.621$\pm$0.090}} & 0.002$\pm$0.002 & 0.004$\pm$0.003 & 0.588$\pm$0.013 & 0.555$\pm$0.015 \\
 & BNCI2014\_002 & 1.154$\pm$0.008 & \underline{\textit{0.486$\pm$0.045}} & 0.006$\pm$0.004 & 0.012$\pm$0.008 & 0.550$\pm$0.022 & 0.491$\pm$0.066 \\
 & BNCI2015\_001 & 1.885$\pm$0.052 & 0.023$\pm$0.005 & 0.005$\pm$0.001 & 0.007$\pm$0.001 & 0.584$\pm$0.002 & 0.546$\pm$0.011 \\
 & Shin2017A & 1.617$\pm$0.055 & \textbf{0.668$\pm$0.140} & 0.002$\pm$0.002 & 0.003$\pm$0.003 & 0.514$\pm$0.013 & 0.470$\pm$0.043 \\
 & Zhou2016 & 1.049$\pm$0.003 & 0.240$\pm$0.043 & 0.018$\pm$0.006 & 0.034$\pm$0.011 & \underline{\textit{0.920$\pm$0.007}} & 0.799$\pm$0.026 \\
 & \textit{5-dataset avg.} & 1.416 & \underline{\textit{0.407}} & 0.006 & 0.012 & 0.631 & 0.572 \\

\midrule
JET & BNCI2014\_001 & 4.794$\pm$0.294 & 0.004$\pm$0.002 & 0.002$\pm$0.001 & 0.002$\pm$0.001 & 0.467$\pm$0.011 & 0.333$\pm$0.000 \\
 & BNCI2014\_002 & 3.402$\pm$0.053 & 0.000$\pm$0.000 & 0.000$\pm$0.000 & 0.000$\pm$0.000 & 0.494$\pm$0.033 & 0.333$\pm$0.000 \\
 & BNCI2015\_001 & 3.864$\pm$0.035 & 0.014$\pm$0.007 & 0.007$\pm$0.009 & 0.008$\pm$0.009 & 0.501$\pm$0.016 & 0.342$\pm$0.015 \\
 & Shin2017A & 3.753$\pm$0.097 & 0.005$\pm$0.005 & 0.000$\pm$0.000 & 0.000$\pm$0.000 & 0.498$\pm$0.009 & 0.333$\pm$0.000 \\
 & Zhou2016 & 3.726$\pm$1.331 & 0.003$\pm$0.003 & 0.000$\pm$0.000 & 0.000$\pm$0.000 & 0.516$\pm$0.151 & 0.405$\pm$0.067 \\
 & \textit{5-dataset avg.} & 3.908 & 0.005 & 0.002 & 0.002 & 0.495 & 0.349 \\

\midrule
Vanilla-Diffusion & BNCI2014\_001 & 1.566$\pm$0.370 & 0.314$\pm$0.194 & \underline{\textit{0.448$\pm$0.347}} & \underline{\textit{0.366$\pm$0.250}} & 0.646$\pm$0.050 & 0.514$\pm$0.094 \\
 & BNCI2014\_002 & 1.444$\pm$0.084 & 0.242$\pm$0.031 & \underline{\textit{0.370$\pm$0.273}} & \underline{\textit{0.264$\pm$0.072}} & 0.598$\pm$0.013 & 0.576$\pm$0.016 \\
 & BNCI2015\_001 & 1.667$\pm$0.184 & \underline{\textit{0.240$\pm$0.053}} & \textbf{0.746$\pm$0.096} & \underline{\textit{0.363$\pm$0.072}} & 0.578$\pm$0.012 & 0.526$\pm$0.050 \\
 & Shin2017A & 3.230$\pm$0.109 & 0.007$\pm$0.008 & 0.004$\pm$0.003 & 0.004$\pm$0.005 & 0.509$\pm$0.016 & 0.419$\pm$0.037 \\
 & Zhou2016 & 1.427$\pm$0.205 & \underline{\textit{0.284$\pm$0.059}} & \textbf{0.264$\pm$0.188} & \textbf{0.240$\pm$0.074} & 0.870$\pm$0.027 & 0.712$\pm$0.089 \\
 & \textit{5-dataset avg.} & 1.867 & 0.217 & \underline{\textit{0.366}} & \underline{\textit{0.247}} & 0.640 & 0.549 \\

\midrule
EEGGAN-2025 & BNCI2014\_001 & 4.943$\pm$0.023 & 0.008$\pm$0.009 & 0.001$\pm$0.001 & 0.001$\pm$0.001 & 0.516$\pm$0.044 & 0.336$\pm$0.004 \\
 & BNCI2014\_002 & 3.424$\pm$0.033 & 0.000$\pm$0.000 & 0.000$\pm$0.000 & 0.000$\pm$0.000 & 0.513$\pm$0.043 & 0.465$\pm$0.019 \\
 & BNCI2015\_001 & 4.005$\pm$0.027 & 0.014$\pm$0.007 & 0.007$\pm$0.006 & 0.009$\pm$0.006 & 0.518$\pm$0.069 & 0.391$\pm$0.042 \\
 & Shin2017A & 3.856$\pm$0.032 & 0.005$\pm$0.005 & 0.000$\pm$0.000 & 0.000$\pm$0.000 & 0.501$\pm$0.018 & 0.405$\pm$0.063 \\
 & Zhou2016 & 2.249$\pm$0.054 & 0.002$\pm$0.001 & 0.000$\pm$0.000 & 0.000$\pm$0.000 & 0.573$\pm$0.123 & 0.408$\pm$0.028 \\
 & \textit{5-dataset avg.} & 3.696 & 0.006 & 0.002 & 0.002 & 0.524 & 0.401 \\

\bottomrule
\end{tabular}}
\end{table*}
\begin{table*}[t]
\centering
\caption{Full per-dataset complementary diagnostics corresponding to Table~\ref{tab:main_quality_5ds}: temporal and dynamic fidelity. Entries are mean$\pm$standard deviation over three generator seeds. For ratio metrics, ranking is by proximity to 1.}
\label{tab:appendix_quality_5ds}
\scriptsize
\renewcommand{\arraystretch}{0.92}
\setlength{\tabcolsep}{2.8pt}
\resizebox{\textwidth}{!}{%
\begin{tabular}{llccccc}
\toprule
Method & Dataset & Temp. corr. $\uparrow$ & Lag-ACF $\uparrow$ & Energy $\to1$ & Dyn. frac. $\to1$ & Adjacent $\to1$ \\
\midrule
GVD-CFM & BNCI2014\_001 & \underline{\textit{0.920$\pm$0.006}} & \textbf{0.999$\pm$0.000} & \textbf{0.995$\pm$0.023} & \textbf{1.024$\pm$0.052} & 1.070$\pm$0.013 \\
 & BNCI2014\_002 & \textbf{0.838$\pm$0.003} & \textbf{0.999$\pm$0.000} & 0.948$\pm$0.018 & \underline{\textit{0.938$\pm$0.036}} & \underline{\textit{0.977$\pm$0.017}} \\
 & BNCI2015\_001 & \textbf{0.958$\pm$0.001} & \textbf{0.999$\pm$0.000} & 0.938$\pm$0.016 & \textbf{0.976$\pm$0.007} & \textbf{0.976$\pm$0.015} \\
 & Shin2017A & \textbf{0.958$\pm$0.002} & 0.994$\pm$0.001 & \underline{\textit{1.068$\pm$0.006}} & \textbf{1.161$\pm$0.011} & 1.317$\pm$0.005 \\
 & Zhou2016 & \textbf{0.877$\pm$0.017} & \underline{\textit{0.992$\pm$0.005}} & \underline{\textit{0.815$\pm$0.039}} & 0.835$\pm$0.011 & 0.871$\pm$0.027 \\
 & \textit{5-dataset avg.} & \textbf{0.910} & \underline{\textit{0.997}} & \textbf{0.953} & \textbf{0.987} & 1.042 \\
\midrule
GVD-cVAE & BNCI2014\_001 & 0.816$\pm$0.024 & \underline{\textit{0.998$\pm$0.000}} & 0.219$\pm$0.005 & 0.239$\pm$0.007 & 0.193$\pm$0.002 \\
 & BNCI2014\_002 & 0.704$\pm$0.038 & 0.997$\pm$0.002 & 0.371$\pm$0.012 & 0.408$\pm$0.009 & 0.327$\pm$0.008 \\
 & BNCI2015\_001 & 0.852$\pm$0.012 & 0.995$\pm$0.000 & 0.382$\pm$0.026 & 0.446$\pm$0.034 & 0.306$\pm$0.023 \\
 & Shin2017A & \underline{\textit{0.802$\pm$0.039}} & \underline{\textit{0.995$\pm$0.002}} & 0.134$\pm$0.020 & 0.159$\pm$0.026 & 0.117$\pm$0.016 \\
 & Zhou2016 & 0.699$\pm$0.034 & 0.992$\pm$0.003 & 0.193$\pm$0.011 & 0.192$\pm$0.014 & 0.182$\pm$0.009 \\
 & \textit{5-dataset avg.} & 0.774 & 0.995 & 0.260 & 0.289 & 0.225 \\
\midrule
GVD-DDPM & BNCI2014\_001 & 0.038$\pm$0.004 & 0.183$\pm$0.077 & 0.694$\pm$0.020 & 0.714$\pm$0.014 & 0.863$\pm$0.025 \\
 & BNCI2014\_002 & 0.029$\pm$0.004 & 0.215$\pm$0.284 & 0.759$\pm$0.006 & 0.871$\pm$0.005 & 0.930$\pm$0.009 \\
 & BNCI2015\_001 & 0.138$\pm$0.009 & 0.791$\pm$0.048 & 0.788$\pm$0.013 & \underline{\textit{0.877$\pm$0.015}} & \underline{\textit{0.972$\pm$0.014}} \\
 & Shin2017A & -0.011$\pm$0.002 & -0.274$\pm$0.150 & 5.466$\pm$0.020 & 4.467$\pm$0.025 & 8.566$\pm$0.028 \\
 & Zhou2016 & 0.103$\pm$0.023 & 0.718$\pm$0.066 & 0.732$\pm$0.012 & 0.788$\pm$0.026 & \underline{\textit{0.916$\pm$0.016}} \\
 & \textit{5-dataset avg.} & 0.059 & 0.327 & 1.688 & 1.543 & 2.450 \\
\midrule
Window-DIFFEO-CFM & BNCI2014\_001 & 0.011$\pm$0.006 & 0.201$\pm$0.167 & 1.487$\pm$0.005 & 1.464$\pm$0.009 & 1.846$\pm$0.010 \\
 & BNCI2014\_002 & 0.003$\pm$0.009 & 0.157$\pm$0.268 & 1.746$\pm$0.009 & 1.794$\pm$0.010 & 2.142$\pm$0.012 \\
 & BNCI2015\_001 & 0.018$\pm$0.014 & 0.053$\pm$0.244 & 1.256$\pm$0.007 & 1.299$\pm$0.014 & 1.553$\pm$0.010 \\
 & Shin2017A & -0.003$\pm$0.017 & -0.214$\pm$0.070 & 3.174$\pm$0.035 & 2.980$\pm$0.018 & 4.981$\pm$0.055 \\
 & Zhou2016 & 0.029$\pm$0.001 & 0.080$\pm$0.210 & 1.308$\pm$0.019 & 1.238$\pm$0.011 & 1.641$\pm$0.024 \\
 & \textit{5-dataset avg.} & 0.012 & 0.055 & 1.794 & 1.755 & 2.433 \\
\bottomrule
\end{tabular}}
\end{table*}

\begin{table*}[t]
\centering
\caption{Full per-dataset temporal and dynamic diagnostics (continued).}
\scriptsize
\renewcommand{\arraystretch}{0.92}
\setlength{\tabcolsep}{2.8pt}
\resizebox{\textwidth}{!}{%
\begin{tabular}{llccccc}
\toprule
Method & Dataset & Temp. corr. $\uparrow$ & Lag-ACF $\uparrow$ & Energy $\to1$ & Dyn. frac. $\to1$ & Adjacent $\to1$ \\
\midrule
cVAE & BNCI2014\_001 & 0.399$\pm$0.031 & 0.948$\pm$0.017 & 0.986$\pm$0.014 & 1.381$\pm$0.031 & \underline{\textit{1.039$\pm$0.023}} \\
 & BNCI2014\_002 & 0.512$\pm$0.060 & 0.976$\pm$0.006 & \textbf{0.956$\pm$0.011} & 0.850$\pm$0.014 & \textbf{1.013$\pm$0.008} \\
 & BNCI2015\_001 & 0.798$\pm$0.019 & 0.990$\pm$0.003 & \textbf{1.023$\pm$0.014} & 0.520$\pm$0.001 & 1.038$\pm$0.013 \\
 & Shin2017A & 0.216$\pm$0.033 & 0.913$\pm$0.026 & 0.651$\pm$0.018 & \underline{\textit{1.290$\pm$0.060}} & \textbf{0.886$\pm$0.023} \\
 & Zhou2016 & 0.710$\pm$0.018 & 0.958$\pm$0.004 & \textbf{0.876$\pm$0.005} & \underline{\textit{0.862$\pm$0.006}} & \textbf{0.955$\pm$0.006} \\
 & \textit{5-dataset avg.} & 0.527 & 0.957 & 0.898 & \underline{\textit{0.981}} & \textbf{0.986} \\
\midrule
JET & BNCI2014\_001 & 0.357$\pm$0.037 & 0.937$\pm$0.012 & 1.897$\pm$0.235 & 9.197$\pm$1.470 & 1.268$\pm$0.096 \\
 & BNCI2014\_002 & 0.399$\pm$0.051 & 0.983$\pm$0.003 & 1.474$\pm$0.087 & 9.324$\pm$0.386 & 1.177$\pm$0.025 \\
 & BNCI2015\_001 & 0.490$\pm$0.040 & 0.985$\pm$0.006 & 1.542$\pm$0.109 & 6.338$\pm$0.175 & 1.146$\pm$0.019 \\
 & Shin2017A & 0.375$\pm$0.047 & 0.785$\pm$0.037 & 2.080$\pm$0.138 & 8.183$\pm$0.626 & 1.456$\pm$0.056 \\
 & Zhou2016 & 0.306$\pm$0.035 & 0.977$\pm$0.017 & 2.709$\pm$0.502 & \textbf{1.135$\pm$0.361} & 1.912$\pm$0.242 \\
 & \textit{5-dataset avg.} & 0.385 & 0.933 & 1.940 & 6.835 & 1.392 \\
\midrule
Vanilla-Diffusion & BNCI2014\_001 & \textbf{0.930$\pm$0.013} & 0.998$\pm$0.002 & \underline{\textit{0.992$\pm$0.024}} & \underline{\textit{1.185$\pm$0.403}} & 1.041$\pm$0.013 \\
 & BNCI2014\_002 & \underline{\textit{0.832$\pm$0.006}} & \underline{\textit{0.998$\pm$0.000}} & \underline{\textit{0.954$\pm$0.075}} & \textbf{0.943$\pm$0.256} & 0.973$\pm$0.051 \\
 & BNCI2015\_001 & \underline{\textit{0.955$\pm$0.003}} & \underline{\textit{0.999$\pm$0.000}} & \underline{\textit{0.949$\pm$0.011}} & 0.688$\pm$0.104 & 0.948$\pm$0.014 \\
 & Shin2017A & 0.323$\pm$0.063 & \textbf{0.996$\pm$0.002} & \textbf{1.023$\pm$0.096} & 3.645$\pm$0.568 & \underline{\textit{1.282$\pm$0.108}} \\
 & Zhou2016 & \underline{\textit{0.875$\pm$0.027}} & \textbf{0.998$\pm$0.000} & 0.788$\pm$0.139 & 2.337$\pm$1.953 & 0.831$\pm$0.097 \\
 & \textit{5-dataset avg.} & \underline{\textit{0.783}} & \textbf{0.998} & \underline{\textit{0.941}} & 1.759 & \underline{\textit{1.015}} \\
\midrule
EEGGAN-2025 & BNCI2014\_001 & -0.079$\pm$0.004 & -0.228$\pm$0.011 & 0.830$\pm$0.044 & 3.490$\pm$0.378 & \textbf{1.038$\pm$0.055} \\
 & BNCI2014\_002 & 0.078$\pm$0.008 & 0.820$\pm$0.089 & 0.622$\pm$0.018 & 8.267$\pm$0.550 & 0.761$\pm$0.024 \\
 & BNCI2015\_001 & 0.098$\pm$0.023 & 0.759$\pm$0.142 & 0.574$\pm$0.020 & 6.277$\pm$0.214 & 0.706$\pm$0.024 \\
 & Shin2017A & -0.020$\pm$0.011 & -0.206$\pm$0.141 & 0.408$\pm$0.006 & 2.072$\pm$0.231 & 0.640$\pm$0.010 \\
 & Zhou2016 & 0.035$\pm$0.045 & 0.385$\pm$0.650 & 0.610$\pm$0.027 & 4.216$\pm$0.671 & 0.765$\pm$0.031 \\
 & \textit{5-dataset avg.} & 0.022 & 0.306 & 0.609 & 4.864 & 0.782 \\
\bottomrule
\end{tabular}}
\end{table*}
\begin{table*}[t]
\centering
\caption{Full per-dataset novelty, diversity, and coverage diagnostics corresponding to Table~\ref{tab:main_quality_5ds}. Entries are mean$\pm$standard deviation over three generator seeds. Diversity has ideal value 1.}
\scriptsize
\renewcommand{\arraystretch}{0.92}
\setlength{\tabcolsep}{2.8pt}

\begin{minipage}{0.72\textwidth}
\centering
\begin{tabular}{llccc}
\toprule
Method & Dataset & Precision $\uparrow$ & Diversity $\to1$ & Coverage $\uparrow$ \\
\midrule
GVD-CFM & BNCI2014\_001 & 0.967$\pm$0.012 & \underline{\textit{0.945$\pm$0.002}} & \textbf{0.175$\pm$0.008} \\
 & BNCI2014\_002 & 0.899$\pm$0.026 & \underline{\textit{0.952$\pm$0.012}} & \textbf{0.211$\pm$0.003} \\
 & BNCI2015\_001 & 0.859$\pm$0.037 & 0.956$\pm$0.007 & \textbf{0.203$\pm$0.006} \\
 & Shin2017A & 0.032$\pm$0.036 & \textbf{1.168$\pm$0.005} & \textbf{0.222$\pm$0.017} \\
 & Zhou2016 & \textbf{1.000$\pm$0.000} & 0.914$\pm$0.022 & \textbf{0.317$\pm$0.012} \\
 & \textit{5-dataset avg.} & 0.751 & \textbf{0.987} & \textbf{0.226} \\
\midrule
GVD-cVAE & BNCI2014\_001 & \textbf{1.000$\pm$0.000} & 0.363$\pm$0.013 & \underline{\textit{0.131$\pm$0.009}} \\
 & BNCI2014\_002 & \textbf{1.000$\pm$0.000} & 0.505$\pm$0.003 & \underline{\textit{0.138$\pm$0.004}} \\
 & BNCI2015\_001 & \textbf{1.000$\pm$0.000} & 0.497$\pm$0.018 & \underline{\textit{0.129$\pm$0.011}} \\
 & Shin2017A & \textbf{0.999$\pm$0.002} & 0.276$\pm$0.027 & \underline{\textit{0.058$\pm$0.014}} \\
 & Zhou2016 & \underline{\textit{1.000$\pm$0.000}} & 0.338$\pm$0.021 & \underline{\textit{0.185$\pm$0.006}} \\
 & \textit{5-dataset avg.} & \textbf{1.000} & 0.396 & \underline{\textit{0.128}} \\
\midrule
GVD-DDPM & BNCI2014\_001 & \underline{\textit{1.000$\pm$0.000}} & 0.827$\pm$0.014 & 0.030$\pm$0.006 \\
 & BNCI2014\_002 & \underline{\textit{1.000$\pm$0.000}} & 0.871$\pm$0.000 & 0.025$\pm$0.002 \\
 & BNCI2015\_001 & \underline{\textit{1.000$\pm$0.000}} & 0.884$\pm$0.011 & 0.034$\pm$0.001 \\
 & Shin2017A & 0.000$\pm$0.000 & 3.391$\pm$0.000 & 0.039$\pm$0.004 \\
 & Zhou2016 & 1.000$\pm$0.001 & 0.863$\pm$0.002 & 0.063$\pm$0.006 \\
 & \textit{5-dataset avg.} & \underline{\textit{0.800}} & 1.367 & 0.038 \\
\midrule
Window-DIFFEO-CFM & BNCI2014\_001 & 0.000$\pm$0.000 & 1.156$\pm$0.003 & 0.030$\pm$0.003 \\
 & BNCI2014\_002 & 0.000$\pm$0.000 & 1.133$\pm$0.006 & 0.021$\pm$0.001 \\
 & BNCI2015\_001 & 0.000$\pm$0.000 & 1.077$\pm$0.004 & 0.031$\pm$0.001 \\
 & Shin2017A & 0.000$\pm$0.000 & 1.488$\pm$0.007 & 0.014$\pm$0.002 \\
 & Zhou2016 & 0.000$\pm$0.000 & 1.092$\pm$0.001 & 0.057$\pm$0.002 \\
 & \textit{5-dataset avg.} & 0.000 & 1.189 & 0.031 \\
\bottomrule
\end{tabular}
\end{minipage}
\end{table*}

\begin{table*}[t]
\centering
\caption{Full per-dataset novelty, diversity, and coverage diagnostics (continued).}
\scriptsize
\renewcommand{\arraystretch}{0.92}
\setlength{\tabcolsep}{2.8pt}

\begin{minipage}{0.72\textwidth}
\centering
\begin{tabular}{llccc}
\toprule
Method & Dataset & Precision $\uparrow$ & Diversity $\to1$ & Coverage $\uparrow$ \\
\midrule
cVAE & BNCI2014\_001 & 0.799$\pm$0.090 & 0.709$\pm$0.002 & 0.027$\pm$0.001 \\
 & BNCI2014\_002 & 0.615$\pm$0.016 & 0.890$\pm$0.015 & 0.029$\pm$0.002 \\
 & BNCI2015\_001 & 0.101$\pm$0.031 & \underline{\textit{1.037$\pm$0.012}} & 0.022$\pm$0.003 \\
 & Shin2017A & \underline{\textit{0.875$\pm$0.122}} & 0.503$\pm$0.034 & 0.008$\pm$0.000 \\
 & Zhou2016 & 0.572$\pm$0.024 & \textbf{0.946$\pm$0.009} & 0.085$\pm$0.005 \\
 & \textit{5-dataset avg.} & 0.593 & 0.817 & 0.034 \\
\midrule
JET & BNCI2014\_001 & 0.000$\pm$0.000 & 1.340$\pm$0.169 & 0.011$\pm$0.005 \\
 & BNCI2014\_002 & 0.000$\pm$0.000 & 0.870$\pm$0.063 & 0.004$\pm$0.001 \\
 & BNCI2015\_001 & 0.000$\pm$0.000 & 0.833$\pm$0.032 & 0.004$\pm$0.001 \\
 & Shin2017A & 0.000$\pm$0.000 & 1.528$\pm$0.112 & 0.006$\pm$0.000 \\
 & Zhou2016 & 0.000$\pm$0.000 & 2.936$\pm$0.331 & 0.023$\pm$0.020 \\
 & \textit{5-dataset avg.} & 0.000 & 1.501 & 0.009 \\
\midrule
Vanilla-Diffusion & BNCI2014\_001 & 0.022$\pm$0.023 & \textbf{1.046$\pm$0.006} & 0.089$\pm$0.034 \\
 & BNCI2014\_002 & 0.098$\pm$0.008 & \textbf{1.032$\pm$0.038} & 0.111$\pm$0.016 \\
 & BNCI2015\_001 & 0.049$\pm$0.007 & \textbf{1.027$\pm$0.006} & 0.094$\pm$0.005 \\
 & Shin2017A & 0.000$\pm$0.000 & 1.303$\pm$0.134 & 0.011$\pm$0.004 \\
 & Zhou2016 & 0.110$\pm$0.053 & \underline{\textit{0.943$\pm$0.203}} & 0.119$\pm$0.034 \\
 & \textit{5-dataset avg.} & 0.056 & 1.070 & 0.085 \\
\midrule
EEGGAN-2025 & BNCI2014\_001 & 0.000$\pm$0.000 & 1.392$\pm$0.060 & 0.004$\pm$0.000 \\
 & BNCI2014\_002 & 0.000$\pm$0.000 & 0.778$\pm$0.046 & 0.002$\pm$0.001 \\
 & BNCI2015\_001 & 0.000$\pm$0.000 & 0.707$\pm$0.033 & 0.001$\pm$0.000 \\
 & Shin2017A & 0.000$\pm$0.000 & \underline{\textit{1.182$\pm$0.046}} & 0.003$\pm$0.001 \\
 & Zhou2016 & 0.000$\pm$0.000 & 1.106$\pm$0.016 & 0.004$\pm$0.001 \\
 & \textit{5-dataset avg.} & 0.000 & \underline{\textit{1.033}} & 0.003 \\
\bottomrule
\end{tabular}
\end{minipage}
\end{table*}
\begin{table*}[t]
\centering
\caption{Full per-dataset DCT and temporal-branch ablation corresponding to Table~\ref{tab:main_dct_ablation_5ds}: generative fidelity and utility. No DCT and DCT, spectral only are trained without the temporal branch; the last row of each block is the full GVD-CFM. Entries are mean$\pm$standard deviation over three generator seeds. Best values are \textbf{bold}; second-best values are \underline{\textit{underlined italics}}.}
\label{tab:appendix_dct_ablation_5ds}
\scriptsize
\renewcommand{\arraystretch}{0.95}
\setlength{\tabcolsep}{3.0pt}
\resizebox{0.92\textwidth}{!}{%
\begin{tabular}{llcccc}
\toprule
Dataset & Variant & Rel. GVD-FID $\downarrow$ & Eva F1 $\uparrow$ & CAS AUC $\uparrow$ & CAS F1 $\uparrow$ \\
\midrule
BNCI2014\_001 & No DCT & 1.030$\pm$0.017 & 0.641$\pm$0.090 & 0.773$\pm$0.027 & 0.691$\pm$0.023 \\
 & DCT, spectral only & \underline{\textit{1.016$\pm$0.015}} & \underline{\textit{0.825$\pm$0.021}} & \underline{\textit{0.784$\pm$0.006}} & \underline{\textit{0.695$\pm$0.005}} \\
 & GVD-CFM (DCT + temporal branch) & \textbf{1.006$\pm$0.014} & \textbf{0.864$\pm$0.027} & \textbf{0.789$\pm$0.015} & \textbf{0.707$\pm$0.007} \\
\midrule
BNCI2014\_002 & No DCT & 1.036$\pm$0.009 & 0.628$\pm$0.090 & 0.757$\pm$0.009 & 0.685$\pm$0.009 \\
 & DCT, spectral only & \underline{\textit{0.988$\pm$0.012}} & \underline{\textit{0.723$\pm$0.042}} & \textbf{0.789$\pm$0.006} & \textbf{0.716$\pm$0.006} \\
 & GVD-CFM (DCT + temporal branch) & \textbf{0.977$\pm$0.016} & \textbf{0.740$\pm$0.022} & \underline{\textit{0.780$\pm$0.015}} & \underline{\textit{0.709$\pm$0.005}} \\
\midrule
BNCI2015\_001 & No DCT & 1.009$\pm$0.006 & \textbf{0.889$\pm$0.026} & 0.704$\pm$0.002 & 0.640$\pm$0.011 \\
 & DCT, spectral only & \underline{\textit{0.992$\pm$0.002}} & \underline{\textit{0.789$\pm$0.059}} & \textbf{0.761$\pm$0.003} & \textbf{0.692$\pm$0.004} \\
 & GVD-CFM (DCT + temporal branch) & \textbf{0.978$\pm$0.005} & 0.775$\pm$0.073 & \underline{\textit{0.746$\pm$0.011}} & \underline{\textit{0.681$\pm$0.002}} \\
\midrule
Shin2017A & No DCT & 1.400$\pm$0.143 & 0.178$\pm$0.045 & 0.661$\pm$0.015 & 0.601$\pm$0.018 \\
 & DCT, spectral only & \textbf{1.153$\pm$0.022} & \underline{\textit{0.351$\pm$0.028}} & \underline{\textit{0.664$\pm$0.008}} & \underline{\textit{0.605$\pm$0.017}} \\
 & GVD-CFM (DCT + temporal branch) & \underline{\textit{1.190$\pm$0.084}} & \textbf{0.482$\pm$0.039} & \textbf{0.672$\pm$0.029} & \textbf{0.622$\pm$0.025} \\
\midrule
Zhou2016 & No DCT & 0.985$\pm$0.002 & \textbf{0.364$\pm$0.021} & \textbf{0.967$\pm$0.009} & \underline{\textit{0.903$\pm$0.007}} \\
 & DCT, spectral only & \underline{\textit{0.978$\pm$0.003}} & \underline{\textit{0.311$\pm$0.047}} & 0.964$\pm$0.002 & 0.896$\pm$0.003 \\
 & GVD-CFM (DCT + temporal branch) & \textbf{0.959$\pm$0.033} & 0.203$\pm$0.110 & \underline{\textit{0.965$\pm$0.003}} & \textbf{0.905$\pm$0.014} \\
\midrule
\textit{5-dataset avg.} & No DCT & 1.092 & 0.540 & 0.772 & 0.704 \\
\textit{5-dataset avg.} & DCT, spectral only & \underline{\textit{1.025}} & \underline{\textit{0.600}} & \textbf{0.792} & \underline{\textit{0.721}} \\
\textit{5-dataset avg.} & GVD-CFM (DCT + temporal branch) & \textbf{1.022} & \textbf{0.613} & \underline{\textit{0.790}} & \textbf{0.725} \\
\bottomrule
\end{tabular}}
\end{table*}

\begin{table*}[t]
\centering
\caption{Full per-dataset DCT and temporal-branch ablation: temporal and dynamic fidelity. Entries are mean$\pm$standard deviation over three generator seeds. For ratio metrics, ranking is by proximity to 1. Best values are \textbf{bold}; second-best values are \underline{\textit{underlined italics}}.}
\label{tab:appendix_dct_ablation_temporal_5ds}
\scriptsize
\renewcommand{\arraystretch}{0.95}
\setlength{\tabcolsep}{3.0pt}
\resizebox{0.92\textwidth}{!}{%
\begin{tabular}{llccccc}
\toprule
Dataset & Variant & Temp. corr. $\uparrow$ & Lag-ACF $\uparrow$ & Energy $\to1$ & Dyn. frac. $\to1$ & Adjacent $\to1$ \\
\midrule
BNCI2014\_001 & No DCT & 0.065$\pm$0.022 & \underline{\textit{0.465$\pm$0.158}} & \underline{\textit{0.992$\pm$0.027}} & \textbf{1.005$\pm$0.059} & 1.231$\pm$0.034 \\
 & DCT, spectral only & \textbf{0.939$\pm$0.003} & \textbf{0.999$\pm$0.000} & 1.013$\pm$0.026 & 1.025$\pm$0.031 & \textbf{1.052$\pm$0.024} \\
 & GVD-CFM (DCT + temporal branch) & \underline{\textit{0.920$\pm$0.006}} & \textbf{0.999$\pm$0.000} & \textbf{0.995$\pm$0.023} & \underline{\textit{1.024$\pm$0.052}} & \underline{\textit{1.070$\pm$0.013}} \\
\midrule
BNCI2014\_002 & No DCT & 0.140$\pm$0.018 & \underline{\textit{0.649$\pm$0.177}} & \underline{\textit{0.974$\pm$0.019}} & \textbf{0.980$\pm$0.013} & 1.193$\pm$0.024 \\
 & DCT, spectral only & \textbf{0.842$\pm$0.003} & \textbf{0.999$\pm$0.001} & \textbf{0.984$\pm$0.032} & \underline{\textit{0.963$\pm$0.002}} & \textbf{0.990$\pm$0.024} \\
 & GVD-CFM (DCT + temporal branch) & \underline{\textit{0.838$\pm$0.003}} & \textbf{0.999$\pm$0.000} & 0.948$\pm$0.018 & 0.938$\pm$0.036 & \underline{\textit{0.977$\pm$0.017}} \\
\midrule
BNCI2015\_001 & No DCT & 0.934$\pm$0.003 & \underline{\textit{0.984$\pm$0.002}} & \textbf{0.971$\pm$0.005} & \underline{\textit{1.010$\pm$0.010}} & 1.051$\pm$0.007 \\
 & DCT, spectral only & \underline{\textit{0.956$\pm$0.000}} & \textbf{0.999$\pm$0.000} & \underline{\textit{0.953$\pm$0.015}} & \textbf{0.995$\pm$0.004} & \underline{\textit{0.971$\pm$0.015}} \\
 & GVD-CFM (DCT + temporal branch) & \textbf{0.958$\pm$0.001} & \textbf{0.999$\pm$0.000} & 0.938$\pm$0.016 & 0.976$\pm$0.007 & \textbf{0.976$\pm$0.015} \\
\midrule
Shin2017A & No DCT & 0.016$\pm$0.020 & 0.252$\pm$0.251 & 1.084$\pm$0.074 & 1.199$\pm$0.051 & 1.695$\pm$0.116 \\
 & DCT, spectral only & \textbf{0.961$\pm$0.003} & \textbf{0.995$\pm$0.000} & \underline{\textit{1.080$\pm$0.033}} & \underline{\textit{1.170$\pm$0.022}} & \underline{\textit{1.318$\pm$0.020}} \\
 & GVD-CFM (DCT + temporal branch) & \underline{\textit{0.958$\pm$0.002}} & \underline{\textit{0.994$\pm$0.001}} & \textbf{1.068$\pm$0.006} & \textbf{1.161$\pm$0.011} & \textbf{1.317$\pm$0.005} \\
\midrule
Zhou2016 & No DCT & 0.865$\pm$0.006 & 0.981$\pm$0.003 & \underline{\textit{0.842$\pm$0.002}} & \underline{\textit{0.861$\pm$0.030}} & \textbf{0.942$\pm$0.004} \\
 & DCT, spectral only & \textbf{0.895$\pm$0.009} & \textbf{0.994$\pm$0.003} & \textbf{0.847$\pm$0.001} & \textbf{0.898$\pm$0.035} & \underline{\textit{0.898$\pm$0.005}} \\
 & GVD-CFM (DCT + temporal branch) & \underline{\textit{0.877$\pm$0.017}} & \underline{\textit{0.992$\pm$0.005}} & 0.815$\pm$0.039 & 0.835$\pm$0.011 & 0.871$\pm$0.027 \\
\midrule
\textit{5-dataset avg.} & No DCT & 0.404 & \underline{\textit{0.666}} & \underline{\textit{0.973}} & \underline{\textit{1.011}} & 1.222 \\
\textit{5-dataset avg.} & DCT, spectral only & \textbf{0.919} & \textbf{0.997} & \textbf{0.975} & \textbf{1.010} & \underline{\textit{1.046}} \\
\textit{5-dataset avg.} & GVD-CFM (DCT + temporal branch) & \underline{\textit{0.910}} & \textbf{0.997} & 0.953 & 0.987 & \textbf{1.042} \\
\bottomrule
\end{tabular}}
\end{table*}

\begin{table*}[t]
\centering
\caption{Paired per-dataset comparison for the DCT and temporal-branch ablation. Each entry is the change in the dataset mean (second variant minus first) followed by the Welch $t$-statistic computed from three generator seeds per variant, $t=\Delta/\sqrt{(\sigma_1^2+\sigma_2^2)/3}$. Entries with $|t|\geq2$ are \textbf{bold}; a dash marks zero seed variance. For Rel. GVD-FID a negative change is an improvement; for Energy the ideal value is 1.}
\label{tab:appendix_paired_ablation}
\scriptsize
\setlength{\tabcolsep}{2.6pt}
\resizebox{\textwidth}{!}{%
\begin{tabular}{llccccc}
\toprule
Comparison & Metric & BNCI2014\_001 & BNCI2014\_002 & BNCI2015\_001 & Shin2017A & Zhou2016 \\
\midrule
No DCT $\rightarrow$ DCT, spectral only & Rel. GVD-FID & -0.014 (-1.1) & -0.048 (\textbf{-5.5}) & -0.017 (\textbf{-4.7}) & -0.247 (\textbf{-3.0}) & -0.007 (\textbf{-3.4}) \\
 & Eva F1 & +0.184 (\textbf{3.4}) & +0.095 (1.7) & -0.100 (\textbf{-2.7}) & +0.173 (\textbf{5.7}) & -0.053 (-1.8) \\
 & CAS AUC & +0.011 (0.7) & +0.032 (\textbf{5.1}) & +0.057 (\textbf{27.4}) & +0.003 (0.3) & -0.003 (-0.6) \\
 & CAS F1 & +0.004 (0.3) & +0.031 (\textbf{5.0}) & +0.052 (\textbf{7.7}) & +0.004 (0.3) & -0.007 (-1.6) \\
 & Temp. corr. & +0.874 (\textbf{68.2}) & +0.702 (\textbf{66.6}) & +0.022 (\textbf{12.7}) & +0.945 (\textbf{80.9}) & +0.030 (\textbf{4.8}) \\
 & Lag-ACF & +0.534 (\textbf{5.9}) & +0.350 (\textbf{3.4}) & +0.015 (\textbf{13.0}) & +0.743 (\textbf{5.1}) & +0.013 (\textbf{5.3}) \\
 & Energy & +0.021 (1.0) & +0.010 (0.5) & -0.018 (-2.0) & -0.004 (-0.1) & +0.005 (\textbf{3.9}) \\
\midrule
DCT, spectral only $\rightarrow$ GVD-CFM & Rel. GVD-FID & -0.010 (-0.8) & -0.011 (-1.0) & -0.014 (\textbf{-4.5}) & +0.037 (0.7) & -0.019 (-1.0) \\
 & Eva F1 & +0.039 (2.0) & +0.017 (0.6) & -0.014 (-0.3) & +0.131 (\textbf{4.7}) & -0.108 (-1.6) \\
 & CAS AUC & +0.005 (0.5) & -0.009 (-1.0) & -0.015 (\textbf{-2.3}) & +0.008 (0.5) & +0.001 (0.5) \\
 & CAS F1 & +0.012 (\textbf{2.4}) & -0.007 (-1.6) & -0.011 (\textbf{-4.3}) & +0.017 (1.0) & +0.009 (1.1) \\
 & Temp. corr. & -0.019 (\textbf{-4.9}) & -0.004 (-1.6) & +0.002 (\textbf{3.5}) & -0.003 (-1.4) & -0.018 (-1.6) \\
 & Lag-ACF & +0.000 (--) & +0.000 (0.0) & +0.000 (--) & -0.001 (-1.7) & -0.002 (-0.6) \\
 & Energy & -0.018 (-0.9) & -0.036 (-1.7) & -0.015 (-1.2) & -0.012 (-0.6) & -0.032 (-1.4) \\
\bottomrule
\end{tabular}}
\end{table*}

\begin{table*}[t]
\centering
\caption{Full per-dataset stable-support ablation summarized in Section~\ref{sec:timecontext_ablation}. GVD-CFM uses the canonical ridge-free stable support; the no-support control uses the ridge required to obtain SPD matrices and the spectral-only velocity network. Entries are mean$\pm$standard deviation over three generator seeds. Best values are \textbf{bold}; second-best values are \underline{\textit{underlined italics}}.}
\label{tab:appendix_support_ablation_5ds}
\scriptsize
\renewcommand{\arraystretch}{0.95}
\setlength{\tabcolsep}{3.0pt}
\resizebox{\textwidth}{!}{%
\begin{tabular}{llccccccc}
\toprule
Dataset & Variant & Rel. GVD-FID $\downarrow$ & Eva F1 $\uparrow$ & CAS AUC $\uparrow$ & CAS F1 $\uparrow$ & Temp. corr. $\uparrow$ & Lag-ACF $\uparrow$ & Energy $\to1$ \\
\midrule
BNCI2014\_001 & GVD-CFM (stable support) & \underline{\textit{1.006$\pm$0.014}} & \textbf{0.864$\pm$0.027} & \textbf{0.789$\pm$0.015} & \textbf{0.707$\pm$0.007} & \underline{\textit{0.920$\pm$0.006}} & \underline{\textit{0.999$\pm$0.000}} & \textbf{0.995$\pm$0.023} \\
 & No support + ridge & \textbf{0.917$\pm$0.006} & \underline{\textit{0.292$\pm$0.032}} & \underline{\textit{0.667$\pm$0.019}} & \underline{\textit{0.614$\pm$0.017}} & \textbf{0.977$\pm$0.001} & \textbf{1.000$\pm$0.000} & \underline{\textit{0.826$\pm$0.007}} \\
\midrule
BNCI2014\_002 & GVD-CFM (stable support) & \underline{\textit{0.977$\pm$0.016}} & \textbf{0.740$\pm$0.022} & \textbf{0.780$\pm$0.015} & \textbf{0.709$\pm$0.005} & \underline{\textit{0.838$\pm$0.003}} & \textbf{0.999$\pm$0.000} & \textbf{0.948$\pm$0.018} \\
 & No support + ridge & \textbf{0.909$\pm$0.001} & \underline{\textit{0.481$\pm$0.024}} & \underline{\textit{0.736$\pm$0.009}} & \underline{\textit{0.674$\pm$0.010}} & \textbf{0.929$\pm$0.001} & \textbf{0.999$\pm$0.000} & \underline{\textit{0.850$\pm$0.010}} \\
\midrule
BNCI2015\_001 & GVD-CFM (stable support) & \underline{\textit{0.978$\pm$0.005}} & \textbf{0.775$\pm$0.073} & \textbf{0.746$\pm$0.011} & \textbf{0.681$\pm$0.002} & \underline{\textit{0.958$\pm$0.001}} & \textbf{0.999$\pm$0.000} & \textbf{0.938$\pm$0.016} \\
 & No support + ridge & \textbf{0.921$\pm$0.005} & \underline{\textit{0.372$\pm$0.050}} & \underline{\textit{0.698$\pm$0.012}} & \underline{\textit{0.638$\pm$0.006}} & \textbf{0.980$\pm$0.000} & \textbf{0.999$\pm$0.000} & \underline{\textit{0.874$\pm$0.006}} \\
\midrule
Shin2017A & GVD-CFM (stable support) & \underline{\textit{1.190$\pm$0.084}} & \underline{\textit{0.482$\pm$0.039}} & \textbf{0.672$\pm$0.029} & \textbf{0.622$\pm$0.025} & \underline{\textit{0.958$\pm$0.002}} & \underline{\textit{0.994$\pm$0.001}} & \underline{\textit{1.068$\pm$0.006}} \\
 & No support + ridge & \textbf{0.999$\pm$0.010} & \textbf{0.800$\pm$0.035} & \underline{\textit{0.622$\pm$0.008}} & \underline{\textit{0.584$\pm$0.005}} & \textbf{0.971$\pm$0.002} & \textbf{0.996$\pm$0.000} & \textbf{0.990$\pm$0.021} \\
\midrule
Zhou2016 & GVD-CFM (stable support) & \underline{\textit{0.959$\pm$0.033}} & \underline{\textit{0.203$\pm$0.110}} & \textbf{0.965$\pm$0.003} & \textbf{0.905$\pm$0.014} & \underline{\textit{0.877$\pm$0.017}} & \textbf{0.992$\pm$0.005} & \textbf{0.815$\pm$0.039} \\
 & No support + ridge & \textbf{0.922$\pm$0.004} & \textbf{0.237$\pm$0.021} & \underline{\textit{0.955$\pm$0.007}} & \underline{\textit{0.883$\pm$0.010}} & \textbf{0.905$\pm$0.007} & \underline{\textit{0.980$\pm$0.007}} & \underline{\textit{0.794$\pm$0.007}} \\
\midrule
\textit{5-dataset avg.} & GVD-CFM (stable support) & \underline{\textit{1.022}} & \textbf{0.613} & \textbf{0.790} & \textbf{0.725} & \underline{\textit{0.910}} & \textbf{0.997} & \textbf{0.953} \\
\textit{5-dataset avg.} & No support + ridge & \textbf{0.934} & \underline{\textit{0.436}} & \underline{\textit{0.736}} & \underline{\textit{0.679}} & \textbf{0.952} & \underline{\textit{0.995}} & \underline{\textit{0.867}} \\
\bottomrule
\end{tabular}}
\end{table*}


\begin{table*}[t]
\centering
\caption{Full temporal-basis ablation across the five main EEG datasets: generative fidelity and downstream utility. Each per-dataset entry is mean$\pm$standard deviation over three generator seeds. All four variants retain all $B=100$ temporal coordinates. Random orthogonal uses a fixed random orthogonal basis, KLT the training-set empirical temporal Karhunen--Lo\`eve basis, and GVD-CFM (DCT) the fixed orthonormal DCT-II basis with the full network including the temporal branch; the other variants use the spectral-only velocity network. Best values are \textbf{bold}; second-best values are \underline{\textit{underlined italics}}.}
\label{tab:appendix_four_basis_ablation_utility}
\scriptsize
\renewcommand{\arraystretch}{0.92}
\setlength{\tabcolsep}{3.2pt}

\resizebox{0.96\textwidth}{!}{%
\begin{tabular}{llcccc}
\toprule
Dataset
& Basis
& Rel. GVD-FID $\downarrow$
& Eva F1 $\uparrow$
& CAS AUC $\uparrow$
& CAS F1 $\uparrow$ \\
\midrule
BNCI2014\_001 & No DCT & 1.030$\pm$0.017 & 0.641$\pm$0.090 & 0.773$\pm$0.027 & 0.691$\pm$0.023 \\
 & Random orthogonal & 1.071$\pm$0.003 & 0.461$\pm$0.045 & 0.780$\pm$0.021 & 0.695$\pm$0.011 \\
 & GVD-CFM (DCT) & \textbf{1.006$\pm$0.014} & \textbf{0.864$\pm$0.027} & \textbf{0.789$\pm$0.015} & \textbf{0.707$\pm$0.007} \\
 & KLT & \underline{\textit{1.023$\pm$0.006}} & \underline{\textit{0.802$\pm$0.008}} & \underline{\textit{0.787$\pm$0.012}} & \underline{\textit{0.697$\pm$0.010}} \\
\midrule
BNCI2014\_002 & No DCT & 1.036$\pm$0.009 & 0.628$\pm$0.090 & 0.757$\pm$0.009 & 0.685$\pm$0.009 \\
 & Random orthogonal & 1.048$\pm$0.015 & 0.578$\pm$0.050 & \underline{\textit{0.783$\pm$0.006}} & 0.700$\pm$0.015 \\
 & GVD-CFM (DCT) & \textbf{0.977$\pm$0.016} & \underline{\textit{0.740$\pm$0.022}} & 0.780$\pm$0.015 & \underline{\textit{0.709$\pm$0.005}} \\
 & KLT & \underline{\textit{0.984$\pm$0.002}} & \textbf{0.758$\pm$0.025} & \textbf{0.794$\pm$0.035} & \textbf{0.737$\pm$0.039} \\
\midrule
BNCI2015\_001 & No DCT & 1.009$\pm$0.006 & \textbf{0.889$\pm$0.026} & 0.704$\pm$0.002 & 0.640$\pm$0.011 \\
 & Random orthogonal & 1.031$\pm$0.019 & 0.652$\pm$0.121 & \underline{\textit{0.754$\pm$0.014}} & \underline{\textit{0.683$\pm$0.013}} \\
 & GVD-CFM (DCT) & \textbf{0.978$\pm$0.005} & 0.775$\pm$0.073 & 0.746$\pm$0.011 & 0.681$\pm$0.002 \\
 & KLT & \underline{\textit{0.994$\pm$0.003}} & \underline{\textit{0.846$\pm$0.022}} & \textbf{0.755$\pm$0.013} & \textbf{0.690$\pm$0.010} \\
\midrule
Shin2017A & No DCT & 1.400$\pm$0.143 & 0.178$\pm$0.045 & 0.661$\pm$0.015 & 0.601$\pm$0.018 \\
 & Random orthogonal & 1.304$\pm$0.041 & 0.075$\pm$0.009 & 0.659$\pm$0.014 & 0.596$\pm$0.015 \\
 & GVD-CFM (DCT) & \underline{\textit{1.190$\pm$0.084}} & \textbf{0.482$\pm$0.039} & \underline{\textit{0.672$\pm$0.029}} & \textbf{0.622$\pm$0.025} \\
 & KLT & \textbf{1.183$\pm$0.036} & \underline{\textit{0.316$\pm$0.036}} & \textbf{0.675$\pm$0.006} & \underline{\textit{0.620$\pm$0.008}} \\
\midrule
Zhou2016 & No DCT & 0.985$\pm$0.002 & \underline{\textit{0.364$\pm$0.021}} & \textbf{0.967$\pm$0.009} & \underline{\textit{0.903$\pm$0.007}} \\
 & Random orthogonal & 1.013$\pm$0.018 & \textbf{0.636$\pm$0.172} & 0.964$\pm$0.002 & 0.891$\pm$0.006 \\
 & GVD-CFM (DCT) & \textbf{0.959$\pm$0.033} & 0.203$\pm$0.110 & 0.965$\pm$0.003 & \textbf{0.905$\pm$0.014} \\
 & KLT & \underline{\textit{0.979$\pm$0.003}} & 0.277$\pm$0.035 & \underline{\textit{0.966$\pm$0.002}} & 0.901$\pm$0.005 \\
\midrule
\textit{5-dataset avg.} & No DCT & 1.092 & 0.540 & 0.772 & 0.704 \\
\textit{5-dataset avg.} & Random orthogonal & 1.093 & 0.480 & 0.788 & 0.713 \\
\textit{5-dataset avg.} & GVD-CFM (DCT) & \textbf{1.022} & \textbf{0.613} & \underline{\textit{0.790}} & \underline{\textit{0.725}} \\
\textit{5-dataset avg.} & KLT & \underline{\textit{1.033}} & \underline{\textit{0.600}} & \textbf{0.795} & \textbf{0.729} \\
\bottomrule
\end{tabular}}
\end{table*}

\begin{thebibliography}{118}
\providecommand{\natexlab}[1]{#1}
\providecommand{\url}[1]{\texttt{#1}}
\expandafter\ifx\csname urlstyle\endcsname\relax
  \providecommand{\doi}[1]{doi: #1}\else
  \providecommand{\doi}{doi: \begingroup \urlstyle{rm}\Url}\fi

\bibitem[Ahmed et~al.(1974)Ahmed, Natarajan, and Rao]{ahmed1974dct}
Nasir Ahmed, T.~Natarajan, and Kamisetty~R. Rao.
\newblock Discrete cosine transform.
\newblock \emph{IEEE Transactions on Computers}, C-23\penalty0 (1):\penalty0
  90--93, 1974.

\bibitem[Alaa et~al.(2022)Alaa, van Breugel, Saveliev, and van~der
  Schaar]{alaa2022faithful}
Ahmed Alaa, Boris van Breugel, Evgeny~S. Saveliev, and Mihaela van~der Schaar.
\newblock How faithful is your synthetic data? {S}ample-level metrics for
  evaluating and auditing generative models.
\newblock In \emph{39th International Conference on Machine Learning}, pages
  290--306, Baltimore, MD, July 2022.

\bibitem[Albergo and Vanden-Eijnden(2023)]{albergo2023interpolants}
Michael~S. Albergo and Eric Vanden-Eijnden.
\newblock Building normalizing flows with stochastic interpolants.
\newblock In \emph{11th International Conference on Learning Representations},
  Kigali, Rwanda, May 2023.

\bibitem[Allen et~al.(2014)Allen, Damaraju, Plis, Erhardt, Eichele, and
  Calhoun]{allen2014tracking}
Elena~A. Allen, Eswar Damaraju, Sergey~M. Plis, Erik~B. Erhardt, Tom Eichele,
  and Vince~D. Calhoun.
\newblock Tracking whole-brain connectivity dynamics in the resting state.
\newblock \emph{Cerebral Cortex}, 24\penalty0 (3):\penalty0 663--676, 2014.

\bibitem[Aristimunha et~al.(2023)Aristimunha, Carrara, Guetschel, Sedlar,
  Rodrigues, Sosulski, Narayanan, Bjareholt, Barthelemy, Schirrmeister,
  Kalunga, Darmet, Gregoire, Hussain, Gatti, Goncharenko, Thielen, Moreau, Roy,
  Jayaram, Barachant, and Chevallier]{aristimunha2023moabb}
Bruno Aristimunha, Igor Carrara, Pierre Guetschel, Sara Sedlar, Pedro
  Rodrigues, Jan Sosulski, Divyesh Narayanan, Erik Bjareholt, Quentin
  Barthelemy, Robin~T. Schirrmeister, Emmanuel Kalunga, Ludovic Darmet, Cattan
  Gregoire, Ali~Abdul Hussain, Ramiro Gatti, Vladislav Goncharenko, Jordy
  Thielen, Thomas Moreau, Yannick Roy, Vinay Jayaram, Alexandre Barachant, and
  Sylvain Chevallier.
\newblock Mother of all {BCI} benchmarks ({MOABB}), 2023.

\bibitem[Arjovsky et~al.(2017)Arjovsky, Chintala, and Bottou]{arjovsky2017wgan}
Martin Arjovsky, Soumith Chintala, and L\'{e}on Bottou.
\newblock {W}asserstein generative adversarial networks.
\newblock In \emph{34th International Conference on Machine Learning}, pages
  214--223, Sydney, Australia, August 2017.

\bibitem[Arsigny et~al.(2007)Arsigny, Fillard, Pennec, and
  Ayache]{arsigny2007logeuclidean}
Vincent Arsigny, Pierre Fillard, Xavier Pennec, and Nicholas Ayache.
\newblock Geometric means in a novel vector space structure on symmetric
  positive-definite matrices.
\newblock \emph{SIAM Journal on Matrix Analysis and Applications}, 29\penalty0
  (1):\penalty0 328--347, 2007.

\bibitem[Ba et~al.(2016)Ba, Kiros, and Hinton]{ba2016layernorm}
Jimmy~Lei Ba, Jamie~Ryan Kiros, and Geoffrey~E. Hinton.
\newblock Layer normalization.
\newblock \emph{arXiv:1607.06450 [stat.ML]}, 2016.

\bibitem[Barachant(2012)]{barachant2012thesis}
Alexandre Barachant.
\newblock \emph{Commande robuste d'un effecteur par une interface cerveau
  machine {EEG} asynchrone}.
\newblock PhD thesis, Universit\'{e} de Grenoble, 2012.

\bibitem[Barachant et~al.(2012)Barachant, Bonnet, Congedo, and
  Jutten]{barachant2012riemannian}
Alexandre Barachant, St\'{e}phane Bonnet, Marco Congedo, and Christian Jutten.
\newblock Multiclass brain--computer interface classification by {R}iemannian
  geometry.
\newblock \emph{IEEE Transactions on Biomedical Engineering}, 59\penalty0
  (4):\penalty0 920--928, 2012.

\bibitem[Blankertz et~al.(2008)Blankertz, Tomioka, Lemm, Kawanabe, and
  M\"{u}ller]{blankertz2008csp}
Benjamin Blankertz, Ryota Tomioka, Steven Lemm, Motoaki Kawanabe, and
  Klaus-Robert M\"{u}ller.
\newblock Optimizing spatial filters for robust {EEG} single-trial analysis.
\newblock \emph{IEEE Signal Processing Magazine}, 25\penalty0 (1):\penalty0
  41--56, 2008.

\bibitem[De~Bortoli et~al.(2022)De~Bortoli, Mathieu, Hutchinson, Thornton, Teh, and
  Doucet]{debortoli2022riemannian}
Valentin~De Bortoli, Emile Mathieu, Michael Hutchinson, James Thornton,
  Yee~Whye Teh, and Arnaud Doucet.
\newblock {R}iemannian score-based generative modelling.
\newblock In \emph{36th Conference on Neural Information Processing Systems},
  pages 2406--2422, December 2022.

\bibitem[Brown et~al.(2020)Brown, Mann, Ryder, Subbiah, Kaplan, Dhariwal,
  Neelakantan, Shyam, Sastry, Askell, Agarwal, Herbert-Voss, Krueger, Henighan,
  Child, Ramesh, Ziegler, Wu, Winter, Hesse, Chen, Sigler, Litwin, Gray, Chess,
  Clark, Berner, McCandlish, Radford, Sutskever, and Amodei]{brown2020gpt3}
Tom~B. Brown, Benjamin Mann, Nick Ryder, Melanie Subbiah, Jared Kaplan,
  Prafulla Dhariwal, Arvind Neelakantan, Pranav Shyam, Girish Sastry, Amanda
  Askell, Sandhini Agarwal, Ariel Herbert-Voss, Gretchen Krueger, Tom Henighan,
  Rewon Child, Aditya Ramesh, Daniel~M. Ziegler, Jeffrey Wu, Clemens Winter,
  Christopher Hesse, Mark Chen, Eric Sigler, Mateusz Litwin, Scott Gray,
  Benjamin Chess, Jack Clark, Christopher Berner, Sam McCandlish, Alec Radford,
  Ilya Sutskever, and Dario Amodei.
\newblock Language models are few-shot learners.
\newblock In \emph{34th Conference on Neural Information Processing Systems},
  pages 1877--1901, December 2020.

\bibitem[Brunner et~al.(2008)Brunner, Leeb, M\"{u}ller-Putz, Schl\"{o}gl, and
  Pfurtscheller]{brunner2008graz}
Clemens Brunner, Robert Leeb, Gernot M\"{u}ller-Putz, Alois Schl\"{o}gl, and
  Gert Pfurtscheller.
\newblock {BCI} {C}ompetition 2008 -- {G}raz data set {A}.
\newblock Technical report, Institute for Knowledge Discovery, Graz University
  of Technology, 2008.

\bibitem[Cavallo et~al.(2024)Cavallo, Sabbaqi, and
  Isufi]{cavallo2024spatiotemporal}
Andrea Cavallo, Maosheng Sabbaqi, and Elvin Isufi.
\newblock Spatiotemporal covariance neural networks.
\newblock In \emph{Joint European Conference on Machine Learning and Knowledge
  Discovery in Databases}, pages 18--34, Cham, August 2024. Springer Nature
  Switzerland.

\bibitem[Chen and Lipman(2024)]{chen2024geometries}
Ricky T.~Q. Chen and Yaron Lipman.
\newblock Flow matching on general geometries.
\newblock In \emph{12th International Conference on Learning Representations},
  Vienna, Austria, May 2024.

\bibitem[Chen et~al.(2018)Chen, Rubanova, Bettencourt, and
  Duvenaud]{chen2018neuralode}
Ricky T.~Q. Chen, Yulia Rubanova, Jesse Bettencourt, and David~K. Duvenaud.
\newblock Neural ordinary differential equations.
\newblock In \emph{32nd Conference on Neural Information Processing Systems},
  pages 6572--6583, Montr\'{e}al, QC, December 2018.

\bibitem[Chen et~al.(2010)Chen, Wiesel, Eldar, and Hero]{chen2010oas}
Yilun Chen, Ami Wiesel, Yonina~C. Eldar, and Alfred~O. Hero.
\newblock Shrinkage algorithms for {MMSE} covariance estimation.
\newblock \emph{IEEE Transactions on Signal Processing}, 58\penalty0
  (10):\penalty0 5016--5029, 2010.

\bibitem[Collas et~al.(2025)Collas, Ju, Salvy, and
  Thirion]{collas2025diffeocfm}
Antoine Collas, Ce~Ju, Nicolas Salvy, and Bertrand Thirion.
\newblock {R}iemannian flow matching for brain connectivity matrices via
  pullback geometry ({DIFFEOCFM}).
\newblock In \emph{39th Conference on Neural Information Processing Systems},
  December 2025.

\bibitem[Cuturi(2013)]{cuturi2013sinkhorn}
Marco Cuturi.
\newblock Sinkhorn distances: Lightspeed computation of optimal transport.
\newblock In \emph{Advances in Neural Information Processing Systems}, 2013.

\bibitem[David and Gu(2019)]{david2019correlation}
Paul David and Weiqing Gu.
\newblock A {R}iemannian structure for correlation matrices.
\newblock \emph{Operators and Matrices}, 13\penalty0 (3):\penalty0 607--627,
  2019.

\bibitem[de~Surrel et~al.(2025)de~Surrel, Lotte, Chevallier, and
  Yger]{desurrel2025wrapped}
Thibault de~Surrel, Fabien Lotte, Sylvain Chevallier, and Florian Yger.
\newblock Wrapped {G}aussian on the manifold of symmetric positive definite
  matrices.
\newblock In \emph{42nd International Conference on Machine Learning}, July
  2025.

\bibitem[Dhariwal and Nichol(2021)]{dhariwal2021beatgans}
Prafulla Dhariwal and Alexander Nichol.
\newblock Diffusion models beat {GAN}s on image synthesis.
\newblock In \emph{35th Conference on Neural Information Processing Systems},
  pages 8780--8794, December 2021.

\bibitem[Ding et~al.(2025)Ding, Jaquier, Peters, and Rozo]{ding2025visuomotor}
Haoran Ding, No\'{e}mie Jaquier, Jan Peters, and Leonel Rozo.
\newblock Fast and robust visuomotor {R}iemannian flow matching policy.
\newblock \emph{IEEE Transactions on Robotics}, 2025.

\bibitem[Dinh et~al.(2017)Dinh, Sohl-Dickstein, and Bengio]{dinh2017realnvp}
Laurent Dinh, Jascha Sohl-Dickstein, and Samy Bengio.
\newblock Density estimation using {R}eal {NVP}.
\newblock In \emph{5th International Conference on Learning Representations},
  Toulon, France, April 2017.

\bibitem[Dowson and Landau(1982)]{dowson1982frechet}
D.~C. Dowson and B.~V. Landau.
\newblock The {F}r\'{e}chet distance between multivariate normal distributions.
\newblock \emph{Journal of Multivariate Analysis}, 12\penalty0 (3):\penalty0
  450--455, 1982.

\bibitem[Esteban et~al.(2017)Esteban, Hyland, and R\"{a}tsch]{esteban2017rcgan}
Crist\'{o}bal Esteban, Stephanie~L. Hyland, and Gunnar R\"{a}tsch.
\newblock Real-valued (medical) time series generation with recurrent
  conditional {GAN}s.
\newblock \emph{arXiv:1706.02633 [stat.ML]}, 2017.

\bibitem[Faller et~al.(2012)Faller, Vidaurre, Solis-Escalante, Neuper, and
  Scherer]{faller2012autocalibration}
Josef Faller, Carmen Vidaurre, Teodoro Solis-Escalante, Christa Neuper, and
  Reinhold Scherer.
\newblock Autocalibration and recurrent adaptation: towards a plug and play
  online {ERD-BCI}.
\newblock \emph{IEEE Transactions on Neural Systems and Rehabilitation
  Engineering}, 20\penalty0 (3):\penalty0 313--319, 2012.

\bibitem[Falorsi et~al.(2019)Falorsi, de~Haan, Davidson, and
  Forr\'{e}]{falorsi2019lie}
Luca Falorsi, Pim de~Haan, Tim~R. Davidson, and Patrick Forr\'{e}.
\newblock Reparameterizing distributions on {L}ie groups.
\newblock In \emph{22nd International Conference on Artificial Intelligence and
  Statistics}, pages 3244--3253, Naha, Japan, April 2019.

\bibitem[Fletcher and Joshi(2004)]{fletcher2004pga}
P.~Thomas Fletcher and Sarang Joshi.
\newblock Principal geodesic analysis on symmetric spaces: statistics of
  diffusion tensors.
\newblock In \emph{Sonka, M., Kakadiaris, I.A., Kybic, J. (eds) Computer Vision and Mathematical Methods in Medical and Biomedical Image Analysis. MMBIA CVAMIA 2004 2004.}, Lecture Notes in Computer Science, vol 3117. Springer, Berlin, Heidelberg.

\bibitem[Fox et~al.(2014)Fox, Buckner, Liu, Chakravarty, Lozano, and
  Pascual-Leone]{fox2014resting}
Michael~D. Fox, Randy~L. Buckner, Hesheng Liu, M.~Mallar Chakravarty, Andres~M.
  Lozano, and Alvaro Pascual-Leone.
\newblock Resting-state networks link invasive and noninvasive brain
  stimulation across diverse psychiatric and neurological diseases.
\newblock \emph{Proceedings of the National Academy of Sciences}, 111\penalty0
  (41):\penalty0 E4367--E4375, 2014.

\bibitem[Golub and Loan(2013)]{golub2013matrix}
Gene~H. Golub and Charles F.~Van Loan.
\newblock \emph{Matrix Computations}.
\newblock JHU Press, Baltimore, MD, 4th edition, 2013.

\bibitem[Goodfellow et~al.(2014)Goodfellow, Pouget-Abadie, Mirza, Xu,
  Warde-Farley, Ozair, Courville, and Bengio]{goodfellow2014gan}
Ian Goodfellow, Jean Pouget-Abadie, Mehdi Mirza, Bing Xu, David Warde-Farley,
  Sherjil Ozair, Aaron Courville, and Yoshua Bengio.
\newblock Generative adversarial networks.
\newblock In \emph{28th Conference on Neural Information Processing Systems},
  pages 2672--2680, Montr\'{e}al, QC, December 2014.

\bibitem[Gramfort et~al.(2013)Gramfort, Luessi, Larson, Engemann, Strohmeier,
  Brodbeck, Goj, Jas, Brooks, Parkkonen, and
  H\"{a}m\"{a}l\"{a}inen]{gramfort2013mne}
Alexandre Gramfort, Martin Luessi, Eric Larson, Denis~A. Engemann, Daniel
  Strohmeier, Christian Brodbeck, Roman Goj, Mainak Jas, Teon Brooks, Lauri
  Parkkonen, and Matti~S. H\"{a}m\"{a}l\"{a}inen.
\newblock {MEG} and {EEG} data analysis with {MNE-P}ython.
\newblock \emph{Frontiers in Neuroscience}, 7:\penalty0 267, 2013.

\bibitem[Gulrajani et~al.(2017)Gulrajani, Ahmed, Arjovsky, Dumoulin, and
  Courville]{gulrajani2017wgangp}
Ishaan Gulrajani, Faruk Ahmed, Martin Arjovsky, Vincent Dumoulin, and Aaron
  Courville.
\newblock Improved training of {W}asserstein {GAN}s.
\newblock In \emph{31st Conference on Neural Information Processing Systems},
  pages 5769--5779, Long Beach, CA, December 2017.

\bibitem[Hairer et~al.(1993)Hairer, N{\o}rsett, and Wanner]{hairer1993solving}
Ernst Hairer, Syvert~P. N{\o}rsett, and Gerhard Wanner.
\newblock \emph{Solving Ordinary Differential Equations I: Nonstiff Problems}.
\newblock Springer-Verlag, Berlin, Heidelberg, 2nd edition, 1993.

\bibitem[Hallac et~al.(2017)Hallac, Park, Boyd, and Leskovec]{hallac2017tvgl}
David Hallac, Youngsuk Park, Stephen Boyd, and Jure Leskovec.
\newblock Network inference via the time-varying graphical lasso.
\newblock In \emph{23rd ACM SIGKDD International Conference on Knowledge
  Discovery and Data Mining}, pages 205--213, Halifax, NS, August 2017.

\bibitem[Hartmann et~al.(2018)Hartmann, Schirrmeister, and
  Ball]{hartmann2018eeggan}
Kay~Gregor Hartmann, Robin~Tibor Schirrmeister, and Tonio Ball.
\newblock {EEG-GAN}: generative adversarial networks for
  electroencephalographic ({EEG}) brain signals.
\newblock \emph{arXiv:1806.01875 [eess.SP]}, 2018.

\bibitem[Hendrycks and Gimpel(2016)]{hendrycks2016gelu}
Dan Hendrycks and Kevin Gimpel.
\newblock {G}aussian error linear units ({GELU}s).
\newblock \emph{arXiv:1606.08415 [cs.LG]}, 2016.

\bibitem[Heusel et~al.(2017)Heusel, Ramsauer, Unterthiner, Nessler, and
  Hochreiter]{heusel2017fid}
Martin Heusel, Hubert Ramsauer, Thomas Unterthiner, Bernhard Nessler, and Sepp
  Hochreiter.
\newblock {GAN}s trained by a two time-scale update rule converge to a local
  {N}ash equilibrium.
\newblock In \emph{31st Conference on Neural Information Processing Systems},
  pages 6629--6640, Long Beach, CA, December 2017.

\bibitem[Hindriks et~al.(2016)Hindriks, Adhikari, Murayama, Ganzetti, Mantini,
  Logothetis, and Deco]{hindriks2016sliding}
Rikkert Hindriks, Mohit~H. Adhikari, Yusuke Murayama, Marco Ganzetti, Dante
  Mantini, Nikos~K. Logothetis, and Gustavo Deco.
\newblock Can sliding-window correlations reveal dynamic functional
  connectivity in resting-state {fMRI}?
\newblock \emph{NeuroImage}, 127:\penalty0 242--256, 2016.

\bibitem[Ho and Salimans(2022)]{ho2022cfg}
Jonathan Ho and Tim Salimans.
\newblock Classifier-free diffusion guidance.
\newblock \emph{arXiv:2207.12598 [cs.LG]}, 2022.

\bibitem[Ho et~al.(2020)Ho, Jain, and Abbeel]{ho2020ddpm}
Jonathan Ho, Ajay Jain, and Pieter Abbeel.
\newblock Denoising diffusion probabilistic models.
\newblock In \emph{34th Conference on Neural Information Processing Systems},
  pages 6840--6851, December 2020.

\bibitem[Horn and Johnson(2012)]{horn2012matrix}
Roger~A. Horn and Charles~R. Johnson.
\newblock \emph{Matrix Analysis}.
\newblock Cambridge University Press, Cambridge, 2nd edition, 2012.

\bibitem[Huang and Ruan(2025)]{huang2025spectral}
Xikun Huang,Tianyu Ruan, Chihao Zhang and Shihua Zhang.
\newblock Graph generation with spectral geodesic flow matching.
\newblock \emph{arXiv:2510.02520 [cs.LG]}, 2025.

\bibitem[Huguet et~al.(2024)Huguet, Vuckovic, Fatras, Thibodeau-Laufer, Lemos,
  Islam, Liu, Rector-Brooks, Akhound-Sadegh, Bronstein, Tong, and
  Bose]{huguet2024se3}
Guillaume Huguet, James Vuckovic, Kilian Fatras, Eric Thibodeau-Laufer, Pablo
  Lemos, Riashat Islam, Cheng-Hao Liu, Jarrid Rector-Brooks, Tara
  Akhound-Sadegh, Michael~M. Bronstein, Alexander Tong, and Avishek~Joey Bose.
\newblock Sequence-augmented {SE}(3)-flow matching for conditional protein
  generation.
\newblock In \emph{38th Conference on Neural Information Processing Systems},
  December 2024.

\bibitem[Hutchison et~al.(2013)Hutchison, Womelsdorf, Allen, Bandettini,
  Calhoun, Corbetta, Penna, Duyn, Glover, Gonzalez-Castillo, Handwerker,
  Keilholz, Kiviniemi, Leopold, de~Pasquale, Sporns, Walter, and
  Chang]{hutchison2013dfc}
R.~Matthew Hutchison, Thilo Womelsdorf, Elena~A. Allen, Peter~A. Bandettini,
  Vince~D. Calhoun, Maurizio Corbetta, Stefania~Della Penna, Jeff~H. Duyn,
  Gary~H. Glover, Javier Gonzalez-Castillo, Daniel~A. Handwerker, Shella
  Keilholz, Vesa Kiviniemi, David~A. Leopold, Francesco de~Pasquale, Olaf
  Sporns, Martin Walter, and Catie Chang.
\newblock Dynamic functional connectivity: promise, issues, and
  interpretations.
\newblock \emph{NeuroImage}, 80:\penalty0 360--378, 2013.

\bibitem[Jayaram and Barachant(2018)]{jayaram2018moabb}
Vinay Jayaram and Alexandre Barachant.
\newblock {MOABB}: trustworthy algorithm benchmarking for {BCI}s.
\newblock \emph{Journal of Neural Engineering}, 15\penalty0 (6):\penalty0
  066011, 2018.

\bibitem[Jiang et~al.(2024)Jiang, Zhao, and Lu]{jiang2024labram}
Weibang Jiang, Liming Zhao, and Bao-Liang Lu.
\newblock Large brain model for learning generic representations with
  tremendous {EEG} data in {BCI}.
\newblock In \emph{12th International Conference on Learning Representations},
  Vienna, Austria, May 2024.

\bibitem[Jo and Hwang(2023)]{jo2023mixture}
Jaehyeong Jo and Sung~Ju Hwang.
\newblock Generative modeling on manifolds through mixture of {R}iemannian
  diffusion processes.
\newblock \emph{arXiv:2310.07216 [cs.LG]}, 2023.

\bibitem[Ju et~al.(2025)Ju, Kobler, Collas, Kawanabe, Guan, and
  Thirion]{ju2025spdsurvey}
Ce~Ju, Reinmar~J. Kobler, Antoine Collas, Motoaki Kawanabe, Cuntai Guan, and
  Bertrand Thirion.
\newblock {SPD} learning for covariance-based neuroimaging analysis:
  perspectives, methods, and challenges.
\newblock \emph{arXiv:2504.18882 [cs.LG]}, 2025.

\bibitem[Kapusniak et~al.(2024)Kapusniak, Potaptchik, Reu, Zhang, Tong,
  Bronstein, Bose, and Giovanni]{kapusniak2024metric}
Kacper Kapusniak, Peter Potaptchik, Teodora Reu, Leo Zhang, Alexander Tong,
  Michael Bronstein, Avishek~Joey Bose, and Francesco~Di Giovanni.
\newblock Metric flow matching for smooth interpolations on the data manifold.
\newblock In \emph{38th Conference on Neural Information Processing Systems},
  December 2024.

\bibitem[Kingma and Ba(2015)]{kingma2015adam}
Diederik~P. Kingma and Jimmy~Lei Ba.
\newblock {ADAM}: A method for stochastic optimization.
\newblock In \emph{3rd International Conference on Learning Representations},
  pages 1--15, San Diego, CA, May 2015.

\bibitem[Kingma and Welling(2014)]{kingma2014vae}
Diederik~P. Kingma and Max Welling.
\newblock Auto-encoding variational {B}ayes.
\newblock In \emph{2nd International Conference on Learning Representations},
  Banff, AB, April 2014.

\bibitem[Kobler et~al.(2022)Kobler, ichiro Hirayama, Zhao, and
  Kawanabe]{kobler2022spdbn}
Reinmar~J. Kobler, Jun ichiro Hirayama, Qibin Zhao, and Motoaki Kawanabe.
\newblock {SPD} domain-specific batch normalization to crack interpretable
  unsupervised domain adaptation in {EEG}.
\newblock In \emph{36th Conference on Neural Information Processing Systems},
  pages 6219--6235, December 2022.

\bibitem[Kynk\"{a}\"{a}nniemi et~al.(2019)Kynk\"{a}\"{a}nniemi, Karras, Laine,
  Lehtinen, and Aila]{kynkaanniemi2019precision}
Tuomas Kynk\"{a}\"{a}nniemi, Tero Karras, Samuli Laine, Jaakko Lehtinen, and
  Timo Aila.
\newblock Improved precision and recall metric for assessing generative models.
\newblock In \emph{33rd Conference on Neural Information Processing Systems},
  Vancouver, BC, December 2019.

\bibitem[Lashgari et~al.(2020)Lashgari, Liang, and
  Maoz]{lashgari2020augmentation}
Elnaz Lashgari, Dehua Liang, and Uri Maoz.
\newblock Data augmentation for deep-learning-based electroencephalography.
\newblock \emph{Journal of Neuroscience Methods}, 346:\penalty0 108885, 2020.

\bibitem[Lawhern et~al.(2018)Lawhern, Solon, Waytowich, Gordon, Hung, and
  Lance]{lawhern2018eegnet}
Vernon~J. Lawhern, Amelia~J. Solon, Nicholas~R. Waytowich, Stephen~M. Gordon,
  Chou~Po Hung, and Brent~J. Lance.
\newblock {EEGN}et: a compact convolutional neural network for {EEG}-based
  brain--computer interfaces.
\newblock \emph{Journal of Neural Engineering}, 15\penalty0 (5):\penalty0
  056013, 2018.

\bibitem[Ledoit and Wolf(2004)]{ledoit2004wellconditioned}
Olivier Ledoit and Michael Wolf.
\newblock A well-conditioned estimator for large-dimensional covariance
  matrices.
\newblock \emph{Journal of Multivariate Analysis}, 88\penalty0 (2):\penalty0
  365--411, 2004.

\bibitem[Leeb et~al.(2007)Leeb, Lee, Keinrath, Scherer, Bischof, and
  Pfurtscheller]{leeb2007virtual}
Robert Leeb, Felix Lee, Claudia Keinrath, Reinhold Scherer, Horst Bischof, and
  Gert Pfurtscheller.
\newblock Brain--computer communication: motivation, aim, and impact of
  exploring a virtual apartment.
\newblock \emph{IEEE Transactions on Neural Systems and Rehabilitation
  Engineering}, 15\penalty0 (4):\penalty0 473--482, 2007.

\bibitem[Leonardi and Van~De~Ville(2015)]{leonardi2015spurious}
Nora Leonardi and Dimitri Van~De Ville.
\newblock On spurious and real fluctuations of dynamic functional connectivity
  during rest.
\newblock \emph{NeuroImage}, 104:\penalty0 430--436, 2015.

\bibitem[Li et~al.(2024)Li, Yu, He, Shen, Li, Sun, and Lin]{li2024spdddpm}
Yunchen Li, Zhou Yu, Gaoqi He, Yunhang Shen, Ke~Li, Xing Sun, and Shaohui Lin.
\newblock {SPD}-{DDPM}: denoising diffusion probabilistic models in the
  symmetric positive definite space.
\newblock In \emph{38th AAAI Conference on Artificial Intelligence}, pages
  13709--13717, Vancouver, BC, February 2024.

\bibitem[Lin(2019)]{lin2019cholesky}
Zhenhua Lin.
\newblock {R}iemannian geometry of symmetric positive definite matrices via
  {C}holesky decomposition.
\newblock \emph{SIAM Journal on Matrix Analysis and Applications}, 40\penalty0
  (4):\penalty0 1353--1370, 2019.

\bibitem[Lipman et~al.(2023)Lipman, Chen, Ben-Hamu, Nickel, and
  Le]{lipman2023flow}
Yaron Lipman, Ricky T.~Q. Chen, Heli Ben-Hamu, Maximilian Nickel, and Matt Le.
\newblock Flow matching for generative modeling.
\newblock In \emph{11th International Conference on Learning Representations},
  Kigali, Rwanda, May 2023.

\bibitem[Lipman et~al.(2024)Lipman, Havasi, Holderrieth, Shaul, Le, Karrer,
  Chen, Lopez-Paz, Ben-Hamu, and Gat]{lipman2024guide}
Yaron Lipman, Marton Havasi, Peter Holderrieth, Neta Shaul, Matt Le, Brian
  Karrer, Ricky T.~Q. Chen, David Lopez-Paz, Heli Ben-Hamu, and Itai Gat.
\newblock Flow matching guide and code.
\newblock \emph{arXiv:2412.06264 [cs.LG]}, 2024.

\bibitem[Liu et~al.(2023)Liu, Gong, and Liu]{liu2023rectified}
Xingchao Liu, Chengyue Gong, and Qiang Liu.
\newblock Flow straight and fast: learning to generate and transfer data with
  rectified flow.
\newblock In \emph{11th International Conference on Learning Representations},
  Kigali, Rwanda, May 2023.

\bibitem[Loshchilov and Hutter(2019)]{loshchilov2019adamw}
Ilya Loshchilov and Frank Hutter.
\newblock Decoupled weight decay regularization.
\newblock In \emph{7th International Conference on Learning Representations},
  New Orleans, LA, May 2019.

\bibitem[Lotte et~al.(2018)Lotte, Bougrain, Cichocki, Clerc, Congedo,
  Rakotomamonjy, and Yger]{lotte2018review}
Fabien Lotte, Laurent Bougrain, Andrzej Cichocki, Maureen Clerc, Marco Congedo,
  Alain Rakotomamonjy, and Florian Yger.
\newblock A review of classification algorithms for {EEG}-based brain--computer
  interfaces: a 10 year update.
\newblock \emph{Journal of Neural Engineering}, 15\penalty0 (3), 2018.

\bibitem[Luo and Lu(2018)]{luo2018cwgan}
Yun Luo and Bao-Liang Lu.
\newblock {EEG} data augmentation for emotion recognition using a conditional
  {W}asserstein {GAN}.
\newblock In \emph{40th Annual International Conference of the IEEE Engineering
  in Medicine and Biology Society (EMBC)}, pages 2535--2538, Honolulu, HI, July
  2018.

\bibitem[Marti(2020)]{marti2020corrgan}
Gautier Marti.
\newblock {C}orr{GAN}: sampling realistic financial correlation matrices using
  generative adversarial networks.
\newblock In \emph{IEEE International Conference on Acoustics, Speech and
  Signal Processing (ICASSP)}, pages 8459--8463, Barcelona, Spain, May 2020.

\bibitem[McCann(1997)]{mccann1997convexity}
Robert~J. McCann.
\newblock A convexity principle for interacting gases.
\newblock \emph{Advances in Mathematics}, 128\penalty0 (1):\penalty0 153--179,
  1997.

\bibitem[Miller et~al.(2024)Miller, Chen, Sriram, and Wood]{miller2024flowmm}
Benjamin~Kurt Miller, Ricky T.~Q. Chen, Anuroop Sriram, and Brandon~M. Wood.
\newblock {F}low{MM}: generating materials with {R}iemannian flow matching.
\newblock In \emph{41st International Conference on Machine Learning}, Vienna,
  Austria, July 2024.

\bibitem[Monti et~al.(2014)Monti, Hellyer, Sharp, Leech, Anagnostopoulos, and
  Montana]{monti2014estimating}
Ricardo~Pio Monti, Peter Hellyer, David Sharp, Robert Leech, Christoforos
  Anagnostopoulos, and Giovanni Montana.
\newblock Estimating time-varying brain connectivity networks from functional
  {MRI} time series.
\newblock \emph{NeuroImage}, 103:\penalty0 427--443, 2014.

\bibitem[Naeem et~al.(2020)Naeem, Oh, Uh, Choi, and Yoo]{naeem2020reliable}
Muhammad~Ferjad Naeem, Seong~Joon Oh, Youngjung Uh, Yunjey Choi, and Jaejun
  Yoo.
\newblock Reliable fidelity and diversity metrics for generative models.
\newblock In \emph{37th International Conference on Machine Learning}, pages
  7176--7185, July 2020.

\bibitem[Paszke et~al.(2019)Paszke, Gross, Massa, Lerer, Bradbury, Chanan,
  Killeen, Lin, Gimelshein, Antiga, Desmaison, K\"{o}pf, Yang, DeVito, Raison,
  Tejani, Chilamkurthy, Steiner, Fang, Bai, and Chintala]{paszke2019pytorch}
Adam Paszke, Sam Gross, Francisco Massa, Adam Lerer, James Bradbury, Gregory
  Chanan, Trevor Killeen, Zeming Lin, Natalia Gimelshein, Luca Antiga, Alban
  Desmaison, Andreas K\"{o}pf, Edward Yang, Zachary DeVito, Martin Raison,
  Alykhan Tejani, Sasank Chilamkurthy, Benoit Steiner, Lu~Fang, Junjie Bai, and
  Soumith Chintala.
\newblock {P}y{T}orch: An imperative style, high-performance deep learning
  library.
\newblock In \emph{33rd Conference on Neural Information Processing Systems},
  pages 8024--8035, Vancouver, BC, December 2019.

\bibitem[Peebles and Xie(2023)]{peebles2023dit}
William Peebles and Saining Xie.
\newblock Scalable diffusion models with transformers.
\newblock In \emph{IEEE/CVF International Conference on Computer Vision}, pages
  4195--4205, Paris, France, October 2023.

\bibitem[Pennec(2006)]{pennec2006intrinsic}
Xavier Pennec.
\newblock Intrinsic statistics on {R}iemannian manifolds: basic tools for
  geometric measurements.
\newblock \emph{Journal of Mathematical Imaging and Vision}, 25\penalty0
  (1):\penalty0 127--154, 2006.

\bibitem[Pennec et~al.(2006)Pennec, Fillard, and Ayache]{pennec2006tensor}
Xavier Pennec, Pierre Fillard, and Nicholas Ayache.
\newblock A {R}iemannian framework for tensor computing.
\newblock \emph{International Journal of Computer Vision}, 66\penalty0
  (1):\penalty0 41--66, 2006.

\bibitem[Pfurtscheller and da~Silva(1999)]{pfurtscheller1999erd}
Gert Pfurtscheller and Fernando H.~Lopes da~Silva.
\newblock Event-related {EEG/MEG} synchronization and desynchronization: basic
  principles.
\newblock \emph{Clinical Neurophysiology}, 110\penalty0 (11):\penalty0
  1842--1857, 1999.

\bibitem[Preti et~al.(2017)Preti, Bolton, and Van~De~Ville]{preti2017dynamic}
Maria~Giulia Preti, Thomas A.~W. Bolton, and Dimitri Van~De Ville.
\newblock The dynamic functional connectome: state-of-the-art and perspectives.
\newblock \emph{NeuroImage}, 160:\penalty0 41--54, 2017.

\bibitem[Qin et~al.(2025)Qin, Madeira, Thanou, and Frossard]{qin2025defog}
Yiming Qin, Manuel Madeira, Dorina Thanou, and Pascal Frossard.
\newblock {D}e{F}o{G}: discrete flow matching for graph generation.
\newblock In \emph{42nd International Conference on Machine Learning}, July
  2025.

\bibitem[Rasul et~al.(2021)Rasul, Seward, Schuster, and
  Vollgraf]{rasul2021timegrad}
Kashif Rasul, Calvin Seward, Ingmar Schuster, and Roland Vollgraf.
\newblock Autoregressive denoising diffusion models for multivariate
  probabilistic time series forecasting.
\newblock In \emph{38th International Conference on Machine Learning}, pages
  8857--8868, July 2021.

\bibitem[Ravuri and Vinyals(2019)]{ravuri2019cas}
Suman Ravuri and Oriol Vinyals.
\newblock Classification accuracy score for conditional generative models.
\newblock In \emph{33rd Conference on Neural Information Processing Systems},
  pages 12268--12279, Vancouver, BC, December 2019.

\bibitem[Roy et~al.(2023)Roy, Moshfeghi, Ibanez, Lopera, Parra, and
  Smith]{roy2023robust}
Om~Roy, Yashar Moshfeghi, Agustin Ibanez, Francisco Lopera, Mario~A. Parra, and
  Keith~M. Smith.
\newblock Robust, high temporal-resolution {EEG} functional connectivity
  detects increased connectivity coinciding with {P300} in visual short-term
  memory binding in both familial and sporadic prodromal {A}lzheimer's disease.
\newblock In \emph{Complex Networks 2023: The 12th International Conference on
  Complex Networks and Their Applications}, pages 679--682, 2023.

\bibitem[Roy et~al.(2024)Roy, Moshfeghi, Ibanez, Lopera, Parra, and
  Smith]{roy2024fast}
Om~Roy, Yashar Moshfeghi, Agustin Ibanez, Francisco Lopera, Mario~A. Parra, and
  Keith~M. Smith.
\newblock {FAST} functional connectivity implicates {P300} connectivity in
  working memory deficits in {A}lzheimer's disease.
\newblock \emph{Network Neuroscience}, 8\penalty0 (4):\penalty0 1467--1490,
  2024.

\bibitem[Roy et~al.(2025)Roy, Moshfeghi, Smith, Ibanez, Parra, and
  Smith]{roy2025hodgefast}
Om~Roy, Yashar Moshfeghi, J.~Smith, Agustin Ibanez, Mario~A. Parra, and
  Keith~M. Smith.
\newblock A {H}odge-{FAST} framework for high-resolution dynamic functional
  connectivity analysis of higher-order interactions in {EEG} signals.
\newblock In \emph{47th Annual International Conference of the IEEE Engineering
  in Medicine and Biology Society (EMBC)}, pages 1--6, 2025.
\newblock \doi{10.1109/EMBC58623.2025.11253015}.

\bibitem[Roy et~al.(2026)Roy, Moshfeghi, and Smith]{roy2026cdnn}
Om~Roy, Yashar Moshfeghi, and Keith~M. Smith.
\newblock Covariance density neural networks.
\newblock \emph{Transactions on Machine Learning Research}, 2026.

\bibitem[Roy et~al.(2019)Roy, Banville, Albuquerque, Gramfort, Falk, and
  Faubert]{roy2019dlreview}
Yannick Roy, Hubert Banville, Isabela Albuquerque, Alexandre Gramfort, Tiago~H.
  Falk, and Jocelyn Faubert.
\newblock Deep learning-based electroencephalography analysis: a systematic
  review.
\newblock \emph{Journal of Neural Engineering}, 16\penalty0 (5):\penalty0
  051001, 2019.

\bibitem[Schirrmeister et~al.(2017)Schirrmeister, Springenberg, Fiederer,
  Glasstetter, Eggensperger, Tangermann, Hutter, Burgard, and
  Ball]{schirrmeister2017deep}
Robin~Tibor Schirrmeister, Jost~Tobias Springenberg, Lukas Dominique~Josef
  Fiederer, Martin Glasstetter, Katharina Eggensperger, Michael Tangermann,
  Frank Hutter, Wolfram Burgard, and Tonio Ball.
\newblock Deep learning with convolutional neural networks for {EEG} decoding
  and visualization.
\newblock \emph{Human Brain Mapping}, 38\penalty0 (11):\penalty0 5391--5420,
  2017.

\bibitem[Schur(1911)]{schur1911bemerkungen}
Issai Schur.
\newblock Bemerkungen zur {T}heorie der beschr\"{a}nkten {B}ilinearformen mit
  unendlich vielen {V}er\"{a}nderlichen.
\newblock \emph{Journal f\"{u}r die reine und angewandte Mathematik},
  140:\penalty0 1--28, 1911.

\bibitem[Shuman et~al.(2013)Shuman, Narang, Frossard, Ortega, and
  Vandergheynst]{shuman2013gsp}
David~I. Shuman, Sunil~K. Narang, Pascal Frossard, Antonio Ortega, and Pierre
  Vandergheynst.
\newblock The emerging field of signal processing on graphs: extending
  high-dimensional data analysis to networks and other irregular domains.
\newblock \emph{IEEE Signal Processing Magazine}, 30\penalty0 (3):\penalty0
  83--98, 2013.

\bibitem[Sihag et~al.(2022)Sihag, Mateos, McMillan, and Ribeiro]{sihag2022vnn}
Saurabh Sihag, Gonzalo Mateos, Corey McMillan, and Alejandro Ribeiro.
\newblock Co{V}ariance neural networks.
\newblock In \emph{36th Conference on Neural Information Processing Systems},
  pages 17003--17016, Red Hook, NY, 2022. Curran Associates Inc.

\bibitem[Skovgaard(1984)]{skovgaard1984riemannian}
Lene~Theil Skovgaard.
\newblock A {R}iemannian geometry of the multivariate normal model.
\newblock \emph{Scandinavian Journal of Statistics}, 11\penalty0 (4):\penalty0
  211--223, 1984.

\bibitem[Smith et~al.(2017)Smith, Escudero, Parra, Ibanez, Starr, and
  Sala]{smith2017fast}
Keith Smith, Javier Escudero, Mario~A. Parra, Agustin Ibanez, John~M. Starr,
  and Sergio~Della Sala.
\newblock Locating temporal functional dynamics of visual short-term memory
  binding using graph modular {D}irichlet energy.
\newblock \emph{Scientific Reports}, 7:\penalty0 42013, 2017.

\bibitem[Smith et~al.(2019)Smith, Spyrou, and Escudero]{smith2019gvsa}
Keith Smith, Loukas Spyrou, and Javier Escudero.
\newblock Graph-variate signal analysis.
\newblock \emph{IEEE Transactions on Signal Processing}, 67\penalty0
  (2):\penalty0 293--305, 2019.
\newblock \doi{10.1109/TSP.2018.2881658}.

\bibitem[Sohl-Dickstein et~al.(2015)Sohl-Dickstein, Weiss, Maheswaranathan, and
  Ganguli]{sohldickstein2015thermo}
Jascha Sohl-Dickstein, Eric Weiss, Niru Maheswaranathan, and Surya Ganguli.
\newblock Deep unsupervised learning using nonequilibrium thermodynamics.
\newblock In \emph{32nd International Conference on Machine Learning}, pages
  2256--2265, Lille, France, July 2015.

\bibitem[Sohn et~al.(2015)Sohn, Lee, and Yan]{sohn2015cvae}
Kihyuk Sohn, Honglak Lee, and Xinchen Yan.
\newblock Learning structured output representation using deep conditional
  generative models.
\newblock In \emph{29th Conference on Neural Information Processing Systems},
  pages 3483--3491, Montr\'{e}al, QC, December 2015.

\bibitem[Song et~al.(2021{\natexlab{a}})Song, Meng, and Ermon]{song2021ddim}
Jiaming Song, Chenlin Meng, and Stefano Ermon.
\newblock Denoising diffusion implicit models.
\newblock In \emph{9th International Conference on Learning Representations},
  May 2021{\natexlab{a}}.

\bibitem[Song and Ermon(2019)]{song2019score}
Yang Song and Stefano Ermon.
\newblock Generative modeling by estimating gradients of the data distribution.
\newblock In \emph{33rd Conference on Neural Information Processing Systems},
  Vancouver, BC, December 2019.

\bibitem[Song et~al.(2021{\natexlab{b}})Song, Sohl-Dickstein, Kingma, Kumar,
  Ermon, and Poole]{song2021sde}
Yang Song, Jascha Sohl-Dickstein, Diederik~P. Kingma, Abhishek Kumar, Stefano
  Ermon, and Ben Poole.
\newblock Score-based generative modeling through stochastic differential
  equations.
\newblock In \emph{9th International Conference on Learning Representations},
  May 2021{\natexlab{b}}.

\bibitem[Steyrl et~al.(2016)Steyrl, Scherer, Faller, and
  M\"{u}ller-Putz]{steyrl2016randomforests}
David Steyrl, Reinhold Scherer, Josef Faller, and Gernot~R. M\"{u}ller-Putz.
\newblock Random forests in non-invasive sensorimotor rhythm brain-computer
  interfaces: a practical and convenient non-linear classifier.
\newblock \emph{Biomedical Engineering / Biomedizinische Technik}, 61\penalty0
  (1):\penalty0 77--86, 2016.

\bibitem[Strang(1999)]{strang1999dct}
Gilbert Strang.
\newblock The discrete cosine transform.
\newblock \emph{SIAM Review}, 41\penalty0 (1):\penalty0 135--147, 1999.

\bibitem[Tancik et~al.(2020)Tancik, Srinivasan, Mildenhall, Fridovich-Keil,
  Raghavan, Singhal, Ramamoorthi, Barron, and Ng]{tancik2020fourier}
Matthew Tancik, Pratul~P. Srinivasan, Ben Mildenhall, Sara Fridovich-Keil,
  Nithin Raghavan, Utkarsh Singhal, Ravi Ramamoorthi, Jonathan~T. Barron, and
  Ren Ng.
\newblock {F}ourier features let networks learn high-frequency functions in
  low-dimensional domains.
\newblock In \emph{34th Conference on Neural Information Processing Systems},
  pages 7537--7547, December 2020.

\bibitem[Tangermann et~al.(2012)Tangermann, M\"{u}ller, Aertsen, Birbaumer,
  Braun, Brunner, Leeb, Mehring, Miller, M\"{u}ller-Putz, Nolte, Pfurtscheller,
  Preissl, Schalk, Schl\"{o}gl, Vidaurre, Waldert, and
  Blankertz]{tangermann2012bci}
Michael Tangermann, Klaus-Robert M\"{u}ller, Ad~Aertsen, Niels Birbaumer,
  Christoph Braun, Clemens Brunner, Robert Leeb, Carsten Mehring, Kai~J.
  Miller, Gernot M\"{u}ller-Putz, Guido Nolte, Gert Pfurtscheller, Hubert
  Preissl, Gerwin Schalk, Alois Schl\"{o}gl, Carmen Vidaurre, Stephan Waldert,
  and Benjamin Blankertz.
\newblock Review of the {BCI} competition {IV}.
\newblock \emph{Frontiers in Neuroscience}, 6:\penalty0 55, 2012.

\bibitem[ten Brinke et~al.(2026)ten Brinke, Minartz, and
  Menkovski]{tenbrinke2026stflow}
Kiet~Bennema ten Brinke, Koen Minartz, and Vlado Menkovski.
\newblock {STF}low: data-coupled flow matching for geometric trajectory
  simulation.
\newblock In \emph{43rd International Conference on Machine Learning}, 2026.

\bibitem[Thanwerdas(2022)]{thanwerdas2022thesis}
Yann Thanwerdas.
\newblock \emph{{R}iemannian and stratified geometries on covariance and
  correlation matrices}.
\newblock PhD thesis, Universit\'{e} C\^{o}te d'Azur, 2022.

\bibitem[Tong et~al.(2024)Tong, Fatras, Malkin, Huguet, Zhang, Rector-Brooks,
  Wolf, and Bengio]{tong2024improving}
Alexander Tong, Kilian Fatras, Nikolay Malkin, Guillaume Huguet, Yanlei Zhang,
  Jarrid Rector-Brooks, Guy Wolf, and Yoshua Bengio.
\newblock Improving and generalizing flow-based generative models with
  minibatch optimal transport.
\newblock \emph{Transactions on Machine Learning Research}, 2024.

\bibitem[Vaswani et~al.(2017)Vaswani, Shazeer, Parmar, Uszkoreit, Jones, Gomez,
  Kaiser, and Polosukhin]{vaswani2017attention}
Ashish Vaswani, Noam Shazeer, Niki Parmar, Jakob Uszkoreit, Llion Jones,
  Aidan~N. Gomez, {\L}ukasz Kaiser, and Illia Polosukhin.
\newblock Attention is all you need.
\newblock In \emph{31st Conference on Neural Information Processing Systems},
  pages 5998--6008, Long Beach, CA, December 2017.

\bibitem[Vignac et~al.(2023)Vignac, Krawczuk, Siraudin, Wang, Cevher, and
  Frossard]{vignac2023digress}
Clement Vignac, Igor Krawczuk, Antoine Siraudin, Bohan Wang, Volkan Cevher, and
  Pascal Frossard.
\newblock {D}i{G}ress: discrete denoising diffusion for graph generation.
\newblock In \emph{11th International Conference on Learning Representations},
  Kigali, Rwanda, May 2023.

\bibitem[Villani(2009)]{villani2009optimal}
C\'{e}dric Villani.
\newblock \emph{Optimal Transport: Old and New}.
\newblock Springer, Berlin, Heidelberg, 2009.

\bibitem[Wang et~al.(2026)Wang, Ma, Li, and You]{wang2026jet}
Yifan Wang, Yijia Ma, Wen Li, and Chenyu You.
\newblock Let {EEG} models learn {EEG}.
\newblock \emph{arXiv:2605.21280 [cs.LG]}, 2026.

\bibitem[Williams et~al.(2023)Williams, Weinhardt, Wirzberger, and
  Musslick]{williams2023augmenting}
Chad~C. Williams, Daniel Weinhardt, Maria Wirzberger, and Sebastian Musslick.
\newblock Augmenting {EEG} with generative adversarial networks enhances brain
  decoding across classifiers and sample sizes.
\newblock In \emph{45th Annual Meeting of the Cognitive Science Society} Sydney, Australia, July 2023.

\bibitem[Williams et~al.(2025)Williams, Weinhardt, Hewson, P{\l}omecka, Langer,
  and Musslick]{williams2025eeggan}
Chad~C. Williams, Daniel Weinhardt, Joshua Hewson, Martyna~Beata P{\l}omecka,
  Nicolas Langer, and Sebastian Musslick.
\newblock {EEG-GAN}: a generative {EEG} augmentation toolkit for enhancing
  neural classification.
\newblock \emph{bioRxiv}, 2025.
\newblock \doi{10.1101/2025.06.23.661164}.

\bibitem[Williams(2025)]{williams2025weighted}
Robert Williams.
\newblock Scalable generative modeling of weighted graphs.
\newblock \emph{arXiv:2507.23111 [cs.LG]}, 2025.

\bibitem[Xu et~al.(2022)Xu, Dong, Li, Yang, Wang, and Zhao]{xu2022dcgan}
Xu F, Dong G, Li J, Yang Q, Wang L, Zhao Y, Yan Y, Zhao J, Pang S, Guo D, Zhang Y and Leng J.
\newblock Deep convolution generative adversarial network-based
  electroencephalogram data augmentation for post-stroke rehabilitation with
  motor imagery.
\newblock \emph{International Journal of Neural Systems}, 32\penalty0
  (9):\penalty0 2250039, 2022.
\newblock \doi{10.1142/S0129065722500393}.

\bibitem[Zhang and Liu(2018)]{zhang2018cdcgan}
Qiqi Zhang and Ying Liu.
\newblock Improving brain computer interface performance by data augmentation
  with conditional deep convolutional generative adversarial networks.
\newblock \emph{arXiv:1806.07108 [cs.HC]}, 2018.

\bibitem[Zhou et~al.(2016)Zhou, Wu, Lv, Zhang, and Guo]{zhou2016trial}
Bangyan Zhou, Xiaopei Wu, Zhao Lv, Lei Zhang, and Xiaojin Guo.
\newblock A fully automated trial selection method for optimization of motor
  imagery based brain-computer interface.
\newblock \emph{PLOS ONE}, 11\penalty0 (9):\penalty0 e0162657, 2016.

\end{thebibliography}
\end{document}